\documentclass[twoside,11pt]{article}
\usepackage{blindtext}
\usepackage[preprint]{jmlr2e}

\usepackage{amsmath,amsfonts}
\usepackage{algorithmic}
\usepackage{algorithm}
\usepackage{array}
\usepackage{subfigure}
\usepackage{textcomp}
\usepackage{stfloats}
\usepackage{url}
\usepackage{verbatim}
\usepackage{graphicx}
\usepackage{cite}
\usepackage{hyperref}
\usepackage{bbding}
\usepackage{booktabs} % professional-quality tables
\usepackage{amsfonts} % blackboard math symbols
\usepackage{nicefrac} % compact symbols for 1/2, etc.
\usepackage{microtype} % microtypography
\usepackage{xcolor} % colors
\usepackage{graphicx}
\usepackage{makecell}
\usepackage{algorithm,algorithmic}
\usepackage{paralist,amsmath, amssymb,bm}
\usepackage{multirow}

\newtheorem{thm}{Theorem}
\newtheorem{ass}{Assumption}

\usepackage{enumitem}
\usepackage{textcase}
\usepackage{booktabs}
\usepackage{fancybox}
\newcommand{\Norm}[1]{\left\|#1\right\|}
\def \E {\mathbb{E}}
\def \R {\mathbb{R}}

\def \cite {\citep}

\def \h {\mathbf{h}}

\def \u {\mathbf{u}}
\def \v {\mathbf{v}}
\def \w {\mathbf{w}}
\def \x {\mathbf{x}}

\def \z {\mathbf{z}}

\def \O {\mathcal{O}}

\usepackage{lastpage}
\jmlrheading{1}{2026}{1-\pageref{LastPage}}{XX/XX}{XX/XX}{21-0000}{Jiang, Yang, Wang, Zhang, and Li}

\ShortHeadings{Finite-sum Coupled Compositional Optimization}{Jiang, Yang, Wang, Zhang, and Li}
\firstpageno{1}

\begin{document}

\title{
Solving Finite-sum Coupled Compositional Optimization via Multi-block-Single-probe Estimator}

\author{\name Wei Jiang\textsuperscript{\rm 1,2} \email {12025220@njust.edu.cn} 
       \AND
       \name Sifan Yang\textsuperscript{\rm 2,3} \email yangsf@lamda.nju.edu.cn 
       \AND
       %\name Wenhao Yang \email yangwh@lamda.nju.edu.cn \\
       \name Yibo Wang\textsuperscript{\rm 2,3} \email wangyb@lamda.nju.edu.cn
       \AND
    \name Lijun Zhang\textsuperscript{\rm 2,3} \email zhanglj@lamda.nju.edu.cn \AND
       \name Zechao Li\textsuperscript{\rm 1} \email zechao.li@njust.edu.cn
       \AND
       \addr
       \textsuperscript{\rm 1}School of Computer Science and Engineering, Nanjing University of Science and Technology, China\\
        \textsuperscript{\rm 2}State Key Laboratory of Novel Software Technology, Nanjing University, China\\
        \textsuperscript{\rm 3}School of Artificial Intelligence, Nanjing University, China
       }

\editor{My editor}

\maketitle

\begin{abstract}%
Traditional variance reduction methods (e.g., SPIDER, SARAH, STORM) have been extensively investigated for improving the convergence rates of stochastic optimization. 
These techniques typically maintain a sequence of estimators for a single function (or gradient) across iterations. 
However, what if we need to track multiple functions, but can only access stochastic samples of $\mathcal{O}(1)$ functions at each iteration? 
This scenario arises in an important emerging family of finite-sum coupled compositional optimization (FCCO) problems of the form $\frac{1}{m}\sum_{i=1}^m f_i(g_i(\mathbf{w}))$, where each $g_i$ is accessible only through a stochastic oracle.
The key challenge is to track $\mathbf g(\mathbf{w})=(g_1(\mathbf{w}), \ldots, g_m(\mathbf{w}))$ over time, where $\mathbf g(\mathbf{w})$ has $m$ blocks but only $\mathcal{O}(1)$ blocks can be probed for their stochastic values at each step.
To address this challenge, we propose a novel Multi-block-Single-probe Variance Reduction (MSVR) estimator to efficiently trace $\mathbf g(\mathbf{w})$ under partial block sampling.
Building on the MSVR estimator, we develop several algorithms for FCCO problems, achieving improved sample complexities for non-convex, convex, strongly convex, and Polyak-{\L}ojasiewicz (PL) objectives. 
We further obtain an improved dependence on $m$ when the outer function gradients $\nabla f_i$ are linear. Empirical studies on multi-task deep AUC maximization further demonstrate the superior performance of the proposed estimators.
\end{abstract}


\section{Introduction}
\label{sec:Introduction}
This paper is motivated by solving the following Finite-sum Coupled Compositional Optimization (FCCO) problem, which arises in a wide range of machine-learning applications~\cite{dependent2022}:
\begin{equation}\label{p:1}
\begin{split}
    \min_{\w\in \R^d} F(\w) := \frac{1}{m} \sum_{i=1}^m f_i(g_i(\w)),
\end{split}
\end{equation}
where the outer function $f_i$ is a deterministic function. We assume that the inner function $g_i(\cdot)$ and its gradient $\nabla g_i(\cdot)$ are accessible only through a stochastic oracle that returns unbiased estimates $g_i(\cdot;\xi_i)$ and $\nabla g_i(\cdot;\xi_i)$, such that $\E\left[ g_i(\cdot;\xi_i) \right]=g_i(\cdot)$ and $\E\left[ \nabla g_i(\cdot;\xi_i) \right]=\nabla g_i(\cdot)$.
A special case to be considered separately is when each $\xi_i$ has finite support and is uniformly distributed. Then, the problem can be written as:
\begin{equation}\label{p:2}
\begin{split}
    \min _{\mathbf{w} \in \R^d} F(\w) := \frac{1}{m} \sum_{i=1}^{m} f_{i}\left(\frac{1}{n} \sum_{j=1}^{n} g_{i}(\mathbf{w}; \xi_{ij})\right).
\end{split}
\end{equation}
Note that the above two FCCO problems differ significantly from the classical two-level stochastic compositional optimization (SCO) objective $\E_{\zeta}[f_\zeta(\E_\xi [g(\w;\xi)])]$, and its finite-sum variant $1/m\sum_{i=1}^m f_i(1/n\sum_{j=1}^n g(\w; \xi_j))$~\cite{wang2017stochastic}. The key distinction is that the inner function is \textbf{coupled} with the outer index in FCCO problems.

Specifically, in standard two-level SCO, since the inner and outer components are decoupled, a single estimator suffices to track the inner function $g(\w)$.
However, for the FCCO problem, we must track the vector $\mathbf g(\w)=(g_1(\w),\ldots,g_m(\w))$ of $m$ distinct inner mappings.
In many settings, it is infeasible to draw samples for all $m$ blocks at every iteration (e.g., due to memory or computational constraints), and only a small subset can be probed.
To address this, the SOX algorithm~\cite{dependent2022} maintains block-wise estimators $\u=(\u^1,\ldots,\u^m)$ and updates the sampled blocks using moving average:
\begin{equation}\label{SOXE}
\begin{split}
    \mathbf{u}_{t}^{i}=\left\{\begin{array}{ll}
(1-\beta)\mathbf{u}_{t-1}^{i}+\beta g_i\left(\mathbf{w}_{t}; \xi_t^{i}\right) &  i \in \mathcal{B}_{1}^{t}\\
\mathbf{u}_{t-1}^{i}   &  i \notin \mathcal{B}_{1}^{t}
\end{array}\right.,
\end{split}
\end{equation}
where $\mathcal B^t_1\subseteq\{1,\ldots, m\}$ is the set of sampled blocks.
A similar moving average is used to construct the gradient estimator, leading to sample complexities of $\O(m\epsilon^{-4})$ for non-convex objectives, $\O(m\epsilon^{-3})$ for convex functions, and $\O(m \mu^{-2}\epsilon^{-1})$ for $\mu$-strongly convex objectives.

Notably, when $m=|\mathcal B^t_1|=1$, FCCO reduces to a special case of classical SCO. 
Nevertheless, the sample complexities of SOX are still worse than the state-of-the-art (SOTA) rates for standard SCO problems: $\O(\epsilon^{-3})$ for non-convex objectives, $\O(\epsilon^{-2})$ for convex functions, and $\O(\mu^{-1}\epsilon^{-1})$ for $\mu$-strongly convex objectives~\cite{Zhang2019ASC,11202712}.
Achieving these SOTA rates typically relies on the use of variance reduction.
Thus, a natural idea is to replace the moving average update of $\u_t^i$ using a variance-reduced estimator and apply a similar technique to the gradient estimator.
For example, one can modify the update for $\u_t^i$ according to the STORM~\cite{cutkosky2019momentum} technique:
\begin{equation*}%\label{SOXE2}
\begin{split}
    \mathbf{u}_{t}^{i}=\left\{\begin{array}{lll}
(1-\beta)\mathbf{u}_{t-1}^{i}+\beta g_i\left(\mathbf{w}_{t}; \xi_t^{i}\right)+ \underbrace{(1-\beta)(g_i\left(\mathbf{w}_{t}; \xi_t^{i}\right) - g_i\left(\mathbf{w}_{t-1}; \xi_t^{i}\right))}\limits_{\text{error correction}}&  i \in \mathcal{B}_{1}^{t}\\
\mathbf{u}_{t-1}^{i}   &  i \notin \mathcal{B}_{1}^{t}
\end{array}\right..
\end{split}
\end{equation*}
However, this direct modification does not improve upon SOX~\cite{dependent2022}.
The reason is that the standard error correction term in STORM only controls the sampling noise in $\xi_t^i$ (i.e., $g_i(\w_t;\xi_t^i)$), but it does not account for the additional randomness introduced by block sampling through $\mathcal B_1^t$. This leads to the following open question:
\emph{How can we improve the sample complexities of FCCO to match the SOTA results of SCO via probing only $\mathcal O(1)$ blocks per iteration?}

To address this question, we propose a novel variance reduction mechanism for tracking $\mathbf g(\w_t)$, which we call the Multi-block-Single-probe Variance Reduction (MSVR) estimator.
MSVR follows a STORM-like update on the sampled blocks, but uses a customized error-correction term that simultaneously handles the randomness from both the oracle sample $\xi_t^i$ and the block selection $\mathcal B_1^t$.
Building on MSVR, we develop several algorithms and provide convergence analyses that cover non-convex, convex, strongly convex, and PL objectives; finite or infinite support for $\xi_i$; and linear or non-linear outer gradients $\nabla f_i$. Our main contributions are:
\begin{compactitem}
\item We propose the novel MSVR estimator for tracking a sequence of multiple blocks of functions by probing only $\mathcal O(1)$ blocks per iteration.
\item We develop three new algorithms (MSVRM-v1, v2, v3) based on the MSVR estimator, establishing improved complexities for non-convex, convex, and PL objectives.
\item Our MSVRM-v1 method enjoys the same order on $\epsilon$ as SOX, but does not depend on $m$; MSVRM-v2 improves the dependence on $\epsilon$, matching the SOTA complexities for standard SCO when $m=1$; our MSVRM-v3 method further reduces the dependence on $\epsilon$ for the finite support of $\xi$, and also attains the SOTA complexities when $m=1$. 
\item We validate our theory through experiments on multi-task deep AUC maximization, demonstrating the practical advantages of the proposed algorithms.
\end{compactitem}

A preliminary version of this work~\cite{jiang2022multiblocksingleprobe} was presented at the conference -- Advances in Neural Information Processing Systems 35. This journal version significantly expands upon the conference paper with the following extensions:
\begin{compactitem}
\item We develop a momentum-based method that achieves a complexity of $\mathcal{O}(m\epsilon^{-4})$ without the average smoothness assumption (Theorem~\ref{new:1}). We also provide the convergence results for convex and PL objectives (Theorem~\ref{new:4}).
\item Our new methods employ adaptive step sizes to avoid reliance on problem-dependent constants (e.g., smoothness parameters and noise variance) while maintaining the same complexity rates (Theorems~\ref{new:1},\ref{new:2},\ref{new:3}).
\item We investigate the special case where $\nabla f_i(\cdot)$ is linear and derive improved dependence on the number of blocks $m$ across all settings (Theorems~\ref{new:1},\ref{new:2},\ref{new:3},\ref{new:4}--\ref{new:6}).
\item We empirically validate the performance of these new adaptive methods, confirming their effectiveness.
\end{compactitem}
A summary of our sample complexity results is shown in Tables~\ref{table:2+} and \ref{table:1+}, which also display the differences between this paper and the previous conference version.
\begin{table*}[t]
\caption{Summary of sample complexity results for non-convex functions.}
\label{table:2+}
\begin{center}
\resizebox{0.95\textwidth}{!}{
\begin{tabular}{lcccc}
\toprule
Method &  Adaptive & Non-linear $\nabla f_i$ & Linear $\nabla f_i$ & Additional Assumptions   \\
\midrule
Theorem~\ref{thm:1}   & \XSolidBrush & $\mathcal{O}\left( \epsilon^{-4}\right)$ & - & Average smoothness \\
%\hline
Theorem~\ref{thm:2}   & \XSolidBrush & $\mathcal{O}\left(m \epsilon^{-3}\right)$ & - & Average smoothness
  \\
Theorem~\ref{thm:3}  & \XSolidBrush & $\mathcal{O}\left(m \sqrt{n}\epsilon^{-2}\right)$ & -  & Finite-sum structure \\
\midrule
\midrule
\textbf{Theorem~\ref{new:1}~(new)}   & \checkmark & $\mathcal{O}\left(m \epsilon^{-4}\right)$ & $\mathcal{O}\left( \epsilon^{-4}\right)$ & - \\
%\hline
\textbf{Theorem~\ref{new:2}~(new)}  & \checkmark & $\mathcal{O}\left(m \epsilon^{-3}\right)$ & $\mathcal{O}\left( \epsilon^{-3}\right)$ & Average smoothness
  \\
%\midrule
\textbf{Theorem~\ref{new:3}~(new)}  & \checkmark & $\mathcal{O}\left(m \sqrt{n}\epsilon^{-2}\right)$ & $\mathcal{O}\left( \sqrt{mn} \epsilon^{-2}\right)$  & Finite-sum structure \\
\bottomrule
\end{tabular}}
\end{center}
\end{table*}
\begin{table*}[!t]
\caption{Summary of sample complexity results for convex/PL functions.}
\label{table:1+}
\begin{center}
\resizebox{0.95\textwidth}{!}{
\begin{tabular}{lcccc}
\toprule
Method &  $\nabla f_i$ &  Convex  & PL condition & Additional Assumptions   \\
\midrule
Theorem~\ref{thm:4} & non-linear   & $\mathcal{O}\left( \epsilon^{-3}\right)$ & $\mathcal{O}\left(\mu^{-2} \epsilon^{-1}\right)$ & Average smoothness \\
%\hline
Theorem~\ref{thm:4+} & non-linear  & $\mathcal{O}\left(m \epsilon^{-2}\right)$ & $\mathcal{O}\left(m \mu^{-1} \epsilon^{-1}\right)$ & Average smoothness
  \\
Theorem~\ref{thm:6} & non-linear &  $\widetilde{\mathcal{O}}\left(m \sqrt{n}\epsilon^{-1}\right)$ & $\widetilde{\mathcal{O}}\left(m\sqrt{n}\mu^{-1} \right)$  & Finite-sum structure \\
\midrule
\midrule
\textbf{Theorem~\ref{new:4}~(new)}  & linear & $\mathcal{O}\left(\epsilon^{-3}\right)$ & $\mathcal{O}\left(\mu^{-2}  \epsilon^{-1}\right)$ & - \\
%\hline
\textbf{Theorem~\ref{new:5}~(new)} & linear & $\mathcal{O}\left( \epsilon^{-2}\right)$ & $\mathcal{O}\left(\mu^{-1}\epsilon^{-1}\right)$ &   Average smoothness 
  \\
%\midrule
\textbf{Theorem~\ref{new:6}~(new)}& linear  & $\widetilde{\mathcal{O}}\left( \sqrt{mn} \epsilon^{-1}\right)$ & $\widetilde{\mathcal{O}}\left( \sqrt{mn} \mu^{-1}\right)$  & Finite-sum structure    \\
\bottomrule
\end{tabular}}
\end{center}
\end{table*}
\section{Related Work}
This section briefly reviews related work on variance reduction methods and stochastic compositional optimization (SCO) problems.
\subsection{Variance Reduction Techniques}
Variance reduction for stochastic optimization originated with the Stochastic Average Gradient (SAG) algorithm~\cite{DBLP:conf/nips/RouxSB12} for finite-sum empirical risk minimization (ERM).
By maintaining a memory of past gradients, SAG attains a linear convergence rate for strongly convex finite-sum problems.
Subsequent methods such as SVRG~\cite{NIPS2013_ac1dd209,NIPS:2013:Zhang} and SAGA~\cite{DBLP:conf/nips/DefazioBL14} further improved complexity guarantees for smooth and strongly convex objectives, achieving a logarithmic complexity. 
SARAH~\cite{arxiv.1703.00102} later refined the estimator design and improved convergence for smooth convex problems.

For non-convex ERM, the SPIDER estimator~\cite{Fang2018SPIDERNN} improved the sample complexity of SGD from $O(\epsilon^{-4})$ to $O(\epsilon^{-3})$ in the stochastic setting, and to $O(\sqrt{n}\epsilon^{-2})$ for finite-sum structure (where $n$ is the number of components in the finite-sum setting).
Variants such as SpiderBoost~\cite{Wang2018SpiderBoostAC} use constant step sizes for better empirical performance, and STORM~\cite{cutkosky2019momentum} avoids using large batches required by earlier variance reduction methods. 
More recent developments include adaptive methods that reduce hyper-parameter tuning~\cite{levy2021storm, Liu2022METASTORMGF, NeurIPS:2024:Jiang:A} and sign-based techniques that reduce communication costs~\cite{chzhen2023signsvrg,NeurIPS:2024:Jiang:B,jiang2025improved,jiang2025convergence}.

\subsection{Stochastic Compositional Optimization Problems}
Several classes of stochastic compositional optimization (SCO) problems have been studied in the literature.

\paragraph{Two-level SCO.} The first class is the classical two-level SCO whose objective is given by $\mathbb E_{\xi}[f_{\xi}(\mathbb E_{\omega}[g_{\omega}(\mathbf w)])]$, where $\xi$ and $\omega$ are random variables.
While early work on two-level SCO dates back to the 1970s, a comprehensive modern study was initiated by literature~\cite{wang2017stochastic}.
They proposed a two-time-scale classical algorithm named SCGD, and established its asymptotic guarantees and non-asymptotic convergence rates.
Later, \cite{Ghadimi2020AST} proposed the NASA algorithm, which incorporates a momentum-based update to track the inner function and gradient, achieving a better sample complexity for non-convex objectives.
Following this work, many studies have sought to improve the algorithmic design and convergence rates for two-level SCO~\cite{wang2016accelerating,Ghadimi2020AST,Zhang2020OptimalAF}.
In particular, recent works have applied variance reduction techniques based on SPIDER/SARAH/STORM to estimate the inner functions  and gradients~\cite{DBLP:journals/corr/abs-1809-02505,Yuan2019EfficientSN,chen2021solving,qi2021online}, achieving complexities of $\O(\epsilon^{-3})$, $\O(\epsilon^{-2})$, $\O(\mu^{-1}\epsilon^{-1})$, for non-convex, convex, and $\mu$-strongly convex functions, respectively.

\paragraph{Multi-level SCO.} The Multi-level SCO, i.e., $\mathbb E_{\xi_1}[f^1_{\xi_1}(\mathbb E_{\xi_2}[f^2_{\xi_2}(\ldots( \mathbb E_{\xi_K}[f^K_{\xi_K}(\mathbf w))]\ldots)])]$, was first investigated by \cite{Yang2019MultilevelSG}, who introduced a multi-timescale stochastic approximation method.
Inspired by NASA, \cite{balasubramanian2020stochastic} utilized a linearized averaging estimator to track the function value, achieving a better complexity for non-convex functions.
More recent studies have focused on further improving sample complexity and reducing the dependence on the number of levels via variance reduction techniques~\cite{balasubramanian2020stochastic,chen2021solving,Zhang2020OptimalAF,Zhang2021MultiLevelCS,ICML:2022:Jiang,ICML:2024:Jiang,TPAMI:2025:Jiang}.
However, applying standard two-level or multi-level SCO methods to FCCO would require probing all blocks of inner function per iteration, which is prohibitive in many applications.

\paragraph{Conditional stochastic optimization (CSO).} CSO problems consider objectives of the form
$\mathbb E_{\xi}[f_{\xi}(\mathbb E_{\omega|\xi} g_{\omega}(\mathbf w; \xi)])]$~\cite{hu2020biased},  where the notation $\omega|\xi$ indicates that the distribution of $\omega$ may depend on $\xi$.
The dependence of the inner component on the outer randomness distinguishes CSO from standard SCO.
For CSO, \cite{hu2020biased} proposed BSGD and a variance-reduced variant (BSpiderBoost), establishing complexities for non-convex, convex, and strongly convex functions.
However, their algorithms require a huge batch size for estimating the inner functions. Later, \cite{he2023debiasing} introduced  stochastic extrapolation to effectively reduce the bias in the estimation.
By further combining extrapolation with variance reduction, they achieved improved sample complexity bounds for CSO problems. 

\paragraph{Finite-sum Coupled Compositional Optimization (FCCO).} 
FCCO is a special case of CSO where the objective has the form $\frac{1}{n}\sum_{i=1}^n f_{\xi_i}(\mathbb E_{\omega|\xi_i}[g_\omega(\mathbf w; \xi_i)])$ and the outer variable $\xi$ has a finite support.
This outer finite-sum structure enables the development of more practical algorithms that avoid large batch sizes. 
The FCCO problem was first studied by~\cite{qi2021stochastic} in the context of maximizing the point-estimator of the area under the precision-recall curve.
It was recently investigated comprehensively by~\cite{dependent2022}, and many other applications have been demonstrated, such as p-norm push, listwise ranking, neighborhood component analysis, deep survival analysis, deep latent variable models, and self-supervised contrastive learning~\cite{yang2023algorithmicfoundationsempiricalxrisk}.
Nevertheless, their SOX algorithm~\cite{dependent2022} suffers from the suboptimal complexities discussed in the previous section. 
More recent developments have focused on different FCCO problem settings, including non-smooth FCCO~\cite{chen2025stochasticmomentummethodsnonsmooth}, weakly-convex FCCO~\cite{hu2023nonsmooth}, and primal-dual FCCO algorithms~\cite{wang2025a}.
%===================   Section  3 ========================
\section{Multi-block-Single-probe Estimator}\label{sec:3}
This section presents the proposed multi-block-single-probe estimator and establishes its theoretical guarantees.
\subsection{The Proposed Estimator}
Suppose that our budget allows probing only $B_1$ out of the $m$ functional mappings in $\mathbf g(\w)$ per iteration.
At iteration $t$, we sample a block set $\mathcal B_1^t\subseteq \{1,\ldots, m\}$ with $|\mathcal B_1^t|=B_1$.
For each $i \in \mathcal{B}_1^t$, we obtain an unbiased stochastic estimation $g_i(\w_t;\xi_t^i)$ and assume a bounded variance condition $\mathbb{E}[\|g_{i}(\mathbf{w}_t ; \xi_t^{i})-g_{i}(\mathbf{w}_t)\|^{2}] \leq \sigma^{2}$.
We then update only the sampled blocks of the estimator $\u_t=(\u_t^1,\ldots,\u_t^m)$.
Specifically, we update $\u_t^i$ for $i \in \mathcal{B}_1^t$ and keep all other blocks ($i \notin \mathcal{B}_1^t$) unchanged.
The estimator $\u_t$ takes the form:
\begin{equation}\label{MSVR}
\begin{split}
    \mathbf{u}_{t}^{i}=\left\{\begin{array}{ll}
\bar{\mathbf u}^i_t &  i \in \mathcal{B}_{1}^{t}\\
\mathbf{u}_{t-1}^{i}  & i \notin \mathcal{B}_{1}^{t}
\end{array}\right.,
\end{split}
\end{equation}
where $\bar{\mathbf u}_i^t$ is defined as:
\begin{align*}
    \bar{\mathbf u}_i^t = (1-\beta_{t})\mathbf{u}_{t-1}^{i}+\beta_{t} g_i\left(\mathbf{w}_{t}; \xi_t^{i}\right)+ \underbrace{\colorbox{green!30}{$\gamma_{t}$} \left( g_i\left(\mathbf{w}_{t}; \xi_t^{i}\right)-g_i\left(\mathbf{w}_{t-1}; \xi_t^{i}\right) \right)}\limits_{\text{customized error correction}}.
\end{align*}
We call equation~\eqref{MSVR} the Multi-block-Single-probe Variance Reduction (MSVR) estimator. Here, `Multi-block' indicates that it tracks multiple functions $(g_1, \ldots, g_m)$ simultaneously, and `Single-probe' highlights that the number of sampled blocks $B_1$ can be as small as one.

When $\gamma_t=0$, the estimator employs a standard momentum-style update. For this simplified estimator, we can obtain the following guarantee.
%===================   Lemma  1 ========================
\begin{lemma}\label{lem:main0}
If $g_i(\w)$ is $C_g$-Lipschitz continuous. Setting $\gamma_{t} = 0$, and $\beta_{t} = \beta < 1$, we have:
\begin{equation*}
    \begin{split}
        \E \left[ \Norm{\u_{t} - g(\w_{t})}^2 \right] \leq \left(1-\frac{\beta B_1}{2m} \right)\E \left[\Norm{ \u_{t-1} - g(\w_{t-1})}^2 \right] 
          + B_1 \sigma^2 \beta^2+\frac{5 m^2 C_g^2}{\beta B_1} \E \left[\left\| \w_{t} -\w_{t-1}\right\|^2\right]. 
        \end{split}
\end{equation*}
\end{lemma}
\textbf{Remark:} By setting the hyper-parameter $\beta$ and the learning rate $\eta_t$ appropriately, we can ensure that the function estimation error $\E [\Norm{ \u_{t} - g(\w_{t})}^2]$ decreases gradually. However, the bound in Lemma~\ref{lem:main0} contains the $1/\beta$ factor multiplying $\E[\|\w_t-\w_{t-1}\|^2]$, which leads to weaker rates and motivates a variance reduction correction.

\subsection{Variance Reduction via Customized Error Correction}
Although Lemma~\ref{lem:main0} ensures decreasing estimation error, this bound can be further improved by choosing $\gamma_t$ more carefully to trigger the effect of customized error correction. 
Specifically, our analysis suggests setting $\gamma_{t} = \frac{m-B_1}{B_1(1-\beta_{t})}+(1-\beta_{t})$. 
This choice differs from the standard STORM update~\cite{cutkosky2019momentum}, which sets $\gamma_{t}=1-\beta_{t}$. 
The additional term in $\gamma_{t}$, which we denote as $\gamma_{t}^0 =\frac{m-B_1}{B_1(1-\beta_{t})}$, is crucial for accounting for the randomness from both block sampling and the estimation error in the blocks that are not updated.
To illustrate the role of $\gamma_{t}^0$, let us analyze the expected tracking error $\|\u_t - \mathbf g(\w_t)\|^2=\sum_{i=1}^m\|\u^{i}_{t} - g_{i}(\w_t)\|^2$. We assume that $g_i(\cdot;\xi_t^i)$ is $C_g$-Lipschitz continuous.  

\textbf{Analysis:}
Let us focus on a fixed index $i\in \{1,\ldots,m\}$. By taking expectation over the sampling of $\mathcal B_1^t$, we have:
\begin{align*}
   \E\left[ \|\u^{i}_{t}  - g_{i}(\w_t)\|^2\right]
   = \frac{B_1}{m}\underbrace{\E\left[ \|\bar\u^{i}_{t}  - g_{i}(\w_t)\|^2\right]}\limits_{A_1} +  (1-\frac{B_1}{m})\underbrace{\E\left[ \|\u^{i}_{t-1}  - g_{i}(\w_t)\|^2\right]}\limits_{A_2}. 
\end{align*}
Note that $A_1$ (the error for the sampled block) can be handled by building the recurrence with $\|\u^i_{t-1} - g_i(\w_{t-1})\|^2$.
The main difficulty lies in the second term, $A_2$ (the error for the unsampled block), which can be decomposed as  
\begin{align*}
A_{2}=&\E\left[\|\u^i_{t-1} - g_i(\w_{t-1}) + g_i(\w_{t-1}) - g_i(\w_t)\|^2 \right]\\
    = &\E[\|\u^i_{t-1} - g_i(\w_{t-1})\|^2  + \|g_i(\w_{t-1}) - g_i(\w_t)\|^2+ \underbrace{2(\u^i_{t-1} - g_i(\w_{t-1}))^{\top}(g_i(\w_{t-1}) - g_i(\w_t))}\limits_{A_{21}}]\\
    \leq & \E \left[\Norm{ \u_{t-1}^i - g_i(\w_{t-1})}^2 + C_g^2 \Norm{ \w_{t-1} - \w_{t}}^2  + A_{21}\right]
\end{align*}
To control the cross-term $A_{21}$, we use the additional factor $\gamma_t^0$ from our customized error correction term in $A_1$. This becomes clear from the decomposition of $A_1$: 
\begin{align*}
    A_1 
    = &\E[\|\underbrace{(1-\beta_{t})(\u^i_{t-1} - g_i(\w_{t-1}))}\limits_{A_{11}} +  \underbrace{\gamma_{t}^0 (g_i(\w_t) - g_i(\w_{t-1}))}\limits_{A_{12}} + \underbrace{\beta_{t} (g_i(\w_t; \xi_t^{i}) - g_i(\w_t))}\limits_{A_{13}}\\ &+\underbrace{\gamma_{t}(g_i(\w_t; \xi_t^{i}) - g_i(\w_{t-1}; \xi_t^{i}) - g_i(\w_t) + g_i(\w_{t-1}))}\limits_{A_{14}} \|^2]\\
    \leq &\E\left[\|A_{11}+A_{12}\|^2 +\|A_{13}+A_{14}\|^2\right]\\
    \leq & \E\left[\|A_{11}\|^2 + \|A_{12}\|^2  + 2A_{11}^{\top}A_{12} + 2\|A_{13}\|^2 + 2\|A_{14}\|^2\right]\\
    \leq & \E\left[(1-\beta_{t})  \Norm{\u^i_{t-1} - g_i(\w_{t-1})}^2 + \gamma_{t}^2 C_g^2   \Norm{\w_{t} - \w_{t-1}}^2\right]\\
    & +\E\left[ 2A_{11}^{\top}A_{12}+ 2\beta_{t}^2\sigma^2+2\gamma_{t}^2C_g^2 \Norm{\w_{t} - \w_{t-1}}^2\right].
\end{align*}
The resulting term $\E[2A_{11}^{\top}A_{12}]$ has the opposite sign to $\E[A_{21}]$. 
By carefully choosing that $\gamma_{t}^0 = \frac{m-B_1}{B_1(1-\beta_t)}$, these cross-terms cancel out perfectly in expectation, i.e., $\frac{B_1}{m} \E[2A_{11}^{\top}A_{12}] + (1-\frac{B_1}{m}) \E[A_{21}]=0$.
By setting $\beta_t \leq 1/2$, we have $\gamma_{t} \leq \frac{2m}{B_1}$ and the bound for a single component $i$ becomes:
\begin{align*}
    &\E \left[ \Norm{\u_{t}^i - g_i(\w_{t})}^2 \right] = \left(1-\frac{B_1}{m}\right) A_2 + \frac{B_1}{m} A_1\\
    \leq & (1-\frac{B_1}{m})\E \left[\Norm{ \u_{t-1}^i - g_i(\w_{t-1})}^2 +C_g^2\Norm{ \w_{t-1} - \w_{t}}^2 \right] +\frac{B_1}{m}\E\left[(1-\beta_{t})  \Norm{\u^i_{t-1} - g_i(\w_{t-1})}^2\right]\\
    & +\frac{B_1}{m}\E\left[  2\beta_{t}^2\sigma^2+3\gamma_{t}^2C_g^2 \Norm{\w_{t} - \w_{t-1}}^2\right]\\
    \leq & (1-\frac{B_1\beta_{t}}{m})\E\left[ \Norm{\u^i_{t-1} - g_i(\w_{t-1})}^2\right] + \frac{2B_1\beta_{t}^2\sigma^2}{m}+ \frac{13 m}{B_1}C_g^2 \E\left[ \Norm{\w_{t} - \w_{t-1}}^2\right]
\end{align*}
Summing over all $i=\{1, \ldots, m\}$, since $\|\u_t - \mathbf g(\w_t)\|^2=\sum_{i=1}^m\|\u^{i}_{t} - g_{i}(\w_t)\|^2$, we obtain the following key lemma.
%===================   Lemma  2 ========================
\begin{lemma}\label{lem:main1}
Suppose that $g_i(\w_t;\xi_t^i)$ is $C_g$-Lipschitz continuous. By setting that $\gamma_{t} = \frac{m-B_1}{B_1(1-\beta_{t})}+(1-\beta_{t})$, for $\beta_{t} \leq \frac{1}{2}$, we have:
\begin{equation*}
\begin{split}
\E\left[\left\|\u_{t}-g\left(\w_{t}\right)\right\|^{2}\right]  \leq \left(1-\frac{B_1\beta_{t}}{m}\right)\E\left[\left\|\u_{t-1}-g\left(\w_{t-1}\right)\right\|^{2}\right]+2B_1\beta_{t}^{2} \sigma^{2}\\
+\frac{13 m^{2} C_g^2}{B_1}\E\left[\left\|\w_{t}-\w_{t-1}\right\|^{2}\right].
\end{split}
\end{equation*}
\end{lemma}
\textbf{Remark:} 
Compared to Lemma~\ref{lem:main0} (where $\gamma_t=0$), we have eliminated the $\frac{1}{\beta_t}$ factor on the $\E[\|\w_{t}-\w_{t-1}\|^{2}]$ term, which allows for a much more rapid decrease in the estimation error.

\subsection{The Single-Point Variant}
A limitation of the MSVR estimator is that it queries the oracle at two points~(i.e., $\w_t$ and $\w_{t-1}$) to compute $g_i(\w_t;\xi_t^i)$ and $g_i(\w_{t-1};\xi_t^i)$.
With a more careful analysis, we can probe each sampled block at a single point, following the idea used by~\cite{balasubramanian2020stochastic} and~\cite{chen2021solving}. 
Specifically, we approximate the difference $g_i(\mathbf{w}_{t}; \xi_t^{i})-g_i(\mathbf{w}_{t-1}; \xi_t^{i})$ via a first-order Taylor expansion, replacing it with $\nabla g_i(\mathbf{w}_{t}; \xi_t^{i})^{\top}(\w_{t}-\w_{t-1})$. 
Since we need to compute $\nabla g_i\left(\mathbf{w}_{t}; \xi_t^{i}\right)$ when estimating the gradient, the computation of the gradient is at no cost. 
This yields the single-point MSVR (MSVR-SP) estimator:
\begin{equation}\label{MSVR-SP}
\begin{split}
\mathbf{u}_{t}^{i}=\left\{\begin{array}{lll}
\widetilde{\u}_t^i &  i \in \mathcal{B}_{1}^{t}\\
\mathbf{u}_{t-1}^{i}  & i \notin \mathcal{B}_{1}^{t}
\end{array}\right.,
\end{split}
\end{equation}
where $\widetilde{\u}_t^i$ is defined as:
\begin{align*}
    \widetilde{\u}_t^i = &(1-\beta_{t})\mathbf{u}_{t-1}^{i}+\beta_{t} g_i\left(\mathbf{w}_{t}; \xi_t^{i}\right) + \gamma_{t} \nabla g_i\left(\mathbf{w}_{t}; \xi_t^{i}\right)^{\top}  \left( \mathbf{w}_{t} -\w_{t-1} \right).
\end{align*}
To demonstrate the main difference between MSVR and MSVR-SP, we provide the following analysis for comparison. 
To deal with the $\E\left[A_{14}\right]$ term appearing in the analysis of MSVR, we use the unbiased nature of $g_i(\cdot; \xi_t^{i})$ and obtain:
\begin{align*}
&\E\left[\Norm{g_i(\mathbf{w}_{t}; \xi_t^{i})-g_i(\mathbf{w}_{t-1}; \xi_t^{i}) - g_i(\w_t) + g_i(\w_{t-1})}^2\right]\\ 
    \leq & \E\left[\Norm{g_i(\mathbf{w}_{t}; \xi_t^{i})-g_i(\mathbf{w}_{t-1}; \xi_t^{i}) }^2\right] \leq C_g^2\E\left[\Norm{\w_t-\w_{t-1}}^2\right].
\end{align*}
While for MSVR-SP, we repalce $g_i(\mathbf{w}_{t}; \xi_t^{i})-g_i(\mathbf{w}_{t-1}; \xi_t^{i})$ with $\nabla g_i(\mathbf{w}_{t}; \xi_t^{i})^{\top}(\w_{t}-\w_{t-1})$. This leads to a very different analysis, in which we employ the smoothness assumption and the unbiased nature of $\nabla g_i(\cdot;\xi_t^i)$, ensuring:
\begin{align*}
 &\E\left[\Norm{ \nabla g_i(\w_{t};\xi_t^i)^{\top}(\w_{t} - \w_{t-1}) - g_i(\w_t) + g_i(\w_{t-1})}^2 \right] \\
 \leq &  \E\left[\Norm{  \left(\nabla g_i(\w_{t};\xi_t^i) - \nabla g_i(\w_{t})\right)^{\top}(\w_{t} - \w_{t-1})  }^2 \right]\\
 & +  \E\left[\Norm{ \nabla g_i(\w_{t})^{\top}(\w_{t} - \w_{t-1}) - g_i(\w_t) + g_i(\w_{t-1})}^2 \right]\\
    \leq &\E\left[ \sigma^2\Norm{\w_t-\w_{t-1}}^2+  L_g^2\Norm{\w_t-\w_{t-1}}^4/4\right].
\end{align*}
Thus, the MSVR-SP estimator can achieve the following recurrence for its estimation error.
%===================   Lemma  3 ========================
\begin{lemma}\label{lem:main3}
Suppose that $g_i(\w_t;\xi_t^i)$ is $L_g$-smooth. Setting $\gamma_{t} = \frac{m-B_1}{B_1(1-\beta_{t})}+(1-\beta_{t})$ and $\beta_t \leq 1/2$, we have:
\begin{align*}
\E\left[\Norm{\u_{t} - g(\w_{t})}^2\right] 
\leq &\left(1 - \frac{B_1\beta_{t}}{m}\right)\E\left[\Norm{\u_{t-1}- g(\w_{t-1})}^2\right] + 2B_1\beta_{t}^2\sigma^2 \\ 
+ &\left(\frac{4 L_g^2}{\beta_{t} } \|\w_{t} - \w_{t-1}\|^2 + 9 C_g^2 +{8 \sigma^2}\right)\frac{m^2}{B_1}\E\left[\|\w_{t} - \w_{t-1}\|^2\right].
\end{align*}
\end{lemma}
{\bf Remark:} The bound in Lemma~\ref{lem:main3} is nearly identical to that in Lemma~\ref{lem:main1}, as long as we can ensure ${\Norm{\w_t-\w_{t-1}}^2}/{\beta_t} = \mathcal{O}(1)$. This condition can be satisfied by using: (i) a normalization technique, i.e., $\w_{t+1} = \w_{t} - \eta \z_t/ \Norm{\z_t}$, with $\eta^2/\beta_t \leq \mathcal{O}(1)$; or (ii) projection operation, i.e., $\w_{t+1} = \w_{t} - \eta_t \Pi_{C_F}\left[\z_t\right]$, with $\eta_t^2/\beta_t \leq \mathcal{O}(1)$.

\section{MSVR Estimator for Non-convex FCCO Problem}
This section introduces algorithms for solving non-convex FCCO problems. We first introduce notation and assumptions, and then present the proposed methods together with their sample complexity guarantees.
\subsection{Notations and Assumptions}\label{assumption}
Let $[m]=\{1,\ldots, m\}$. We measure the efficiency of stochastic algorithms via the (stochastic oracle) sample complexity defined below.
\begin{definition}
The sample complexity is the number of stochastic oracle queries required to find a point that satisfies $\E \left[\Norm{\nabla F(\w)}\right] \leq \epsilon$ ($\epsilon$-stationary point) or $\E \left[ F(\w)-\inf_{\w} F(\w)\right] \leq \epsilon$ ($\epsilon$-optimal solution).
\end{definition}
We adopt the following assumptions, which are commonly used in the SCO literature and variance reduction analyses
\cite{wang2016accelerating,wang2017stochastic,Yuan2019EfficientSN,Zhang2019ASC,Zhang2021MultiLevelCS,guo2022stochastic}.
\begin{ass}\label{ass1} 
We assume that each $f_i$ is $L_f$-smooth and $C_f$-Lipschitz continuous; each $g_i$ is $L_g$-smooth and $C_g$-Lipschitz continuous; function $F$ is lower bounded by $F_{*}$. 
\end{ass}
\textbf{Remark:} This implies that $F$ is $C_F$-Lipschitz continuous and $L_F$-smooth, where $C_F= C_f C_g$ and $L_F = C_f^2L_g + C_g^2L_f$~\cite{Zhang2021MultiLevelCS}.
\begin{ass}\label{ass2}  (Bounded variance)
\begin{equation*}
\begin{split}
\mathbb{E}\left[g_i(\x;\xi_t^i)\right] = g_i(\x); 
	\quad&\mathbb{E}\left[\nabla g_i(\x;\xi_t^i)\right] = \nabla g_i(\x); \\
\mathbb{E}\left[\left\|g_{i}\left(\mathbf{x} ; \xi_t^{i}\right)-g_{i}(\mathbf{x})\right\|^{2}\right] \leq \sigma^{2}; 
\quad&\mathbb{E}\left[\left\|\nabla g_{i}\left(\mathbf{x} ; \xi_t^{i}\right)-\nabla g_{i}(\mathbf{x})\right\|^{2}\right] \leq \sigma^{2}.
\end{split}
\end{equation*} 
\end{ass}
%\begin{ass}\label{ass3} (Smoothness and Lipschitz continuity)
%\begin{equation*}
%\begin{split}
%\mathbb{E}\left[\left\|g_{i}\left(\mathbf{x} \right)-g_{i}\left(\mathbf{y} \right)\right\|^{2}\right]  \leq C_g^{2}\|\mathbf{x}-\mathbf{y}\|^{2};\\
%\mathbb{E}\left[\left\|\nabla g_{i}\left(\mathbf{x} \right)-\nabla g_{i}\left(\mathbf{y} \right)\right\|^{2}\right] \leq L_g^{2}\|\mathbf{x}-\mathbf{y}\|^{2}.
%\end{split}
%\end{equation*} 
%\end{ass}
\begin{ass}\label{ass4} (Average smoothness and Lipschitz continuity)
\begin{equation*}
\begin{split}
\mathbb{E}\left[\left\|g_{i}\left(\mathbf{x} ; \xi_t^{i}\right)-g_{i}\left(\mathbf{y} ; \xi_t^{i}\right)\right\|^{2}\right]  \leq C_g^{2}\|\mathbf{x}-\mathbf{y}\|^{2};\\
\mathbb{E}\left[\left\|\nabla g_{i}\left(\mathbf{x} ; \xi_t^{i}\right)-\nabla g_{i}\left(\mathbf{y} ; \xi_t^{i}\right)\right\|^{2}\right] \leq L_g^{2}\|\mathbf{x}-\mathbf{y}\|^{2}.
\end{split}
\end{equation*} 
\end{ass}

%===================   Algorithm 0  ========================
\begin{algorithm}[tb] 
	\caption{MSVRM-v1 and MSVRM-v2 method}
	\label{alg:0}
	\begin{algorithmic}[1]
	\STATE {\bfseries Input:} number of iterations $T$, parameters $\alpha_t$,$\beta_t$,$\gamma_t$,$\eta_t$ and initial points $(\w_1,\u_1,\z_1)$.
		\FOR{time step $t = 1$ {\bfseries to} $T$}
		\STATE Sample a subset $\mathcal{B}_{1}^{t}$ from $\{1,2,\cdots,m \}$ 
        \STATE{Compute estimator $\u_t$ according to  equation~(\ref{MSVR}) or (\ref{MSVR-SP})}% \ \hfill$\diamond$ Use MSVR or MSVR-SP update
        \STATE (v1) Compute estimator $\z_t$ according to equation~(\ref{momentum})%\hfill$\diamond$ Use moving average update
         \STATE (v2) Compute estimator $\z_t$ according to equation~(\ref{eqn:z}) %\hfill$\diamond$ Use STORM update
        \STATE $\w_{t+1} = \w_t - \eta_t \z_{t}$
		\ENDFOR
	\STATE Choose $\tau$ uniformly at random from $\{1, \ldots, T\}$
	\STATE Return ($\w_{\tau}$, $\u_{\tau}$, $\z_{\tau}$)
	\end{algorithmic}
\end{algorithm}
\subsection{The Proposed Method}
We now present our proposed Multi-block-Single-probe Variance Reduction Method (MSVRM) for solving problem~\eqref{p:1}. 
To begin with, for each iteration $t$, we use the proposed MSVR or MSVR-SP estimator $\u_t$ to approximate the inner function. 
Then, following prior work~\cite{wang2021momentum,dependent2022}, we use the moving average estimator $\z_t$ to track the gradient as:
\begin{equation} \label{momentum}
        \z_{t} = (1-\alpha_t)\z_{t-1} +  \frac{\alpha_t}{B_1}\sum_{i \in \mathcal{B}_{1}^{t}}  \nabla f_{i}(\u_{t-1}^{i}) \nabla g_{i}(\w_t;\xi_t^{i}),
\end{equation}
Note that we use the estimator from the previous step, i.e., $\nabla f_i(\u_{t-1}^i)$, rather than $\nabla f_i(\u_{t}^i)$ in the equation. This is to avoid statistical dependencies on the random variable $\xi_t^{i}$ used to compute $\u_{t}^i$, which may lead to dependent issues otherwise. 
Finally, we update the model parameters using the estimated gradient $\z_{t}$. The whole algorithm is presented in Algorithm~\ref{alg:0}, named MSVRM-v1.

We first provide a theoretical guarantee for MSVRM-v1 with the momentum-based MSVR estimator ($\gamma_t=0$) and adaptive learning rates.
\begin{thm} \label{new:1}
For our MSVRM-v1 algorithm, set that $\gamma_t = 0$, $\alpha_t = \sqrt{\frac{B_1}{T}}$ and $\eta_t = \frac{\eta}{\Norm{\z_t}}$. Under Assumptions~\ref{ass1} and \ref{ass2}, we have the following guarantees.    
\begin{compactitem}
    \item By choosing $\beta_t = \sqrt{\frac{m }{B_1 T}}$ and $\eta = \frac{B_1^{1/4}}{m^{1/4}T^{3/4}} $, we find an $\epsilon$-stationary point with a complexity of $  \mathcal{O}\left(m \epsilon^{-4} \right)$.
    \item If function $\nabla f_i$ is linear, by setting $\beta_t = 1$ and $\eta = \frac{B_1^{1/4}}{T^{3/4}}$, we find an $\epsilon$-stationary point with a complexity of $  \mathcal{O}\left(\epsilon^{-4} \right)$.
\end{compactitem}
\end{thm}
\textbf{Remark:} The lower bound for stochastic non-convex optimization under standard smoothness assumptions is $\Omega(\epsilon^{-4})$~\cite {Arjevani2019LowerBF}, indicating that our $\mathcal{O}(\epsilon^{-4})$ complexity is optimal in this setting.

While this rate is optimal under standard smoothness, it can be improved by leveraging the average smoothness (Assumption~\ref{ass4}) and our full MSVR estimator (with $\gamma_t \neq 0$).
%===================   Theorem 1   ========================
\begin{thm} \label{thm:1}
For MSVRM-v1 method, set $\eta_t = \min\left\{\left(\frac{B_1}{m}\right)^{2/3}(a+t)^{-1/3},\sqrt{B_1}(a+t)^{-1/2}\right\}$, $\alpha_{t+1} = \O\left( \eta_t \right)$,  $\beta_{t+1} = \O( \frac{m^2 \eta_t^2}{B_1^2})$, $\gamma_{t} = \frac{m-B_1}{B_1(1-\beta_{t})}+(1-\beta_{t})$,  and $a = \O(\frac{m}{B_1})$. 
Under Assumptions~\ref{ass1}, \ref{ass2} and \ref{ass4}, we can find an $\epsilon$-stationary point with a complexity of $  \mathcal{O}\left( {m \epsilon^{-3}}+{\epsilon^{-4}} \right)$.
\end{thm}
%=================== MSVR-v2  ========================

However, the complexity of MSVRM-v1 is still on the order of $\O(\epsilon^{-4})$. Due to the biased nature of the estimated gradient, using the simple moving average update for $\z_t$ is insufficient to achieve the SOTA complexity of $\O(\epsilon^{-3})$. 
To overcome this, we introduce the MSVRM-v2 method, which uses a STORM-like~\cite{cutkosky2019momentum} update for $\z_t$:
\begin{equation}
\begin{split} \label{eqn:z}
    \z_{t} = &(1-\alpha_t)\z_{t-1} + \frac{1}{B_1}\sum_{i \in \mathcal{B}_{1}^{t}}  \nabla f_{i}(\u_{t-1}^{i}) \nabla g_{i}(\w_t;\xi_t^{i})   \\
    & - (1-\alpha_t)\frac{1}{B_1}\sum_{i \in \mathcal{B}_{1}^{t}}  \nabla f_{i}(\u_{t-2}^{i}) \nabla g_{i}(\w_{t-1};\xi_t^{i}).
\end{split}
\end{equation}
Now, we show this new method (MSVRM-v2) can achieve an improved complexity of $\O(\epsilon^{-3})$.
%===================   Theorem  2   ========================
\begin{thm}\label{thm:2}
Under Assumptions~\ref{ass1}, \ref{ass2} and \ref{ass4}, MSVRM-v2 with $\gamma_{t} = \frac{m-B_1}{B_1(1-\beta_{t})}+(1-\beta_{t})$, $\alpha_{t+1} = \O(\frac{m \eta_{t}^2}{B_1})$, $a=O(\frac{m }{B_1}$), $\eta_t = \O\left((\frac{B_1 }{m})^{2/3}(a+t)^{-1/3} \right)$, and $\beta_{t+1} = \O\left(\frac{m^2 \eta_t^2}{B_1^2}\right)$,  can find an $\epsilon$-stationary point with a sample complexity of $ \mathcal{O}\left({m  \epsilon^{-3}} \right)$.
\end{thm}
\textbf{Remark:} The proposed algorithms can also incorporate Adam-style learning rates while preserving the same complexities. The details are provided in the appendix.

The analysis in Theorems~\ref{thm:1} and \ref{thm:2} require problem-dependent constants (e.g., $L_g$, $L_f$, $C_f$, $C_g$) to set hyperparameters. 
We can overcome this limitation by using adaptive step sizes. Additionally, we establish improved rates for linear $\nabla f_i$ in the following theorem.
\begin{thm} \label{new:2}
For MSVRM-v2 method, set that $\eta_t = \frac{\eta}{\Norm{\z_t}}$. Under Assumptions~\ref{ass1}, \ref{ass2} and \ref{ass4}, we have the following guarantees.
\begin{compactitem}
    \item By using that $\gamma_{t} = \frac{m-B_1}{B_1(1-\beta)}+(1-\beta)$, $\alpha_t = \frac{m^{2/3}B_1^{1/3}}{T^{2/3}}$, $\beta_t = \frac{m^{2/3}}{B_1^{2/3}T^{2/3}}$ and $\eta = \frac{B_1^{1/3}}{m^{1/3}T^{2/3}}$, we find an $\epsilon$-stationary point with a complexity of $  \mathcal{O}\left(m \epsilon^{-3} \right)$.
    \item If function $\nabla f_i$ is linear, by setting that $\gamma_t=0$, $\beta_t = 1$, $\alpha_t = \frac{B_1^{1/3}}{T^{2/3}}$, $\eta = \frac{B_1^{1/3}}{T^{2/3}}$, we find an $\epsilon$-stationary point with a complexity of $\mathcal{O}\left(\epsilon^{-3} \right)$.
\end{compactitem}
\end{thm}
\textbf{Remark:} Compared with Theorem~\ref{thm:2}, this result does not require knowledge of problem-dependent parameters. Moreover, it removes the dependence on $m$ when $\nabla f_i$ is linear.

\noindent\textbf{Remark:} When $m=1$ and $f$ is the identity function, the problem reduces to standard stochastic  optimization. The $\Omega\left({\epsilon^{-3}}\right)$ lower bound for this class (average smoothness )~\cite{Arjevani2019LowerBF} indicates that our MSVRM-v2 complexity is optimal.

%===================   Section   ======================== 
\section{Improved Rates for the Finite-sum Structure}
In this section, we exploit the finite-sum setting, in which $g_i(\w) = \frac{1}{n} \sum_{j=1}^{n} g_i(\w; \xi_{ij})$. This structure enables us to compute the exact value of $g_i(\w)$ periodically and thus yields faster rates. Accordingly, we modify the MSVR estimator to leverage the finite-sum structure.

Inspired by SVRG~\cite{NIPS2013_ac1dd209,NIPS:2013:Zhang}, we compute the exact inner function values every $I$ iterations at a snapshot point $\mathbf{w}_{\tau}$ (where $\tau \bmod I=0$), i.e., $g_i\left(\mathbf{w}_{\tau}\right) = \frac{1}{n} \sum_{j=1}^{n} g_{i}(\w_{\tau}; \xi_{ij})$ for all $i = \{1, \ldots, m\}$. Then, at each step $t$, we replace the raw stochastic sample $g_i(\w_t;\xi_t^i)$ in MSVR with the following unbiased estimator
\begin{equation*} 
    \begin{split}
    \widehat g_{i}^t =   g_{i}(\w_t; \xi_t^{i}) -   g_{i}(\w_\tau; \xi_t^{i}) + g_{i}(\w_\tau).
    \end{split}
\end{equation*}
In this way, our MSVR estimator is modified as follows:
\begin{equation} \label{MSVR_SVRG}
\begin{split}
    \mathbf{u}_{t}^{i}=\left\{\begin{array}{ll}
\widehat \u_t^i \quad \quad \quad  i \in \mathcal{B}_{1}^{t}\\
\mathbf{u}_{t-1}^{i}     \quad  \quad  i \notin \mathcal{B}_{1}^{t}
\end{array}\right.,
\end{split}
\end{equation}
where $\widehat \u_t^i$ is defined as
\begin{align*}
    \widehat \u_t^i = (1-\beta)\mathbf{u}_{t-1}^{i}+\beta \widehat g_i^t+ \gamma \left( g_i\left(\mathbf{w}_{t}; \xi_t^{i}\right)-g_i\left(\mathbf{w}_{t-1}; \xi_t^{i}\right) \right).
\end{align*}
In the original MSVR estimator, the term $\beta g_{i}(\w_t; \xi_t^{i})$ would contribute an error bounded by $\beta^2\E[\Norm{g_{i}(\w_t; \xi_t^{i})-g_i(\w_t)}^2] \leq \beta^2 \sigma^2$. By using the newly designed term $\beta \widehat{g}_{i}^t$ instead, this becomes $\beta^2\E[\Norm{\widehat{g}_{i}^t-g_i(\w_t)}^2] \leq \beta^2 C_g^2 \Norm{\w_t - \w_\tau}^2$. Moreover, since $\sum_{t=1}^T \Norm{\w_t - \w_\tau}^2 \leq I^2 \sum_{t=1}^T \Norm{\w_t - \w_{t-1}}^2$, the new error term can be controlled by an appropriate choice of $\beta$, $\eta$, and $I$. This modification leads to the following guarantee for the function estimator.
%===================   Lemma  2 ========================
\begin{lemma}\label{lem:main2}
By setting that $\gamma = \frac{m-B_1}{B_1(1-\beta)}+(1-\beta)$, $\beta \leq 1/2$ and $\beta I \leq {m}/{B_1}$, we have:
\begin{equation*}
\begin{split}
&\frac{1}{T} \sum_{t=1}^T \E \left[ \Norm{\u_{t} - g(\w_{t})}^2 \right] \leq  \frac{15 m^3C_g^2}{B_1^2 \beta T}  \sum_{t=1}^T\E\left[ \Norm{\w_{t+1} - \w_{t}}^2\right]
\end{split}
\end{equation*}
\end{lemma}
\textbf{Remark: } Compared with Lemma~\ref{lem:main1}, we remove the $2 B_1 \beta^2 \sigma^2$ term. This is the key to improving the complexity, as it allows us to use a much larger $\beta$.

To achieve the optimal complexity, we apply a similar correction to the gradient estimator:
\begin{equation} \label{eqn:z+}
\begin{split} 
    \z_{t} = &(1-\alpha)\z_{t-1} +  \alpha \h_t + (1-\alpha)\frac{1}{B_1}\sum_{i \in \mathcal{B}_{1}^{t}}  \nabla f_{i}(\u_{t-1}^{i}) \nabla g_{i}(\w_t;\xi_t^{i}) \\
    & - (1-\alpha)\frac{1}{B_1}\sum_{i \in \mathcal{B}_{1}^{t}}  \nabla f_{i}(\u_{t-2}^{i}) \nabla g_{i}(\w_{t-1};\xi_t^{i}),
\end{split}
\end{equation}
where $\h_t$ involves both the full gradient at the snapshot and the stochastic gradient:
\begin{equation*}
\begin{split}
    &\h_t = \frac{1}{m}\sum_{i=1}^{m}  \nabla f_{i}(\u_{\tau-1}^{i}) \nabla g_{i}(\w_{\tau})+\frac{1}{B_1}\sum_{i \in \mathcal{B}_{1}^{t}}  (\nabla f_{i}(\u_{t-1}^{i}) \nabla g_{i}(\w_t;\xi_t^{i}) -  \nabla f_{i}(\u_{\tau-1}^{i}) \nabla g_{i}(\w_{\tau};\xi_t^{i})).
\end{split}
\end{equation*}
This method is summarized in Algorithm~\ref{alg:2} (named as MSVRM-v3). Next, we show that MSVRM-v3 is equipped with an optimal complexity of $\O(\sqrt{n}\epsilon^{-2})$.
%===================   Algorithm 2  ========================
\begin{algorithm}[t]	
	\caption{MSVRM-v3 method}
	\label{alg:2}
	\begin{algorithmic}[1]
	\STATE {\bfseries Input:} number of iterations $T$, parameters $\alpha$,$\beta$,$\gamma$,$I$,$\eta$ and initial points $(\w_1,\u_1,\z_1)$.
		\FOR{time step $t = 1$ {\bfseries to} $T$}
		\IF{$t \mod I ==0$}
        \STATE Set $\tau = t$ 
        \STATE Compute and save  $g_i(\w_\tau), \nabla f_i(\u^i_{\tau-1})$ for all $i$
		\ENDIF
        \STATE Sample a subset $\mathcal{B}_{1}^{t}$ from $\{1,2,\cdots,m \}$
		\STATE{Compute estimator $\u_t$ according to equation~(\ref{MSVR_SVRG})}
        \STATE{Compute estimator $\z_t$ according to equation~(\ref{eqn:z+})}
        \STATE $\w_{t+1} = \w_t - \eta \z_{t}$
		\ENDFOR
	\STATE Choose $\tau$ uniformly at random from $\{1, \ldots, T\}$	
 \STATE Return ($\w_{\tau}$, $\u_{\tau}$, $\z_{\tau}$)
	\end{algorithmic}
\end{algorithm}
%===================   Theorem  3   ========================
\begin{thm} \label{thm:3}
Under Assumptions~\ref{ass1}, \ref{ass2} and \ref{ass4}, our MSVRM-v3 algorithm with $I =\O\left(\frac{mn}{B_1 }\right)$, $\alpha = \O\left(\frac{ B_{1} }{m n }\right)$, $\beta = \O\left(\frac{ 1 }{n}\right)$ and $\eta = \mathcal{O}\left(\frac{B_1 }{m \sqrt{n}} \right)$, finds an $\epsilon$-stationary point with a sample complexity of $ \mathcal{O}\left( {m \sqrt{n}} \epsilon^{-2} \right)$.
\end{thm} 
\textbf{Remark:} When $m=B_1=1$ and $f$ is the identity function, problem~(\ref{p:2}) reduces to standard finite-sum optimization, i.e., $\min \frac{1}{n} \sum_{j=1}^{n} g_{i}(\mathbf{w}; \xi_{j})$. The lower bound for this setting is $\Omega\left(\sqrt{n}\epsilon^{-2}\right)$~\cite{Fang2018SPIDERNN,pmlr-v139-li21a}, indicating that our rate is optimal in its dependence on $n$ and $\epsilon$.

Next, we obtain the same rate with adaptive hyperparameters and attain an improved rate in the case where $\nabla f_i$ is linear.
\begin{thm} \label{new:3}
For our MSVRM-v3 algorithm, set that $\eta_t = \frac{\eta}{\Norm{\z_t}}$, $I=\frac{mn}{B_1}$, $\alpha = \frac{B_1}{mn}$, $\eta = \frac{B_1^{1/2}}{m^{1/4}n^{1/4}T^{1/2}}$. Under Assumptions~\ref{ass1}, \ref{ass2} and \ref{ass4}, we have the following guarantees.
\begin{compactitem}
    \item Choosing $\gamma = \frac{m-B_1}{B_1(1-\beta)}+(1-\beta)$ and $\beta = \frac{1}{n}$, we find an $\epsilon$-stationary point with a complexity of $  \mathcal{O}\left(m \sqrt{n}\epsilon^{-2} \right)$.
    \item When $\nabla f_i$ is linear, choosing $\widehat \u_t^i =  g_i\left(\mathbf{w}_{t}; \xi_t^{i}\right)$, we find an $\epsilon$-stationary point with a complexity of $  \mathcal{O}\left(\sqrt{mn}\epsilon^{-2} \right)$.
\end{compactitem}
\end{thm}
\textbf{Remark:} Theorem~\ref{new:3} removes reliance on problem-dependent parameters and improves the dependence on $m$ for linear $\nabla f_i$. This $\mathcal{O}(\sqrt{mn})$ reliance is known to be optimal in the finite-sum settings~\cite{Fang2018SPIDERNN,pmlr-v139-li21a}.

\section{Improved rates for convex and PL objectives}
Next, we demonstrate that our complexities can be further improved when the objective function $F(\cdot)$ is convex or satisfies the Polyak-Łojasiewicz (PL) condition. We first give out the definition of PL condition as follows.
\begin{definition}
$F(\w)$ satisfies the $\mu$-PL condition if there exists $\mu > 0$ such that for all $\w$:
\begin{equation*}
    2 \mu\left(F(\mathbf{w})-F_{*}\right) \leq\|\nabla F(\mathbf{w})\|^{2}. 
\end{equation*}
\end{definition}
\textbf{Remark:} Since $\mu$-strong convexity implies the $\mu$-PL condition~\cite{Karimi2016LinearCO}, all results under the $\mu$-PL condition apply directly to the $\mu$-strongly convex objectives.

For convex or PL objectives, we employ a stage-wise reduction scheme, which runs a base MSVRM method for multiple stages. In the new algorithm, we decrease $\alpha_s$ and $\beta_s$ after stage $s$ and increase the number of iterations $T_s$. At the end of each stage, we save the output and use it to restart the next stage. With these modifications, we can ensure that both the gradient estimation error and the optimal gap can be reduced after each stage, leading to better overall convergence guarantees. This new method is summarized in Algorithm~\ref{alg:3}, named Stage-wise MSVRM.

\begin{algorithm}[tb]
	\caption{Stage-wise MSVRM method}
	\label{alg:3}
	\begin{algorithmic}
	\STATE {\bfseries Input:} initial points $\left(\w_0,\u_0,\z_0\right)$
		\FOR{stage $s = 1$ {\bfseries to} $S$}
		\STATE $\w_{s},\u_{s},\z_{s}= \text{MSVRM } (\text{with } T_{s}, \alpha_s, \beta_s, \gamma_s, \eta_{s}, \left(\w_{s-1},\u_{s-1},\z_{s-1}\right))$
		\ENDFOR
	\STATE Return $\w_{S}$
	\end{algorithmic}
\end{algorithm}

To illustrate the effect of this design, we present the following lemma for the stage-wise MSVRM-v2 method with linear $\nabla f_i$. The other settings are similar to this configuration.
\begin{lemma}
    Suppose the function satisfies the $\mu$-PL condition, and  $\nabla f_i$ is linear. Assume we have $\E[F(\w_{s-1}) - F_* ]\leq \epsilon_{s-1}$ and $\E[\Norm{\z_{s-1}-\nabla F(\w_{s-1})}^2] \leq \mu \epsilon_{s-1}$ for stage $s-1$. By setting $\alpha_s = \mathcal{O}(B_1 \mu \epsilon_s)$, $\eta_s = \mathcal{O}(B_1 \sqrt{\mu \epsilon_s})$, $T_s = \mathcal{O}(1/(B_1 \mu \epsilon_s))$, $\gamma_s = 0$ and $\beta_s =1$, we ensure
    \begin{align*}
        \E[F(\w_{s}) - F_* ]\leq \epsilon_{s} = \epsilon_{s-1}/2, \quad \E[\Norm{\z_{s}-\nabla F(\w_{s})}^2] \leq \mu \epsilon_{s}= \mu\epsilon_{s-1}/2.
    \end{align*}
\end{lemma}
\textbf{Remark:} This lemma shows that the optimal gap $\E[F(\w_{s}) - F_* ]$ and the gradient estimation error $\E[\Norm{\z_{s}-\nabla F(\w_{s})}^2]$ are both halved after each stage.

Applying the above stage-wise design yields the following complexity results.
\begin{thm}\label{thm:4}
	The stage-wise MSVRM-v1, under Assumptions~\ref{ass1}, \ref{ass2} and \ref{ass4}, derives a complexity of $\O(m\epsilon^{-3})$ for convex functions and  $\O(m\mu^{-2}\epsilon^{-1})$ for functions satisfying the $\mu$-PL condition.
\end{thm}
\begin{thm}\label{thm:4+}
	 The stage-wise MSVRM-v2, under Assumptions~\ref{ass1}, \ref{ass2} and \ref{ass4}, derives a complexity of $ \mathcal{O}\left(m \epsilon^{-2} \right)$ for convex functions and $ \mathcal{O}\left({m}\mu^{-1} \epsilon^{-1} \right)$ for functions satisfying $\mu$-PL objectives. 
\end{thm}
Similarly, a better complexity can be obtained for the finite-sum structure.% under the convexity or PL condition.
\begin{thm}\label{thm:6}
    The stage-wise MSVRM-v3, under Assumptions~\ref{ass1}, \ref{ass2} and \ref{ass4}, attains a complexity of $\mathcal{O}\left(\frac{m \sqrt{n} }{\epsilon B_1  } \log{\frac{1}{\epsilon}} \right)$ for convex objectives and $ \mathcal{O}\left(\frac{m \sqrt{n} }{\mu B_1   } \log{\frac{1}{\epsilon}} \right)$ for the $\mu$-PL condition functions.
\end{thm}

Additionally, we improve the dependence on $m$ when the gradients $\nabla f_i(\cdot)$ are linear.
\begin{thm}\label{new:4}
	For stage-wise MSVRM-v1, under Assumptions~\ref{ass1} and \ref{ass2}, when $\nabla f_i$ is linear, it derives a complexity of $\O(\epsilon^{-3})$ for convex functions and $\O(\mu^{-2}\epsilon^{-1})$ for the $\mu$-PL functions. 
\end{thm}
\begin{thm}\label{new:5}
 For stage-wise MSVRM-v2, under Assumptions~\ref{ass1}, \ref{ass2} and \ref{ass4}, supposing that function $\nabla f_i$ is linear, it derives a complexity of $\O(\epsilon^{-2})$ for convex functions and $\O(\mu^{-1}\epsilon^{-1})$ for functions satisfying the $\mu$-PL condition.
\end{thm}
\textbf{Remark:} 
The above complexities are optimal, matching the $\Omega\left(\epsilon^{-2}\right)$ and $\Omega\left(\mu^{-1}\epsilon^{-1}\right)$ lower bounds for stochastic convex and strongly convex optimization~\cite{Agarwal2012InformationTheoreticLB}.
\begin{thm}\label{new:6}
    For stage-wise MSVRM-v3, under Assumptions~\ref{ass1}, \ref{ass2} and \ref{ass4}, supposing that function $\nabla f_i$ is linear, it derives a sample complexity of $\mathcal{O}\left({ \sqrt{mn} \epsilon^{-1}} \log{\frac{1}{\epsilon}} \right)$ for convex functions and $ \mathcal{O}\left({ \sqrt{mn} \mu^{-1}} \log{\frac{1}{\epsilon}} \right)$ for functions satisfying the $\mu$-PL condition.
\end{thm}
\textbf{Remark:} For the finite-sum PL functions, we achieve a linear convergence rate $\mathcal{O}\left( \log(1/\epsilon) \right)$, matching the SOTA results for single-level finite-sum problems~\cite{pmlr-v139-li21a}.

%===================   Section   ======================== 
\section{Experiments}\label{sec:exp}
In this section, we evaluate our methods on multi-task deep AUC maximization. We first introduce the setup and present the numerical results. Then, we provide an ablation study on our estimator design, followed by experiments with different network and batch sizes.
\subsection{Multi-task Deep AUC Maximization}
For binary classification with labels $y \in \{1,-1\}$, AUC maximization can be formulated as minimizing the following composite objective~\cite{AUROC2022}:
\begin{equation*}
    \begin{split}
    \min _{\mathbf{w}, a, b} \mathbb{E}_{\mathbf{x} \mid y=1}\left[(h_{\mathbf{w}}(\mathbf{x})-a)^{2}\right]&+\mathbb{E}_{\mathbf{x}^{\prime} \mid y^{\prime}=-1}\left[\left(h_{\mathbf{w}}\left(\mathbf{x}^{\prime}\right)-b\right)^{2}\right] +\ell(a(\mathbf{w})-b(\mathbf{w})),
    \end{split}
\end{equation*}
where $\ell(\cdot)$ is a surrogate loss function, and function $a(\mathbf{w})=\mathbb{E}\left[h_{\mathbf{w}}(\mathbf{x}) \mid y=1\right]$, $b(\mathbf{w})=\mathbb{E}\left[h_{\mathbf{w}}(\mathbf{x}) \mid y=-1\right]$. 
The above objective function can recover the pairwise square loss and the min-max margin loss proposed by~\cite{robustdeepAUC} for deep AUC maximization when choosing $\ell(\cdot)$ as the square function or squared hinge loss.  
\begin{figure*}[t]
	\centering
	\subfigure[STL10]{
		\includegraphics[width=0.22\textwidth]{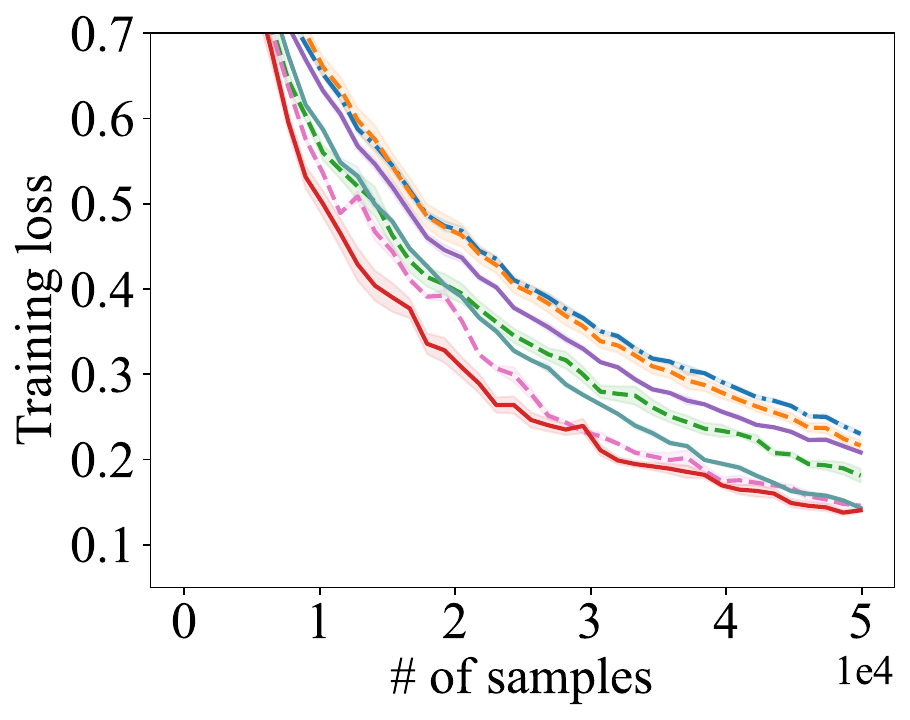}
	}	
 \subfigure[USPS]{
		\includegraphics[width=0.22\textwidth]{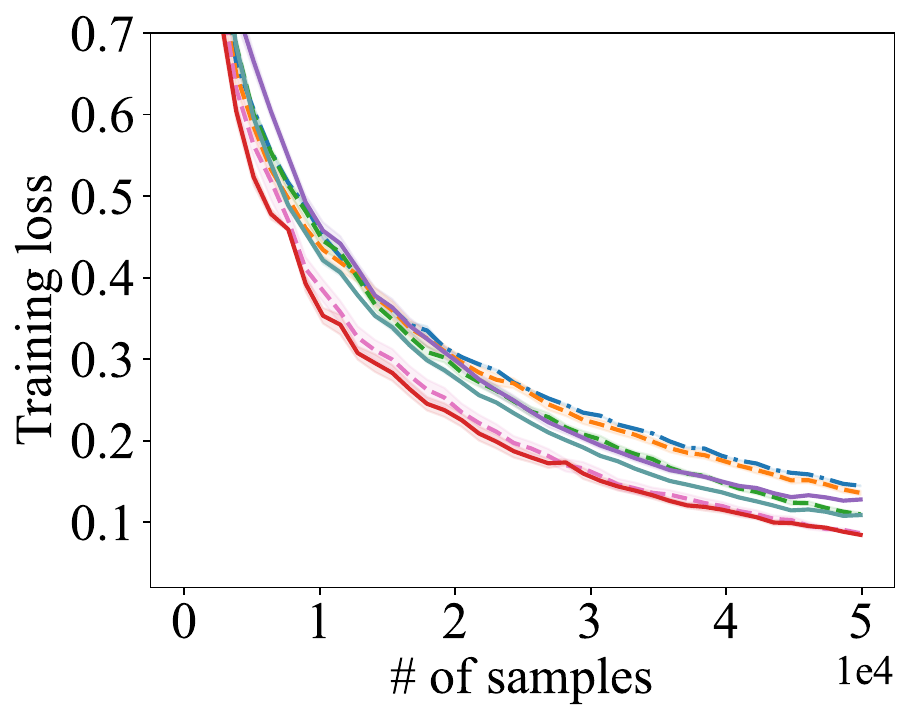}
	}
	\subfigure[CIFAR10]{
		\includegraphics[width=0.22\textwidth]{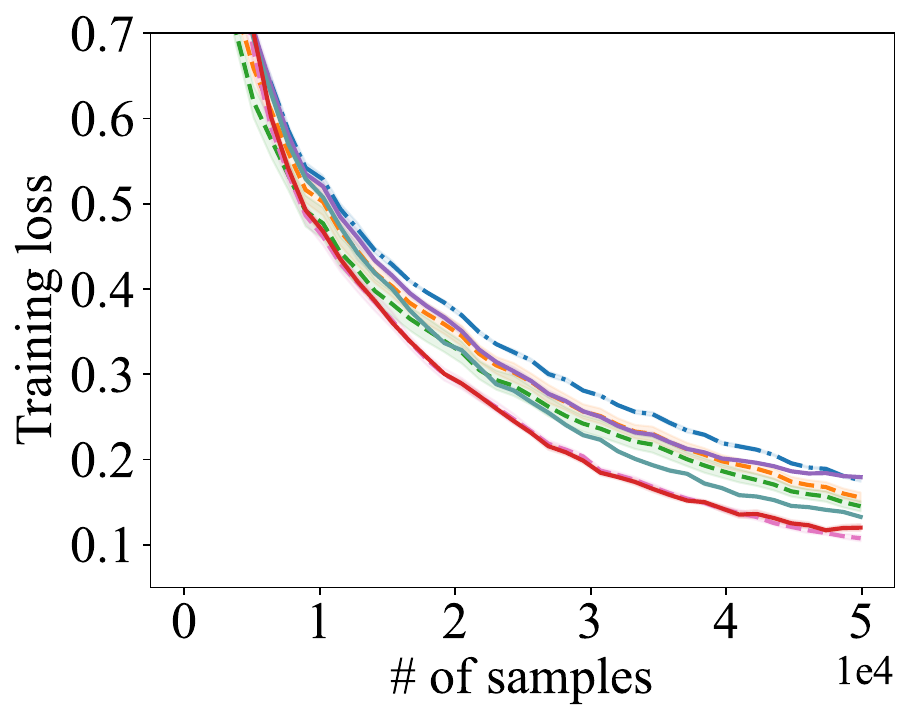}
	}
	\subfigure[CIFAR100]{
		\includegraphics[width=0.22\textwidth]{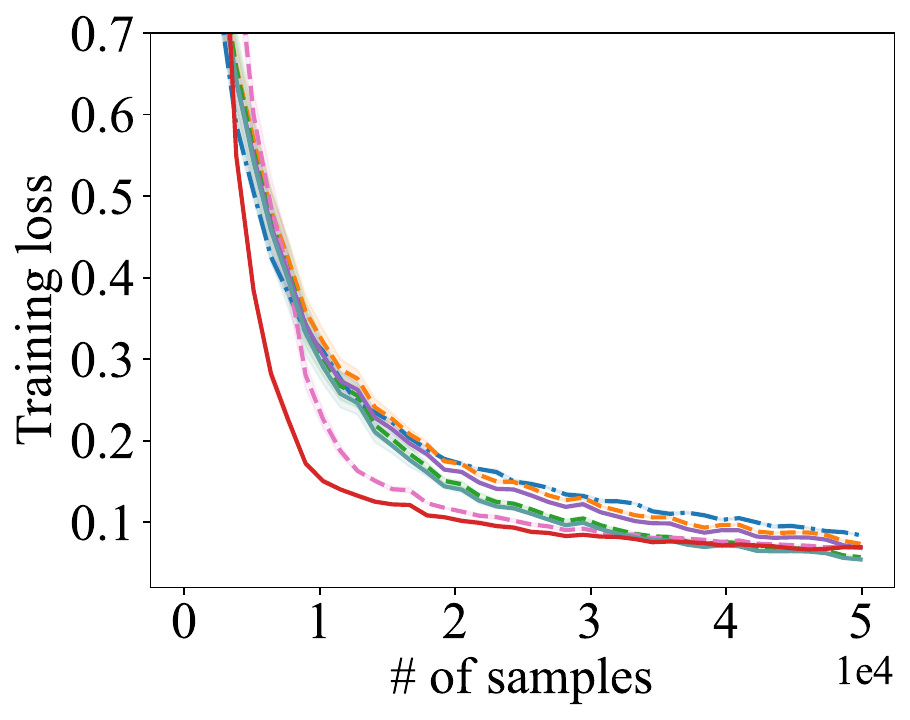}
	}
	\subfigure[MNIST]{
		\includegraphics[width=0.22\textwidth]{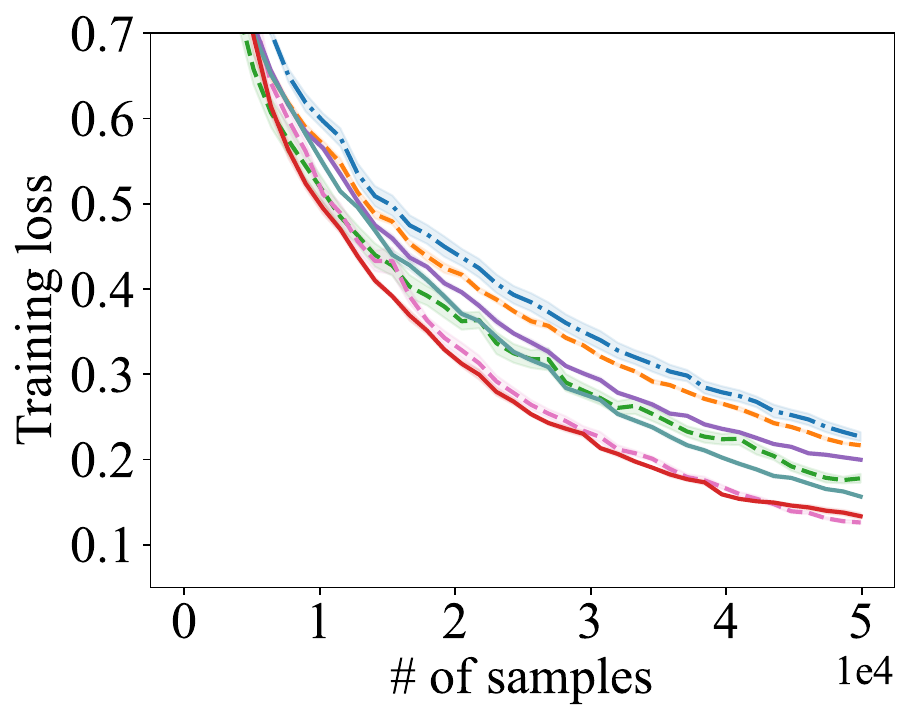}
	}
     \subfigure[KMNIST]{
		\includegraphics[width=0.22\textwidth]{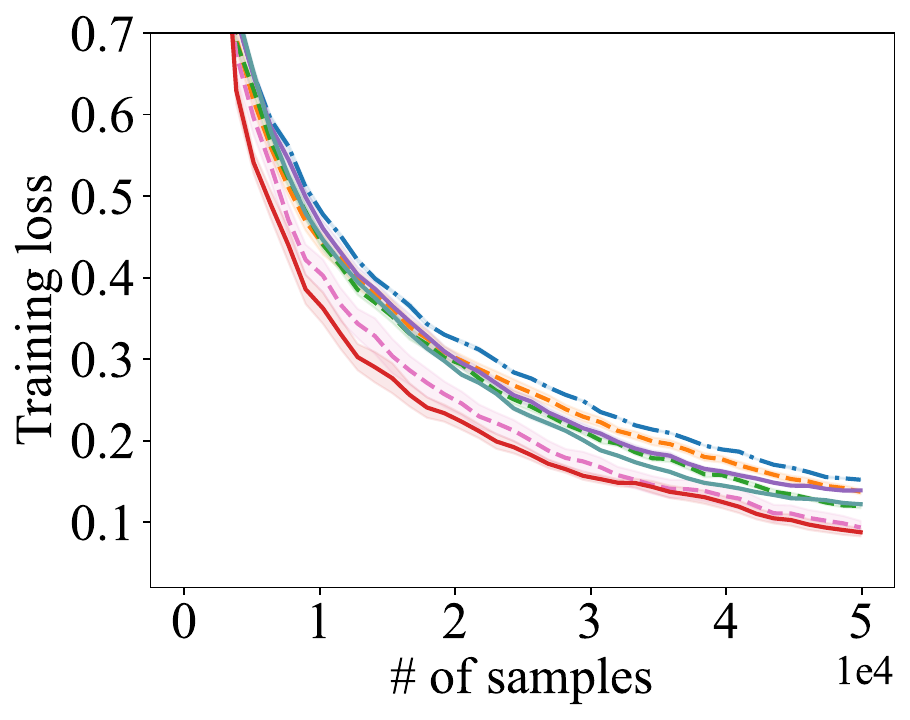}
	}
	\subfigure[Fashion-MNIST]{
		\includegraphics[width=0.22\textwidth]{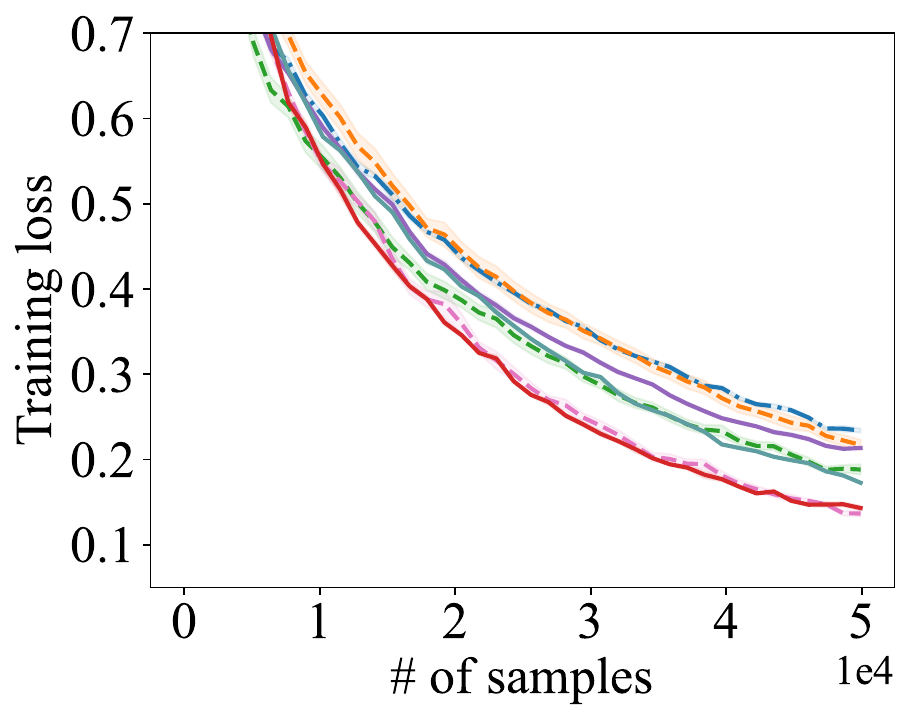}
	}
	\subfigure[SVHN]{
		\includegraphics[width=0.22\textwidth]{./figure/MulticlassAUC/svhn_loss.pdf}
	}
	\includegraphics[width=0.99\textwidth]{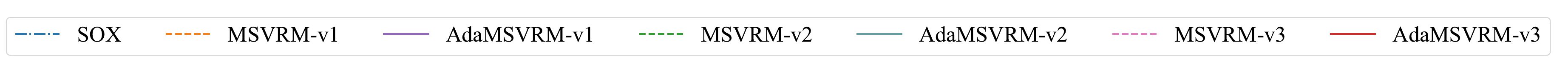}
	\caption{Training loss versus the number of  samples for Multi-task AUC Optimization.}
	\label{fig:1}
	%\vskip -0.1in
\end{figure*}

For multi-class or multi-task classification problems, we optimize the average AUC loss over $m$ different tasks, i.e., $AUC = \frac{1}{m} \sum_{i=1}^m AUC(i)$. 
The coupled compositional structure arises from the $\ell(a(\mathbf{w})-b(\mathbf{w}))$ term, and the objective fits the FCCO formulation by defining:
\begin{equation*}
    \begin{split}
        g_{i}(\mathbf{w})&=\frac{1}{|\mathcal{D}^i_{+}|} \sum_{\mathbf{x} \in \mathcal{D}^i_{+}} h_{\mathbf{w}}(\mathbf{x})- \frac{1}{|\mathcal{D}^i_{-}|} \sum_{\mathbf{x} \in \mathcal{D}^i_{-}} h_{\mathbf{w}}(\mathbf{x}),\quad 
        f\left(g_{i}(\mathbf{w})\right)=\ell(g_{i}(\mathbf{w})).
    \end{split}
\end{equation*} 
where $\mathcal D^i_{+}$ and $\mathcal D^i_{-}$ denote the positive and negative subsets of data for the $i$-th task.
        
\noindent{\bf Experimental setup.} We follow the setup in literature~\cite{AUROC2022} and use the squared hinge loss $\ell(x) =\frac{1}{2}(\max\{c+x,0\})^{2}$ as the surrogate function.
We employ ResNet18 as the network backbone and train on eight different datasets: STL10~\cite{Coates2011STL10}, USPS~\cite{uspsdataset}, CIFAR10~\cite{Krizhevsky2009Cifar10}, CIFAR100~\cite{Krizhevsky2009Cifar10}, MNIST~\cite{LeCun1998MNIST}, KMNIST~\cite{clanuwat2018deep}, Fashion-MNIST~\cite{Xiao2017Fashion-MNIST}, and SVHN~\cite{Netzer2011SVHN}.
We compare the proposed methods with the SOX baseline~\cite{dependent2022}. For our methods, hyper-parameters $\alpha$ and $\beta$ are tuned from $\{0.1, 0.5, 0.9, 1.0\}$. 
For SOX, its corresponding parameters $\beta$ and $\gamma$ are tuned from the same set.  
We set the number of probed blocks to $B_1=50$ for CIFAR100 and $B_1=5$ for the other datasets.
The inner batch size for sampling $\xi_t^i$ is fixed to $128$ for all methods. The learning rate for each method is tuned from $\{1e-4, 1e-3, 2e-3, 5e-3, 1e-2\}$. All experiments are conducted on a single NVIDIA Tesla M40 GPU.
\begin{table*}[t]
\centering
\caption{Final test AUC on Multi-task AUC Optimization.}
\label{tab:auc}
\begin{tabular}{lcccc}
\toprule
 & \textbf{MNIST} & \textbf{FASHION MNIST} & \textbf{SVHN} & \textbf{STL10}  \\
\midrule
SOX & 0.9905 & 0.9907 & 0.9755 & 0.9711  \\
MSVRM-v1 & 0.9926 & 0.9929 & 0.9811 & 0.9714 \\
AdaMSVRM-v1 & 0.9948 & 0.9957 & 0.9821 & 0.9774  \\
MSVRM-v2 & 0.9941 & 0.9951 & 0.9867 & 0.9830  \\
AdaMSVRM-v2 & 0.9958 & 0.9955 & 0.9839 & 0.9857  \\
MSVRM-v3 & 0.9947 & 0.9925 & 0.9930 & 0.9929  \\
AdaMSVRM-v3 & \textbf{0.9978} & \textbf{0.9966} & \textbf{0.9938} & \textbf{0.9961}\\
\midrule
\midrule
 &  \textbf{CIFAR10} & \textbf{CIFAR100} & \textbf{USPS} & \textbf{KMNIST} \\
\midrule
SOX & 0.9837 & 0.9960 & 0.9894 & 0.9942 \\
MSVRM-v1  & 0.9869 & 0.9971 & 0.9893 & 0.9951 \\
AdaMSVRM-v1  & 0.9897 & 0.9985 & 0.9958 & 0.9973 \\
MSVRM-v2  & 0.9921 & 0.9992 & 0.9949 & 0.9972 \\
AdaMSVRM-v2 & 0.9927 & 0.9989 & 0.9988 & 0.9987 \\
MSVRM-v3 & \textbf{0.9964} & \textbf{0.9998} & 0.9979 & 0.9976 \\
AdaMSVRM-v3  & 0.9958 & 0.9995 & \textbf{0.9998} & \textbf{0.9997} \\
\bottomrule
\end{tabular}
\end{table*}

\noindent{\bf Results.} Figure~\ref{fig:1} plots the training loss versus the number of samples, averaged over 5 runs. We observe that MSVRM-v1 performs comparably to or slightly better than SOX, and MSVRM-v2 consistently converges faster than both SOX and MSVRM-v1. MSVRM-v3 shows the fastest convergence, decreasing the loss most rapidly and demonstrating its superior efficiency, which aligns with our theoretical findings. For adaptive methods~(AdaMSVRM-v1,v2,v3), they perform slightly better than the non-adaptive counterpart in most cases, validating their effectiveness in practice. 

We also report the final test AUC~(Area Under the Curve) and AP~(Average Precision) in Tables~\ref{tab:auc} and \ref{tab:ap}. Generally, (Ada)MSVRM-v1 obtains better results than SOX, and (Ada)MSVRM-v2 is still better than (Ada)MSVRM-v1. Among all these methods, the (Ada)MSVRM-v3 algorithm achieves the highest AUC and AP in almost all tasks.

\begin{table*}[t]
\centering
\caption{Final test AP on Multi-task AUC Optimization.}
\label{tab:ap}
\begin{tabular}{lcccccccc}
\toprule
 & \textbf{MNIST} & \textbf{FASHION MNIST} & \textbf{SVHN} & \textbf{STL10}  \\
\midrule
SOX & 0.9191 & 0.9141 & 0.8367 & 0.7977 \\
MSVRM-v1 & 0.9332 & 0.9489 & 0.8688 & 0.7989  \\
AdaMSVRM-v1 & 0.9699 & 0.9507 & 0.8760 & 0.8392\\
MSVRM-v2 & 0.9546 & \textbf{0.9528} & 0.9107 & 0.8668  \\
AdaMSVRM-v2 & 0.9689 & 0.9523 & 0.9115 & 0.8896 \\
MSVRM-v3 & 0.9468 & 0.9503 & 0.9118 & 0.9256  \\
AdaMSVRM-v3 & \textbf{0.9709} & 0.9439 & \textbf{0.9221} & \textbf{0.9404} \\
\midrule
\midrule
 & \textbf{CIFAR10} & \textbf{CIFAR100} & \textbf{USPS} & \textbf{KMNIST} \\
\midrule
SOX  & 0.8853 & 0.8375 & 0.9041 & 0.9486 \\
MSVRM-v1 & 0.9024 & 0.8672 & 0.9274 & 0.9526 \\
AdaMSVRM-v1 & 0.9174 & 0.9113 & 0.9783 & 0.9621 \\
MSVRM-v2 & 0.9424 & 0.9439 & 0.9589 & 0.9797 \\
AdaMSVRM-v2 & 0.9485 & 0.9413 & 0.9879 & 0.9881 \\
MSVRM-v3 & \textbf{0.9625} & \textbf{0.9858} & 0.9677 & 0.9697 \\
AdaMSVRM-v3 & 0.9621 & 0.9610 & \textbf{0.9883} & \textbf{0.9981} \\
\bottomrule
\end{tabular}
\end{table*}

\subsection{Ablation Study on Algorithm Design}
In this subsection, we conduct an ablation study to verify the effect of our customized error correction term. To do this, we compare against a variant that replaces our MSVR estimator with a delicate application of the STORM estimator. This variant uses the update:
\begin{equation*}    \begin{split}
        \mathbf{u}_{t}^{i}= \begin{cases}(1-\beta) \mathbf{u}_{t-1}^{i}+\beta \frac{m}{B_{1}} g_{i}\left(\mathbf{w}_{t} ; \xi_{t}^{i}\right)+(1-\beta) \frac{m}{B_{1}}\left(g_{i}\left(\mathbf{w}_{t} ; \xi_{t}^{i}\right)-g_{i}\left(\mathbf{w}_{t-1} ; \xi_{t}^{i}\right)\right) & i \in \mathcal{B}_{1}^{t} \\ (1-\beta) \mathbf{u}_{t-1}^{i} & i \notin \mathcal{B}_{1}^{t}\end{cases}
    \end{split}
\end{equation*}
Replacing the MSVR estimator in MSVRM-v1 and MSVRM-v2 yields Variant-v1 and Variant-v2, respectively.
For the finite-sum case, we modify the estimator similarly:
\begin{equation*}
    \begin{split}
        \mathbf{u}_{t}^{i}= \begin{cases}(1-\beta) \mathbf{u}_{t-1}^{i}+\beta \frac{m}{B_{1}} \hat{g}_{i}\left(\mathbf{w}_{t} ; \xi_{t}^{i}\right)+(1-\beta) \frac{m}{B_{1}}\left(g_{i}\left(\mathbf{w}_{t} ; \xi_{t}^{i}\right)-g_{i}\left(\mathbf{w}_{t-1} ; \xi_{t}^{i}\right)\right) & i \in \mathcal{B}_{1}^{t} \\ (1-\beta) \mathbf{u}_{t-1}^{i} & i \notin \mathcal{B}_{1}^{t}\end{cases}
    \end{split}
\end{equation*}
where $\widehat g_{i}(\w_t; \xi_t^{i}) =   g_{i}(\w_t; \xi_t^{i}) -   g_{i}(\w_\tau; \xi_t^{i}) + g_{i}(\w_\tau)$. We replace the MSVR estimator in MSVRM-v3 with the above equation and name this new method Variant-v3.
\begin{figure*}[t]
	\centering
	\subfigure[MSVRM-v1 vs Variant-v1]{
		\includegraphics[width=0.3\textwidth]{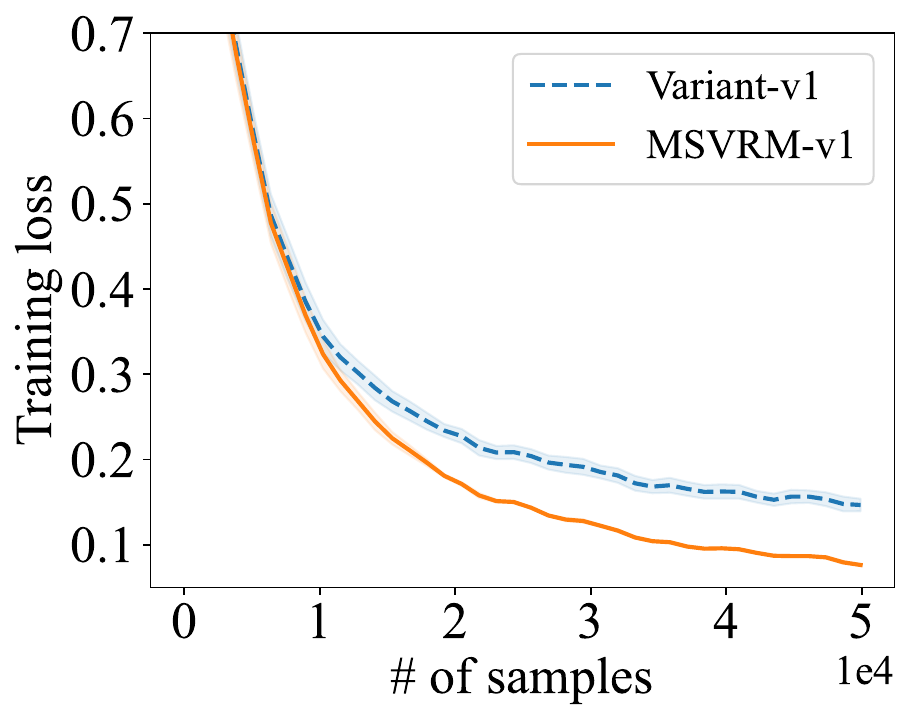}
	}
	\subfigure[MSVRM-v2 vs Variant-v2]{
		\includegraphics[width=0.3\textwidth]{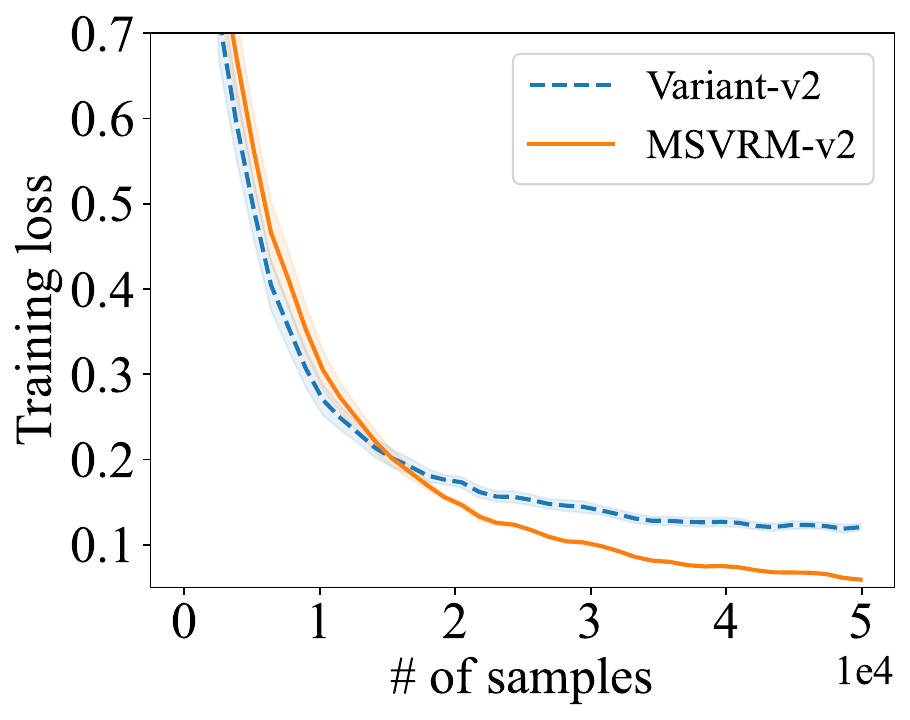}
	}
	%\vspace{-4mm}
	\subfigure[MSVRM-v3 vs Variant-v3]{
		\includegraphics[width=0.3\textwidth]{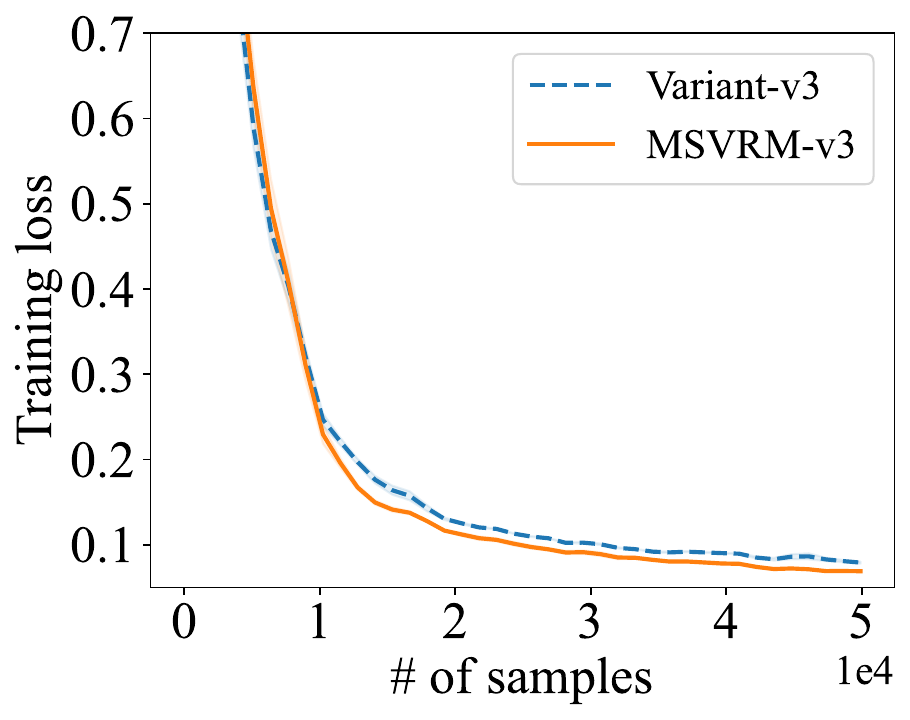}
	}
	\caption{Ablation study that compares our proposed SVRM estimator against variants using the STORM correction.}
	\label{fig:2}
\end{figure*}

\begin{figure*}[t]
    \centering
    \subfigure[ResNet18]{
        \includegraphics[width=0.3\textwidth]{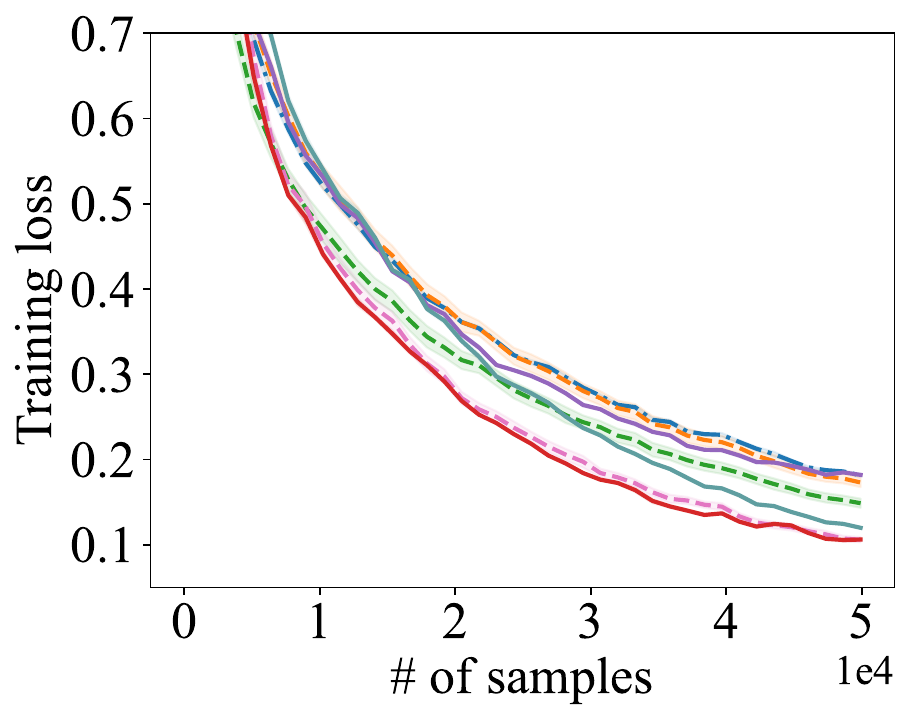}
    }
    %\vspace{-4mm}
    \subfigure[ResNet34]{
        \includegraphics[width=0.3\textwidth]{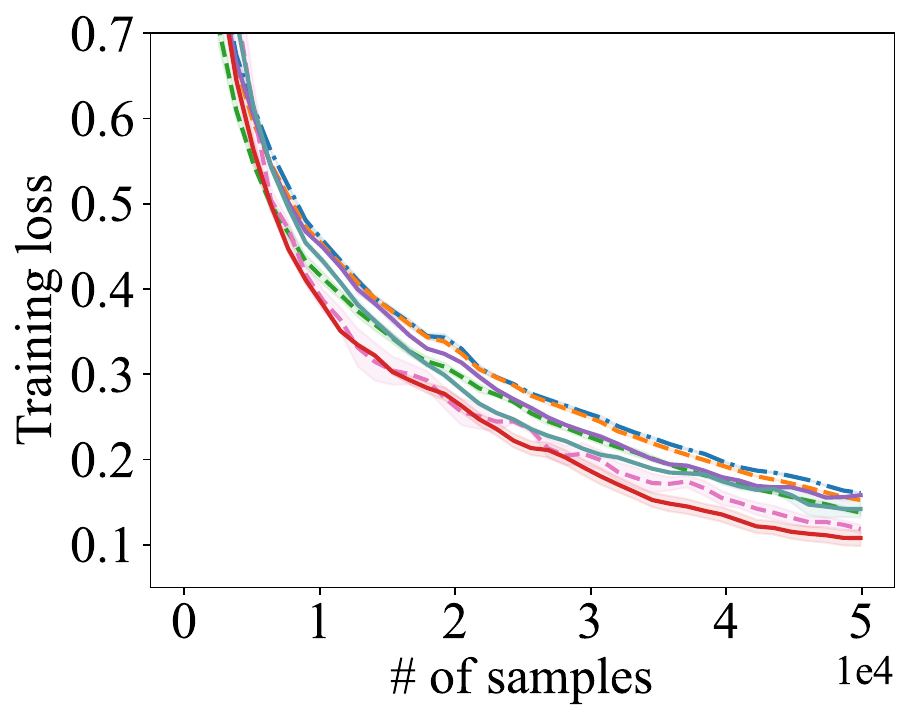}
    }
    %\vspace{-4mm}
    \subfigure[DenseNet121]{
        \includegraphics[width=0.3\textwidth]{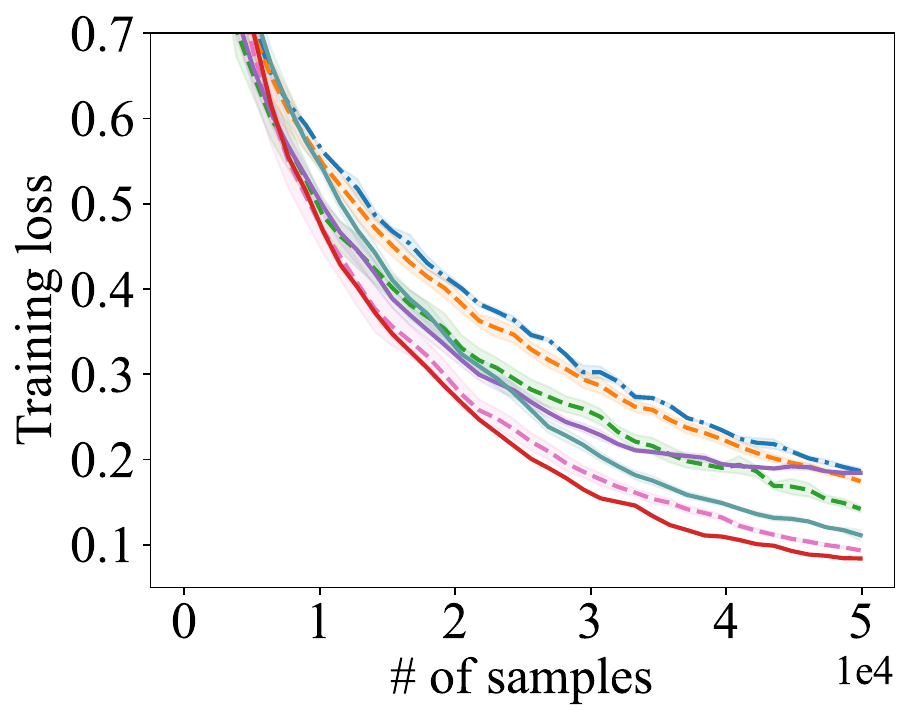}
    }
    \includegraphics[width=0.95\textwidth]{./figure/MulticlassAUC/legend.pdf}
    \caption{Results with different network architectures.}
    \label{fig:3}
\end{figure*}
\noindent{\bf Results.} We compare our proposed algorithms (MSVRM-v1, v2, v3) against their variant counterparts on CIFAR100. The results are reported in Figure~\ref{fig:2}. As can be seen, all three variant methods perform worse than our original algorithms. This empirically confirms the importance of our customized error correction term, which properly handles the dual sources of randomness.

\begin{figure*}[t]
	\centering
	\subfigure[MSVRM-v1]{
		\includegraphics[width=0.3\textwidth]{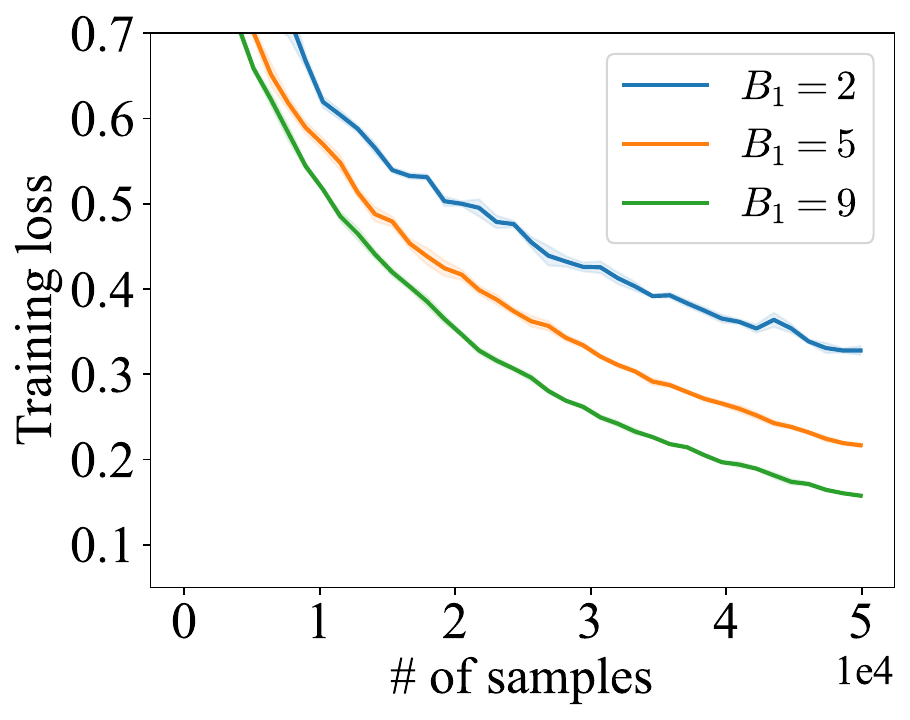}
	}
	%\vspace{-4mm}
	\subfigure[MSVRM-v2]{
		\includegraphics[width=0.3\textwidth]{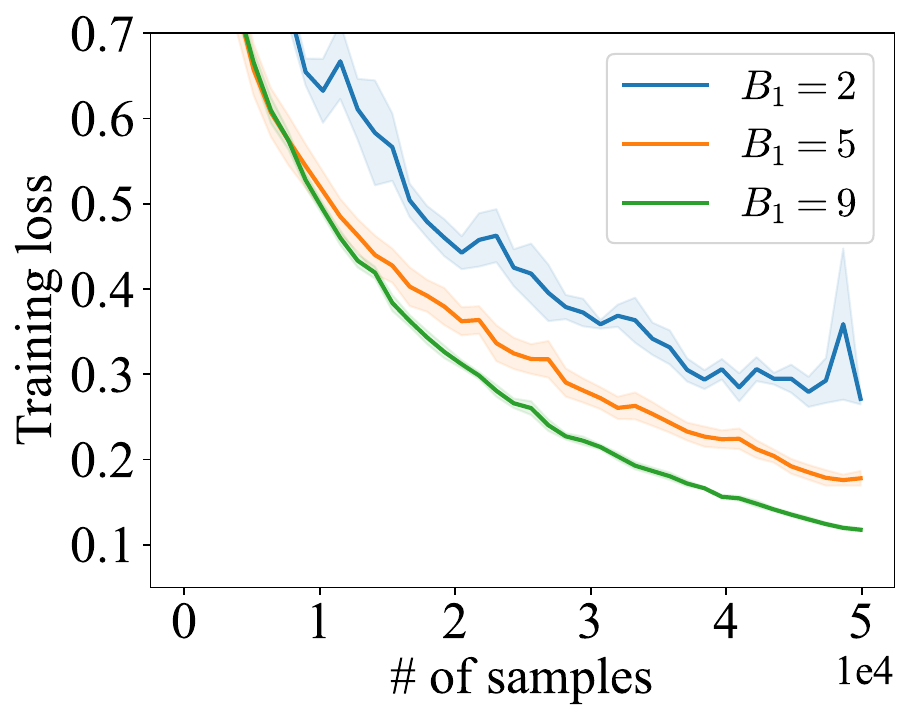}
	}
	%\vspace{-4mm}
	\subfigure[MSVRM-v3]{
		\includegraphics[width=0.3\textwidth]{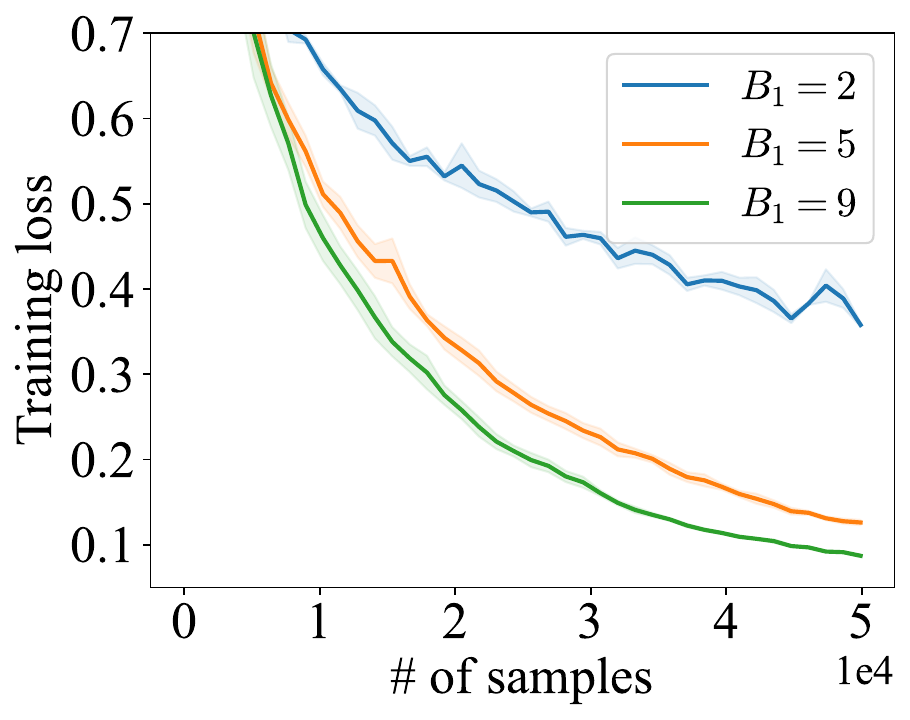}
	}
 	\subfigure[AdaMSVRM-v1]{
		\includegraphics[width=0.3\textwidth]{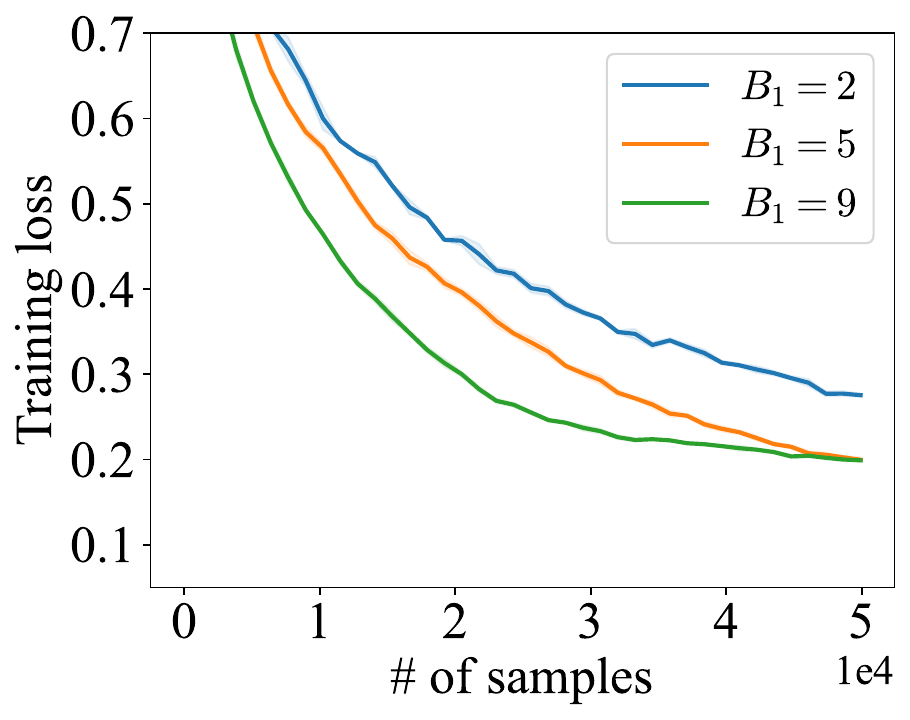}
	}
	%\vspace{-4mm}
	\subfigure[AdaMSVRM-v2]{
		\includegraphics[width=0.3\textwidth]{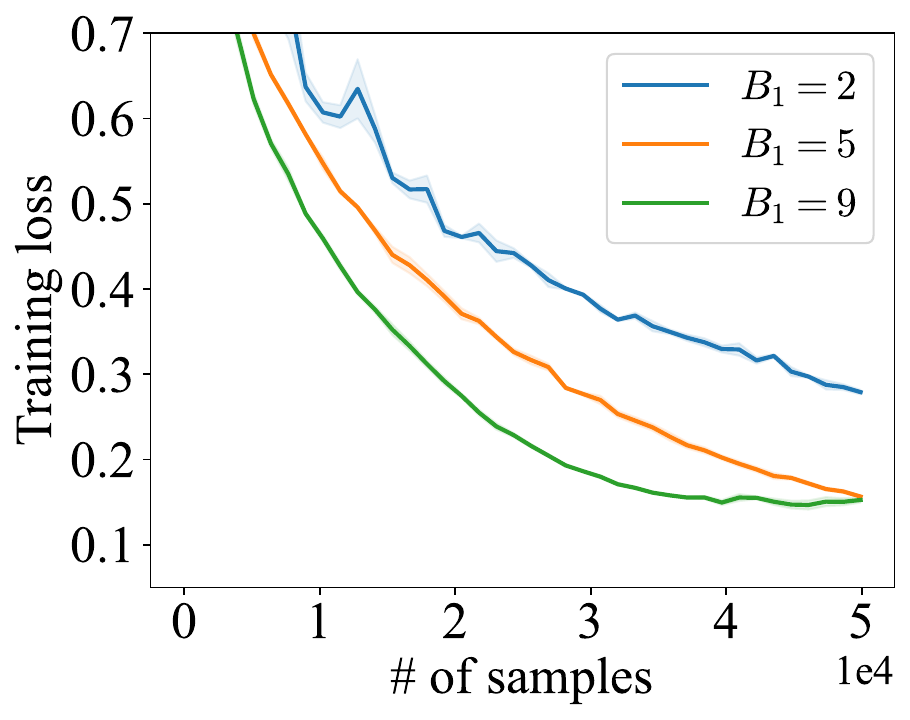}
	}
	%\vspace{-4mm}
	\subfigure[AdaMSVRM-v3]{
		\includegraphics[width=0.3\textwidth]{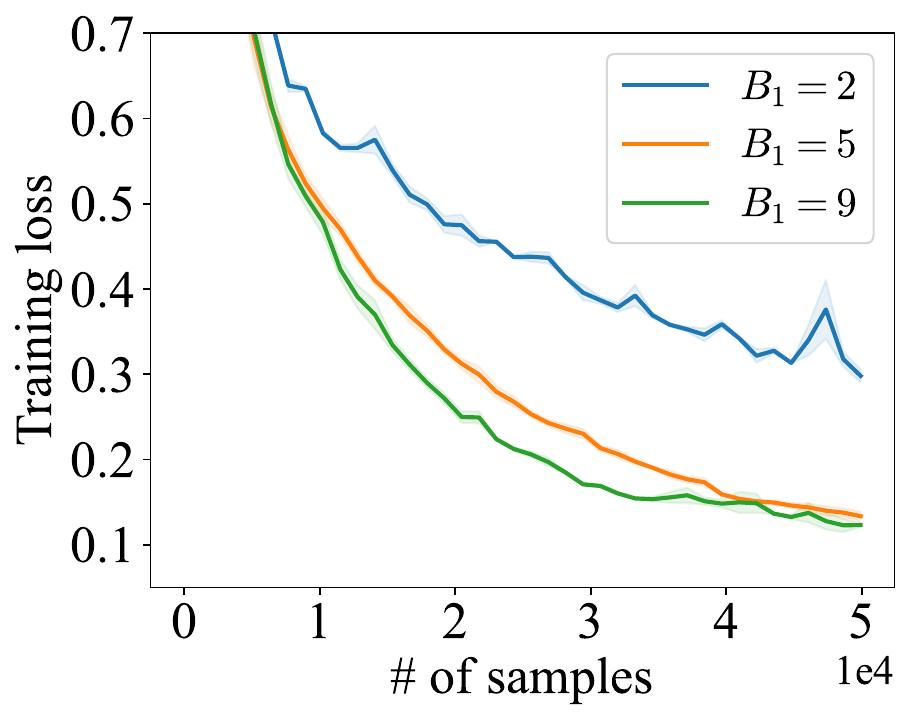}
	}
	\caption{Results with varying outer batch size $B_1$.}
	\label{fig:4}
\end{figure*}

\subsection{Results with Different Networks and Batch Sizes}
\subsubsection{Different networks} 
We first test the robustness of our methods across different network architectures: ResNet18, ResNet34, and DenseNet121, using the SVHN dataset. The results in Figure~\ref{fig:3} show that the relative performance remains consistent: (Ada)MSVRM-v1 is comparable to SOX, (Ada)MSVRM-v2 is faster, and (Ada)MSVRM-v3 is the fastest. This demonstrates the robustness of our conclusion across various architectures.

\subsubsection{Different Batch Sizes}
Then, we explore the effect of different batch sizes. Specifically, we vary the number of probed blocks $B_1$ in the range $\{2,5,9\}$. We conduct the experiments on the Fashion-MNIST data set and show the results in Figure~\ref{fig:4}. As can be seen, a larger batch size $B_1$ would improve the convergence speed of the algorithm.

\section{Conclusion}
In this paper, we propose the novel Multi-block Single-probe Variance Reduction~(MSVR) estimator for tracking a vector of $m$ function mappings where only $\mathcal{O}(1)$ blocks can be probed per iteration. 
Building on this MSVR estimator, we develop new algorithms for solving Finite-sum Coupled Compositional Optimization (FCCO) problems. We establish SOTA sample complexities across a spectrum of settings, including non-convex, convex, strongly convex and PL objectives, for both stochastic and finite-sum oracles. We further improve the reliance on $m$ when the outer gradient $\nabla f_i$ is a linear function.
Empirical results on multi-task deep AUC maximization also validate our theory and demonstrate clear advantages over existing approaches.

\vskip 0.2in
\bibliography{ref}

\newpage
\appendix

\section{Proof of Lemma~1}
%===================   Lemma  1   ========================
By setting $\gamma_t = 0$ and $\beta_t = \beta$, we have that $\bar{\u}_{t+1}^i=(1-\beta) \u_{t}^{i}+\beta g_{i}(\w_{t+1};\xi_{t+1}^{i})$ and
\begin{equation*}
    \begin{split}
        &\E \left[ \Norm{\u_{t+1}^i - g_i(\w_{t+1})}^2 \right]
        = (1-\frac{B_1}{m})\E \left[\Norm{ \u_{t}^i - g_i(\w_{t+1})}^2\right] + \frac{B_1}{m} \E \left[\Norm{\bar{\u}_{t+1}^i-g_{i}(\w_{t+1})}^2  \right]\\
        \leq & (1-\frac{B_1}{m})(1+\frac{\beta B_1}{2m})\E \left[\Norm{ \u_{t}^i - g_i(\w_{t})}^2\right] +(1-\frac{B_1}{m})(1+\frac{2m}{\beta B_1})\E \left[\Norm{ g_i(\w_{t}) - g_i(\w_{t+1})}^2\right]\\
        &\quad + \frac{B_1}{m} \E \left[\Norm{\bar{\u}_{t+1}^i-g_{i}(\w_{t+1})}^2  \right]\\
        \leq & (1-\frac{B_1}{m}+\frac{\beta B_1}{2m})\E \left[\Norm{ \u_{t}^i - g_i(\w_{t})}^2\right] +\frac{3m C_g^2}{\beta B_1}\E \left[\left\| \w_{t+1} -\w_{t}\right\|^2\right]+ \frac{B_1}{m} \E \left[\Norm{\bar{\u}_{t+1}^i-g_{i}(\w_{t+1})}^2  \right]
        \end{split}
\end{equation*}
For the term $\E \left[\Norm{\bar{\u}_{t+1}^i-g_{i}(\w_{t+1})}^2  \right]$, we can decompose it as:
\begin{equation}
    \begin{split}
        &\E \left[  \Norm{\bar{\u}_{t+1}^i-g_{i}(\w_{t+1})}^2  \right] \\
        = & \E \left[  \left\|(1-\beta) \left(\u_{t}^{i} - g_i(\w_t)\right) + (1-\beta) \left(g_i(\w_t) - g_i(\w_{t+1})\right) + \beta \left(g_{i}(\w_{t+1};\xi_{t+1}^{i}) - g_i(\w_{t+1})\right)\right\|^2  \right] \\
        =  & \E \left[  \left\|(1-\beta) \left(\u_{t}^{i} - g_i(\w_t) + g_i(\w_t) - g_i(\w_{t+1})\right) \right\|^2\right] +\beta^2 \E\left[\Norm{ g_{i}(\w_{t+1};\xi_{t+1}^{i}) - g_i(\w_{t+1})}^2\right] \\
        \leq  &   (1-\beta)^2(1+\beta) \E \left[ \left\| \u_{t}^{i} - g_i(\w_t)\right\|^2 \right] + (1-\beta)^2(1+\frac{1}{\beta}) \E \left[ \left\| g_i(\w_t) - g_i(\w_{t+1})\right\|^2 \right] +\beta^2 \sigma^2\\
        \leq  &   (1-\beta) \E \left[ \left\| \u_{t}^{i} - g_i(\w_t)\right\|^2 \right] + \frac{2}{\beta} \E \left[ \left\| g_i(\w_t) - g_i(\w_{t+1})\right\|^2 \right] +\beta^2 \sigma^2\\
        \leq  &   (1-\beta) \E \left[ \left\| \u_{t}^{i} - g_i(\w_t)\right\|^2 \right] + \frac{2C_g^2}{\beta} \E \left[ \left\| \w_t - \w_{t+1}\right\|^2 \right] +\beta^2 \sigma^2
        \end{split} 
\end{equation}
 The second equation is due to $\E \left[g_{i}(\w_{t+1};\xi_{t+1}^{i}) - g_i(\w_{t+1})\right] = 0$. Summing up, we have that
 \begin{equation*}
    \begin{split}
        &\E \left[ \Norm{\u_{t+1}^i - g_i(\w_{t+1})}^2 \right] \\
        \leq & (1-\frac{B_1}{m}+\frac{\beta B_1}{2m})\E \left[\Norm{ \u_{t}^i - g_i(\w_{t})}^2\right] +\frac{3m C_g^2}{\beta B_1}\E \left[\left\| \w_{t+1} -\w_{t}\right\|^2\right]+ \frac{B_1}{m} \E \left[\Norm{\bar{\u}_{t+1}^i-g_{i}(\w_{t+1})}^2  \right]\\
        \leq & (1-\frac{B_1}{m}+\frac{\beta B_1}{2m})\E \left[\Norm{ \u_{t}^i - g_i(\w_{t})}^2\right] +\frac{3m C_g^2}{\beta B_1}\E \left[\left\| \w_{t+1} -\w_{t}\right\|^2\right]\\
        &\quad + \frac{B_1}{m} (1-\beta) \E \left[ \left\| \u_{t}^{i} - g_i(\w_t)\right\|^2 \right] + \frac{2C_g^2 B_1}{\beta m}\E \left[\left\| \w_{t+1} -\w_{t}\right\|^2\right]  +\frac{ B_1}{m }\beta^2\sigma^2 \\
        \leq & (1-\frac{\beta B_1}{2m})\E \left[\Norm{ \u_{t}^i - g_i(\w_{t})}^2\right] +\frac{5m C_g^2}{\beta B_1}\E \left[\left\| \w_{t+1} -\w_{t}\right\|^2\right]  + \frac{ B_1}{m }\beta^2\sigma^2
        \end{split}
\end{equation*}
Finally,  we have:
\begin{equation*}
    \begin{split}
        &\E \left[ \Norm{\u_{t+1} - g(\w_{t+1})}^2 \right]
        = \sum_{i=1}^{m}\E \left[\Norm{ \u_{t+1}^i - g_i(\w_{t+1})}^2 \right] \\
        \leq & \left(1-\frac{\beta B_1}{2m} \right)\E \left[\Norm{ \u_{t} - g(\w_{t})}^2 \right] +\frac{5 m^2 C_g^2}{\beta B_1} \E \left[\left\| \w_{t+1} -\w_{t}\right\|^2\right] + B_1 \sigma^2 \beta^2 
        \end{split}
\end{equation*}

\section{Proof of Lemma~3}
Let us focus on a fixed $i\in[m]$. Then we have
\begin{align*}
   \E\left[ \|\u^{i}_{t}  - g_{i}(\w_t)\|^2\right]
   = \frac{B_1}{m}\underbrace{\E\left[ \|\widetilde \u^{i}_{t}  - g_{i}(\w_t)\|^2\right]}\limits_{A_1} +  (1-\frac{B_1}{m})\underbrace{\E\left[ \|\u^{i}_{t-1}  - g_{i}(\w_t)\|^2\right]}\limits_{A_2}.
\end{align*}
We first decompose the second term as follows:  
\begin{align*}
A_{2}=&\E\left[\|\u^i_{t-1} - g_i(\w_{t-1}) + g_i(\w_{t-1}) - g_i(\w_t)\|^2 \right]\\
    = &\E[\|\u^i_{t-1} - g_i(\w_{t-1})\|^2  + \|g_i(\w_{t-1}) - g_i(\w_t)\|^2 + \underbrace{2(\u^i_{t-1} - g_i(\w_{t-1}))^{\top}(g_i(\w_{t-1}) - g_i(\w_t))}\limits_{A_{21}}]\\
    \leq & \E \left[\Norm{ \u_{t-1}^i - g_i(\w_{t-1})}^2 + C_g^2 \Norm{ \w_{t-1} - \w_{t}}^2  + A_{21}\right]
\end{align*}
Then we decompose the first term as 
\begin{align*}
    A_1 = &\E[\|(1-\beta_{t})(\u^i_{t-1} - g_i(\w_{t-1})) +  \gamma_{t}^0 (g_i(\w_t) - g_i(\w_{t-1})) + \beta_{t} (g_i(\w_t; \xi_t^{i}) - g_i(\w_t))\\ &+\gamma_{t}( \nabla g_i(\w_{t};\xi_{t}^i)^{\top}(\w_{t} - \w_{t-1}) - g_i(\w_t) + g_i(\w_{t-1})) \|^2]\\
    &\E[\|\underbrace{(1-\beta_{t})(\u^i_{t-1} - g_i(\w_{t-1}))}\limits_{A_{11}} +  \underbrace{\gamma_{t}^0 (g_i(\w_t) - g_i(\w_{t-1}))}\limits_{A_{12}} + \underbrace{\beta_{t} (g_i(\w_t; \xi_t^{i}) - g_i(\w_t))}\limits_{A_{13}}\\ 
    &+\underbrace{\gamma_{t}\left( \nabla g_i(\w_{t};\xi_{t}^i)-\nabla g_i(\w_{t})\right)^{\top}(\w_{t} - \w_{t-1}) }\limits_{A_{14}} \|^2+\underbrace{\gamma_{t}( \nabla g_i(\w_{t})^{\top}(\w_{t} - \w_{t-1}) - g_i(\w_t) + g_i(\w_{t-1}))}\limits_{A_{15}} \|^2]\\
    \leq &\E\left[\|A_{11}+A_{12}+A_{15}\|^2 +\|A_{13}+A_{14}\|^2\right]\\
    \leq & \E\left[(1+\beta_t)\|A_{11}\|^2 + (1+\beta_t)\|A_{12}\|^2  + (1+{2}/{\beta_t})\|A_{15}\|^2 + 2A_{11}^{\top}A_{12} + 2\|A_{13}\|^2 + 2\|A_{14}\|^2\right]\\
    \leq & \E\left[(1-\beta_{t})  \Norm{\u^i_{t-1} - g_i(\w_{t-1})}^2 + 2\gamma_{t}^2 C_g^2   \Norm{\w_{t} - \w_{t-1}}^2 + 3\gamma_t^2 {L_g^2}\Norm{\w_{t}-\w_{t-1}}^4/{(4\beta_t)}\right]\\
    & +\E\left[ 2A_{11}^{\top}A_{12}+ 2\beta_{t}^2\sigma^2+2\gamma_{t}^2{\sigma^2} \Norm{\w_{t} - \w_{t-1}}^2\right].
\end{align*}
where the last inequality is due to $\Norm{A_{15}} \leq \gamma_t L_g\Norm{\w_t-\w_{t-1}}^2/2$ and $\Norm{A_{15}} \leq 2\gamma_t C_g\Norm{\w_t-\w_{t-1}}$. By setting $\gamma_{t}^0 = \frac{m-B_1}{B_1(1-\beta_t)}$, we cancel both terms such that $\frac{B_1}{m} \E[2A_{11}^{\top}A_{12}] + (1-\frac{B_1}{m}) \E[A_{21}]=0$. Noting that $\gamma_{t+1} \leq {2m}/{B_1}$ for $\beta_{t} \leq 1/2$, we have:
\begin{align*}
    &\E \left[ \Norm{\u_{t}^i - g_i(\w_{t})}^2 \right] = \left(1-{B_1}/{m}\right) A_2 + {B_1}/{m} \cdot A_1\\
    \leq & (1-{B_1}/{m})\E \left[\Norm{ \u_{t-1}^i - g_i(\w_{t-1})}^2 +C_g^2\Norm{ \w_{t-1} - \w_{t}}^2 \right] \\
    &+\frac{B_1}{m}\E\left[(1-\beta_{t})  \Norm{\u^i_{t-1} - g_i(\w_{t-1})}^2+  2\beta_{t}^2\sigma^2+2\gamma_{t}^2(C_g^2+\sigma^2) \Norm{\w_{t} - \w_{t-1}}^2 + \frac{\gamma_{t}^2L_g^2}{\beta_t} \Norm{\w_{t} - \w_{t-1}}^4\right]\\
    \leq & (1-{B_1\beta_{t}}/{m})\E\left[ \Norm{\u^i_{t-1} - g_i(\w_{t-1})}^2\right] + {2B_1\beta_{t}^2\sigma^2}/{m}\\
    &+ {m}/{B_1} \cdot \E\left[\left(9C_g^2+8\sigma^2+\frac{4L_g^2}{\beta_t}\Norm{\w_t-\w_{t-1}}^2 \right) \Norm{\w_{t} - \w_{t-1}}^2\right]
\end{align*}
Note that $\|\u_t - \mathbf g(\w_t)\|^2=\sum_{i=1}^m\|\u^{i}_{t}  - g_{i}(\w_t)\|^2$, we finish the proof of the lemma.

%%%%%%%%%%%%%%%%%%%%%%%  Lemma 4   %%%%%%%%%%%%%%%%%%%%
\section{Proof of Lemma~4}
First, we have:
\begin{equation}
\begin{split} \label{eqn:tau1}
    \E\left[\left\| \widehat g_{i}^t - g_{i}(\w_{t})\right\|^{2}\right]   
    =& \E\left[\left\| g_{i}(\w_{t}; \xi_{t}^i) -   g_{i}(\w_{\tau}; \xi_{t}^i) + g_{i}(\w_{\tau}) - g_{i}(\w_{t})\right\|^{2}\right]  \\
    =& \E\left[\left\| g_{i}(\w_{t}; \xi_{t}^i) -   g_{i}(\w_{\tau}; \xi_{t}^i) \right\|^{2}+\left\| g_{i}(\w_{\tau}) - g_{i}(\w_{t})\right\|^{2} \right. \\
 &\left. +2 \left(g_{i}(\w_{t}; \xi_{t}^i) -   g_{i}(\w_{\tau}; \xi_{t}^i)\right)^\top  \left(g_{i}(\w_{\tau}) - g_{i}(\w_{t})\right)\right]  \\
    =& \E\left[\left\| g_{i}(\w_{t}; \xi_{t}^i) -   g_{i}(\w_{\tau}; \xi_{t}^i) \right\|^{2}-\left\| g_{i}(\w_{\tau}) - g_{i}(\w_{t})\right\|^{2} \right] \\
    \leq& C_g^2\left\|\w_{t}-\w_{\tau}\right\|^{2}.
\end{split}
\end{equation}
Since $\tau$ is the closest small index to $t$ such that $\tau$ mod $I$ = 0, we have:
\begin{equation}
\begin{split}\label{eqn:tau2}
    \sum_{t=1}^{T} \left\|\w_{t}-\w_{\tau}\right\|^{2} &\leq \sum_{t=1}^T\left\|\sum_{k=\tau+1}^t (\w_{k}-\w_{k-1})\right\|^2  \\ &\leq  \sum_{t=1}^T \sum_{k=\tau+1}^t I \left\| \w_{k}-\w_{k-1}\right\|^2
    \leq  I^2 \sum_{t=1}^T   \left\| \w_{t}-\w_{t-1}\right\|^2.
\end{split}
\end{equation}
As a result, we know that
\begin{equation}
\begin{split}
    \sum_{t=1}^{T} \E\left[\left\| \widehat g_{i}^t - g_{i}(\w_{t})\right\|^{2}\right] &
    \leq C_g^2 I^2 \sum_{t=1}^T   \left\| \w_{t}-\w_{t-1}\right\|^2.
\end{split}
\end{equation}
Then, let us focus on a fixed $i\in[m]$. Then we have
\begin{align*}
   \E\left[ \|\u^{i}_{t}  - g_{i}(\w_t)\|^2\right]
   = &\frac{B_1}{m}\underbrace{\E\left[ \|\widehat \u^{i}_{t}  - g_{i}(\w_t)\|^2\right]}\limits_{A_1} +  (1-\frac{B_1}{m})\underbrace{\E\left[ \|\u^{i}_{t-1}  - g_{i}(\w_t)\|^2\right]}\limits_{A_2}.
\end{align*}
We first decompose the second term as 
\begin{align*}
A_{2}=&\E\left[\|\u^i_{t-1} - g_i(\w_{t-1}) + g_i(\w_{t-1}) - g_i(\w_t)\|^2 \right]\\
    = &\E[\|\u^i_{t-1} - g_i(\w_{t-1})\|^2  + \|g_i(\w_{t-1}) - g_i(\w_t)\|^2 + \underbrace{2(\u^i_{t-1} - g_i(\w_{t-1}))^{\top}(g_i(\w_{t-1}) - g_i(\w_t))}\limits_{A_{21}}]\\
    \leq & \E \left[\Norm{ \u_{t-1}^i - g_i(\w_{t-1})}^2 + C_g^2 \Norm{ \w_{t-1} - \w_{t}}^2  + A_{21}\right]
\end{align*}
Next, we decompose $A_1$ as follows: 
\begin{align*}
     A_1 = &\E[\underbrace{\|(1-\beta_{t})(\u^i_{t-1} - g_i(\w_{t-1}))}\limits_{A_{11}} +  \underbrace{\gamma_{t}^0 (g_i(\w_t) - g_i(\w_{t-1}))}\limits_{A_{12}}+ \underbrace{\beta_{t} (\widehat{g}_i^t - g_i(\w_t))}\limits_{A_{13}}\\ &+\underbrace{\gamma_{t}(g_i(\w_t; \xi_t^{i}) - g_i(\w_{t-1}; \xi_t^{i}) - g_i(\w_t) + g_i(\w_{t-1}))}\limits_{A_{14}} \|^2]\\
    \leq &\E\left[\|A_{11}+A_{12}\|^2 +\|A_{13}+A_{14}\|^2\right]\\
    \leq & \E\left[\|A_{11}\|^2 + \|A_{12}\|^2  + 2A_{11}^{\top}A_{12} + 2\|A_{13}\|^2 + 2\|A_{14}\|^2\right]\\
    \leq & \E\left[(1-\beta_{t})  \Norm{\u^i_{t-1} - g_i(\w_{t-1})}^2 + 3\gamma_{t}^2 C_g^2   \Norm{\w_{t} - \w_{t-1}}^2+ 2A_{11}^{\top}A_{12}+ 2\beta_{t}^2\Norm{\widehat g_{i}^t - g_{i}(\w_{t})}^2\right].
\end{align*}
The resulting term $\E[2A_{11}^{\top}A_{12}]$ has a negative sign as $A_{21}$. Hence, by carefully choosing $\gamma_{t}^0 = \frac{m-B_1}{m(1-\beta_t)}$, we can cancel both terms such that $\frac{B_1}{m} \E[2A_{11}^{\top}A_{12}] + (1-\frac{B_1}{m}) \E[A_{21}]=0$. Noting that $\gamma_{t+1} \leq \frac{2m}{B_1}$ for $\beta_{t} \leq 1/2$, we have:
\begin{align*}
    &\E \left[ \Norm{\u_{t}^i - g_i(\w_{t})}^2 \right] = \left(1-\frac{B_1}{m}\right) A_2 + \frac{B_1}{m} A_1\\
    \leq & (1-\frac{B_1}{m})\E \left[\Norm{ \u_{t-1}^i - g_i(\w_{t-1})}^2 +C_g^2\Norm{ \w_{t-1} - \w_{t}}^2 \right] \\
    &+\frac{B_1}{m}\E\left[(1-\beta_{t})  \Norm{\u^i_{t-1} - g_i(\w_{t-1})}^2 + 2\beta_{t}^2\Norm{\widehat g_{i}^t - g_{i}(\w_{t})}^2+3\gamma_{t}^2C_g^2 \Norm{\w_{t} - \w_{t-1}}^2\right] \\
    \leq & (1-\frac{B_1\beta_{t}}{m})\E\left[ \Norm{\u^i_{t-1} - g_i(\w_{t-1})}^2\right] + \frac{2B_1\beta_{t}^2}{m}\Norm{\widehat g_{i}^t - g_{i}(\w_{t})}^2+ \frac{13 m}{B_1}C_g^2 \E\left[ \Norm{\w_{t} - \w_{t-1}}^2\right]
\end{align*}
By noting that $\|\u_t - \mathbf g(\w_t)\|^2=\sum_{i=1}^m\|\u^{i}_{t}  - g_{i}(\w_t)\|^2$, we have
\begin{align*}
    &\E \left[ \Norm{\u_{t} - g(\w_{t})}^2 \right] \\
    \leq & (1-\frac{B_1\beta_{t}}{m})\E\left[ \Norm{\u_{t-1} - g(\w_{t-1})}^2\right] + {2B_1\beta_{t}^2}\Norm{\widehat g_{i} - g_{i}(\w_{t})}^2+ \frac{13 m^2}{B_1}C_g^2 \E\left[ \Norm{\w_{t} - \w_{t-1}}^2\right].
\end{align*}
By setting $\beta_t=\beta$ and $\beta I = m/B_1$, we have
\begin{align*}
    &\frac{1}{T} \sum_{t=1}^T \E \left[ \Norm{\u_{t} - g(\w_{t})}^2 \right] \\
    \leq & \frac{m}{B_1\beta T}\E\left[ \Norm{\u_{1} - g(\w_{1})}^2\right] + 2m\beta\frac{1}{T} \sum_{t=1}^T\Norm{\widehat g_{i}^t - g_{i}(\w_{t})}^2+ \frac{13 m^3C_g^2}{B_1^2 \beta} \frac{1}{T} \sum_{t=1}^T\E\left[ \Norm{\w_{t+1} - \w_{t}}^2\right]\\
    \leq & \frac{m}{B_1\beta T}\E\left[ \Norm{\u_{1} - g(\w_{1})}^2\right] + \frac{2m\beta C_g^2 I^2}{T} \sum_{t=1}^T   \left\| \w_{t+1}-\w_{t}\right\|^2+ \frac{13 m^3C_g^2}{B_1^2 \beta} \frac{1}{T} \sum_{t=1}^T\E\left[ \Norm{\w_{t+1} - \w_{t}}^2\right]\\
    \leq & \frac{15 m^3C_g^2}{B_1^2 \beta} \frac{1}{T} \sum_{t=1}^T\E\left[ \Norm{\w_{t+1} - \w_{t}}^2\right].
\end{align*}

\section{Proof of Theorem~1}
\paragraph{Non-linear $\nabla f$}
By noting that $\eta_t = \frac{\eta}{\Norm{\z_t}}$ and $\w_{t+1} = \w_t - \eta_t \z_t$, we have:
\begin{align*}
F\left(\w_{t+1}\right)  \leq &F\left(\w_{t}\right)+\left\langle\nabla F(\w_{t}), \w_{t+1}-\w_{t}\right\rangle+\frac{L_{F}}{2}\left\|\w_{t+1}-\w_{t}\right\|^{2} \\ 
=&F(\w_{t})-\eta\left\langle\nabla F(\w_{t}), \frac{\z_{t}}{\Norm{\z_t}}\right\rangle+\frac{\eta^{2} L_{F}}{2} \\ 
=&F(\w_{t})   +  \eta\left\langle\z_t - \nabla F(\w_{t}), \frac{\z_{t}}{\Norm{\z_t}}\right\rangle - \eta \left\langle\z_t, \frac{\z_{t}}{\Norm{\z_t}}\right\rangle+\frac{\eta^{2} L_{F}}{2} \\ 
\leq & F(\w_t) + \eta \Norm{\z_t - \nabla F(\w_t)} \frac{\Norm{\z_t}}{\Norm{\z_t}} - \eta \frac{\Norm{\z_t}^2}{\Norm{\z_t}} + \frac{\eta^{2} L_{F}}{2}\\
\leq & F(\w_t) + \eta \Norm{\z_t - \nabla F(\w_t)} - \eta \Norm{\z_t} + \frac{\eta^{2} L_{F}}{2}.
\end{align*}   
This leads to the fact that
\begin{align*}
\Norm{\nabla F(\w_t)}\leq & \Norm{\z_t} + \Norm{\z_t - \nabla F(\w_t)}
    \leq \frac{F(\w_t) - F(\w_{t+1})}{\eta}  +  2\Norm{\z_t - \nabla F(\w_t)} +\frac{\eta L_F}{2} .
\end{align*}
Summing up, we know that
\begin{align*}
     \frac{1}{T}\sum_{t=1}^T \E\left[\Norm{\nabla F(\w_t)}\right] & \leq \frac{F(\w_{1})-F(\w_{T+1})}{\eta T} + \E\left[\frac{2}{T}\sum_{t=1}^T \Norm{\z_t - \nabla F(\w_t)}\right] + \frac{\eta L_F}{2}.
\end{align*}
Next, we bound the gradient estimation error as follows. By setting $\alpha_t = \alpha$, we have: 
\begin{equation*}
\begin{split}
&\E\left[\|\z_{t} - \nabla F(\w_t)\|^2 \right]\\
=&\E\left[\left\| (1-\alpha) \left(\z_{t-1} - \nabla F(\w_{t})\right)+ \alpha\left(\frac{1}{B_1}\sum_{i \in \mathcal{B}_{1}^{t}}  \nabla g_{i}(\w_t;\xi_t^{i}) \nabla f_{i}(\u_{t-1}^{i}) - \nabla F(\w_t) \right)  \right\|^2\right]\\
=&\E\left[\left\| (1-\alpha) \left(\z_{t-1} - \nabla F(\w_{t})\right)+ \frac{\alpha}{B_1}\sum_{i \in \mathcal{B}_{1}^{t}}  \left(\nabla g_{i}(\w_t;\xi_t^{i}) - \nabla g_{i}(\w_t)\right)\nabla f_{i}(\u_{t-1}^{i}) \right.\right.\\
& \left.\left. + \frac{\alpha}{B_1}\sum_{i \in \mathcal{B}_{1}^{t}}  \nabla g_{i}(\w_t)\left(\nabla f_{i}(\u_{t-1}^{i})-\nabla f_{i}(g_{i}(\w_t)) \right) \right.\right.\\
& \left.\left. + \alpha\left(\frac{1}{B_1}\sum_{i \in \mathcal{B}_{1}^{t}}  \nabla g_{i}(\w_t)\nabla f_{i}(g_{i}(\w_t)) 
- \frac{1}{m}\sum_{i=1}^m  \nabla g_{i}(\w_t) \nabla f_{i}(g_{i}(\w_t))\right)   \right\|^2\right]\\
\leq &\E\left[\left\| (1-\alpha) \left(\z_{t-1} - \nabla F(\w_{t})\right)+ \frac{\alpha}{B_1}\sum_{i \in \mathcal{B}_{1}^{t}}  \nabla g_{i}(\w_t)\left(\nabla f_{i}(\u_{t-1}^{i})-\nabla f_{i}(g_{i}(\w_t)) \right)  \right\|^2\right]\\
 &+\frac{\alpha^2 C_f^2 \sigma^2}{B_1} + \frac{\alpha^2 C_f^2 C_g^2}{B_1} \\
\leq &(1-\alpha)^2(1+\alpha)\E\left[\left\|  \z_{t-1} - \nabla F(\w_{t-1})\right\|^2\right] +\frac{\alpha^2 C_f^2 \sigma^2}{B_1} + \frac{\alpha^2 C_f^2 C_g^2}{B_1} \\
&+(1+\frac{1}{\alpha})\E\left[\left\| (1-\alpha) \left(\nabla F(\w_{t-1}) - \nabla F(\w_{t})\right)+ \frac{\alpha}{B_1}\sum_{i \in \mathcal{B}_{1}^{t}}  \nabla g_{i}(\w_t)\left(\nabla f_{i}(\u_{t-1}^{i})-\nabla f_{i}(g_{i}(\w_t)) \right)  \right\|^2\right] \\
\leq &(1-\alpha)\E\left[\left\|  \z_{t-1} - \nabla F(\w_{t-1})\right\|^2\right] +\frac{\alpha^2 C_f^2 \sigma^2}{B_1} + \frac{\alpha^2 C_f^2 C_g^2}{B_1} +\frac{4 L_F^2 \eta^2}{\alpha} \\
&+\frac{4\alpha C_g^2 L_f^2}{m}\E\left[\sum_{i=1}^m \left\| \u_{t-1}^{i}-g_{i}(\w_t)  \right\|^2\right]\\
\leq &(1-\alpha)\E\left[\left\|  \z_{t-1} - \nabla F(\w_{t-1})\right\|^2\right] +\frac{\alpha^2 C_f^2 \sigma^2}{B_1} + \frac{\alpha^2 C_f^2 C_g^2}{B_1} +\frac{4 L_F^2 \eta^2}{\alpha} \\
&+\frac{4\alpha C_g^2 L_f^2}{m}\E\left[\sum_{i=1}^m \left\| \u_{t-1}^{i}-g_{i}(\w_{t-1})\right\|^2\right] + 4\alpha C_g^4 L_f^2 \eta^2.
\end{split}
\end{equation*}
Summing up and rearranging, by setting the hyper-parameters as 
\begin{align*}
    \alpha = \sqrt{\frac{B_1}{T}}, \beta = \sqrt{\frac{m}{B_1 T}}, \eta = \frac{B_1^{1/4}}{m^{1/4}T^{3/4}},
\end{align*}
we can ensure that
\begin{equation*}
\begin{split}
&\frac{1}{T}\sum_{t=1}^T \E\left[\|\z_{t} - \nabla F(\w_t)\|^2 \right]\\
\leq &\frac{\E\left[\|\z_{1} - \nabla F(\w_1)\|^2 \right]}{\alpha T}
+ \frac{\alpha C_f^2 \sigma^2}{B_1} + \frac{\alpha C_f^2 C_g^2}{B_1} +\frac{4 L_F^2 \eta^2}{\alpha^2} \\
&+\frac{4 C_g^2 L_f^2}{m T}\sum_{t=1}^T \E\left[ \left\|\u_{t}-g(\w_t)  \right\|^2\right]+ 4 C_g^4 L_f^2 \eta^2 \\
\leq &\frac{C_g^2 C_f^2}{\alpha T} 
+ \frac{\alpha C_f^2\left( \sigma^2 + C_g^2 \right)}{B_1}  +\frac{\eta^2}{\alpha^2}\left( 4 L_F^2 + 4 C_g^4 L_f^2\right) \\
&+{4 C_g^2 L_f^2}\left(\frac{2\sigma^2}{\beta T B_1 }+\frac{10 m^2 C_g^2 \eta^2}{\beta^2 B_1^2}  + 2 \sigma^2 \beta \right)\\
\leq &  \mathcal{O}\left(\sqrt{\frac{m}{B_1 T}}\right)
\end{split}
\end{equation*}
In conclusion, we have that
\begin{align*}
     \frac{1}{T}\sum_{t=1}^T \E\left[\Norm{\nabla F(\w_t)}\right] 
     & \leq \frac{F(\w_{1})-F(\w_{T+1})}{\eta T} + \E\left[\frac{2}{T}\sum_{t=1}^T \Norm{\z_t - \nabla F(\w_t)}\right] + \frac{\eta L_F}{2} \\
     &\leq \mathcal{O}\left(\frac{m^{1/4}}{(B_1  T)^{1/4}}\right).
\end{align*}
Thus, the overall sample complexity is $\mathcal{O}\left( \frac{m}{B_1 \epsilon^4}\right)$.

\textbf{MSVR-SP estimator:}
The convergence property of MSVR-SP is almost the same as the MSVR estimator as long as ${\Norm{\w_t-\w_{t+1}}^2}/{\beta_t} = \mathcal{O}(1)$. This is valid for the above setup of hyperparameters, since $\frac{ \|\w_{t+1} - \w_{t}\|^2}{\beta_{t+1} } = \frac{\eta^2}{\beta} \leq 1$. This indicates that the result is valid for the MSVR-SP estimator.

\paragraph{Linear $\nabla f$}
When $\nabla f_i$ is a linear function such that $\nabla f_i(\x) = k \x$, we set $\beta_t =1$, and thus $\bar{\u}_{t}^i = g_i(\w_t;\xi_t^i)$. Note that we still have
\begin{align*}
     \frac{1}{T}\sum_{t=1}^T \E\left[\Norm{\nabla F(\w_t)}\right] & \leq \frac{F(\w_{1})-F(\w_{T+1})}{\eta T} + \E\left[\frac{2}{T}\sum_{t=1}^T \Norm{\z_t - \nabla F(\w_t)}\right] + \frac{\eta L_F}{2}.
\end{align*}
Next, we bound the gradient estimation error as follows. By setting $\alpha_t = \alpha$, we have: 
\begin{equation*}
\begin{split}
&\E\left[\|\z_{t} - \nabla F(\w_t)\|^2 \right]\\
=&\E\left[\left\| (1-\alpha) \left(\z_{t-1} - \nabla F(\w_{t})\right)+ \alpha\left(\frac{1}{B_1}\sum_{i \in \mathcal{B}_{1}^{t}}  \nabla g_{i}(\w_t;\xi_t^{i}) \nabla f_{i}(\u_{t-1}^{i}) - \nabla F(\w_t) \right)  \right\|^2\right]\\
=&\E\left[\left\| (1-\alpha) \left(\z_{t-1} - \nabla F(\w_{t})\right)+ \frac{\alpha}{B_1}\sum_{i \in \mathcal{B}_{1}^{t}}  \left(\nabla g_{i}(\w_t;\xi_t^{i}) - \nabla g_{i}(\w_t)\right)\nabla f_{i}(\u_{t-1}^{i})  \right.\right.\\
& \left.\left. + \frac{\alpha k}{B_1}\sum_{i \in \mathcal{B}_{1}^{t}}  \nabla g_{i}(\w_t)\left(g_{i}(\w_{t-1};\xi_{t-1}^i)-g_{i}(\w_t;\xi_{t}^i) \right) + \frac{\alpha k}{B_1}\sum_{i \in \mathcal{B}_{1}^{t}}  \nabla g_{i}(\w_t)\left(g_{i}(\w_{t};\xi_{t}^i)-g_{i}(\w_t) \right) \right.\right.\\
&\left.\left.+ \alpha\left(\frac{1}{B_1}\sum_{i \in \mathcal{B}_{1}^{t}}  \nabla g_{i}(\w_t)\nabla f_{i}(g_{i}(\w_t)) 
- \frac{1}{m}\sum_{i=1}^m  \nabla g_{i}(\w_t) \nabla f_{i}(g_{i}(\w_t))\right)   \right\|^2\right]\\
\leq &\E\left[\left\| (1-\alpha) \left(\z_{t-1} - \nabla F(\w_{t})\right) + \frac{\alpha k}{B_1}\sum_{i \in \mathcal{B}_{1}^{t}}  \nabla g_{i}(\w_t)\left(g_{i}(\w_{t-1};\xi_{t-1}^i)-g_{i}(\w_t;\xi_{t}^i) \right) \right\|^2\right] \\
&+\frac{2\alpha^2 C_f^2 \sigma^2}{B_1} +\frac{2\alpha^2 k^2 C_g^2 \sigma^2}{B_1}+ \frac{\alpha^2 C_f^2 C_g^2}{B_1} \\
\leq &(1-\alpha)^2(1+\alpha)\E\left[\left\| \z_{t-1} - \nabla F(\w_{t-1})\right\|^2\right] +\frac{2 C_f^2 \sigma^2 + 2 k^2 C_g^2 \sigma^2 + C_f^2 C_g^2}{B_1}\alpha^2  \\
&+2(1+\frac{1}{\alpha})\E\left[\left\| (1-\alpha) \left(\nabla F(\w_{t-1}) - \nabla F(\w_{t})\right) \right\|^2\right] \\
&+2(1+\frac{1}{\alpha})\E\left[\left\|  \frac{\alpha k}{B_1}\sum_{i \in \mathcal{B}_{1}^{t}}  \nabla g_{i}(\w_t)\left(g_{i}(\w_{t-1};\xi_{t-1}^i)-g_{i}(\w_t;\xi_{t}^i) \right) \right\|^2\right] \\
\leq &(1-\alpha)\E\left[\left\|  \z_{t-1} - \nabla F(\w_{t-1})\right\|^2\right] +\frac{2 C_f^2 \sigma^2 + 2 k^2 C_g^2 \sigma^2 + C_f^2 C_g^2}{B_1}\alpha^2+\frac{4\eta^2}{\alpha}\left(L_F^2 + 4k^2C_g^4 \right).
\end{split}
\end{equation*}
Summing up and rearranging, by setting $\alpha = \sqrt{{B_1}/{T}}$, $\beta = 1$ and $\eta = {B_1^{1/4}}{T^{-3/4}}$, we have:
\begin{equation*}
\begin{split}
\frac{1}{T}\sum_{t=1}^T \E\left[\|\z_{t} - \nabla F(\w_t)\|^2 \right]
\leq \frac{C_f^2 C_g^2}{\alpha T} 
+\frac{2 (C_f^2 +k^2 C_g^2)\sigma^2   + C_f^2 C_g^2}{B_1}\alpha+\frac{4\eta^2\left(L_F^2 + 4k^2C_g^4 \right)}{\alpha^2}
\leq  \mathcal{O}\left(\sqrt{\frac{1}{B_1 T}}\right).
\end{split}
\end{equation*}
In conclusion, we have that
\begin{align*}
     \frac{1}{T}\sum_{t=1}^T \E\left[\Norm{\nabla F(\w_t)}\right] \leq \frac{F(\w_{1})-F(\w_{T+1})}{\eta T} + \E\left[\frac{2}{T}\sum_{t=1}^T \Norm{\v_t - \nabla F(\w_t)}\right] + \frac{\eta L_F}{2}\leq \mathcal{O}\left(\frac{1}{(B_1 T)^{1/4} }\right).
\end{align*}
Thus, the overall sample complexity is $ \mathcal{O}\left( \frac{1}{B_1 \epsilon^4}\right)$.

\textbf{MSVR-SP estimator:}
Note that the convergence property of MSVR-SP is almost the same as the MSVR estimator as long as ${\Norm{\w_t-\w_{t+1}}^2}/{\beta_t} = \mathcal{O}(1)$. This is valid for the above setup of hyperparameters, since $\frac{ \|\w_{t+1} - \w_{t}\|^2}{\beta_{t+1} } = \frac{\eta^2}{\beta}\leq 1$.
%This indicates that the result is valid for the MSVR-SP estimator.

\section{Proof of Theorem~2}
%===================   Lemma  4   ========================
\begin{lemma}\label{lem:2}Denote $\|\u_{t}  - g(\w_t)\|^2 = \sum_{i=1}^m\|\u^{i}_{t}  - g_{i}(\w_t)\|^2$, $\|\u_{t}  - \u_{t-1}\|^2 = \sum_{i=1}^m\|\u^{i}_{t}  - \u^{i}_{t-1}\|^2$.
\begin{equation*}
\begin{split}
&\mathbb{E}\left[ \|\z_{t+1} - \nabla F(\w_{t+1})\|^2 \right] \leq (1-\alpha_{t+1}) \mathbb{E}\left[ \|\z_{t} - \nabla F(\w_{t})\|^2 \right]+\frac{3 C \eta_t^{2} \mathbb{E}\left[\left\|\mathbf{z}_{t}\right\|^{2}\right]}{\alpha_{t+1}} \\
&\quad\quad +\frac{4 L_{f}^{2} C_{g}^{2}}{m} \mathbb{E}\left[\left\|\u_{t+1}-\u_{t}\right\|^{2}\right]+\frac{2 \alpha_{t+1}^{2} C_{f}^{2}\left(\sigma^{2}+C_{g}^{2}\right)}{B_{1}}+ \frac{5\alpha_{t+1} L_{f}^{2} C_{g}^{2}}{m} \mathbb{E}\left[\left\|\u_{t} - g(\w_{t})\right\|^{2}\right] 
\end{split}
\end{equation*}\end{lemma}
\textbf{Proof}
According to Lemma 9 in \cite{dependent2022}, if $\alpha \leq {2}/{7}$, we have:
\begin{equation*}
\begin{split}
&\mathbb{E}\left[ \|\z_{t+1} - \nabla F(\w_{t+1})\|^2 \right] \leq(1-\alpha_{t+1}) \mathbb{E}\left[ \|\z_{t} - \nabla F(\w_{t})\|^2 \right]+{2 L_{F}^{2} \eta_t^{2} \mathbb{E}\left[\left\|\mathbf{z}_{t}\right\|^{2}\right]}/{\alpha_{t+1}} \\
&\quad +\frac{3 L_{f}^{2} C_{g}^{2}}{m} \mathbb{E}\left[\left\|\u_{t+1}-\u_{t}\right\|^{2}\right]+\frac{2 \alpha_{t+1}^{2} C_{f}^{2}\left(\sigma^{2}+C_{g}^{2}\right)}{B_{1}}+ \frac{5\alpha_{t+1} L_{f}^{2} C_{g}^{2}}{m} \mathbb{E}\left[\left\|\u_{t+1} - g(\w_{t+1})\right\|^{2}\right].
\end{split}
\end{equation*}
By setting $\alpha \leq {1}/{15}$, we have the above lemma.

%===================   Lemma  5   ========================
\begin{lemma}\label{lem:3} If $\beta_{t+1} \leq {1}/{2}$, we have:
\begin{equation*}
\begin{split}
  \E\left[\|\u_{t+1}-\u_{t}\|^2\right]   &\leq {2B_1\beta_{t+1}^{2}  \sigma^{2} }   + \frac{4B_1 \beta_{t+1}^{2} }{m} \E\left[ \left\|\u_{t}- g(\w_{t})\right\|^{2}\right]+\frac{9 m^2 C_g^2}{B_1}\E\left[\|\w_{t+1}-\w_{t}\|^{2}\right]
\end{split}
\end{equation*}
\end{lemma}

\textbf{Proof} 
Note that with $\beta_{t+1} \leq \frac{1}{2}$, we have $\gamma_{t+1} \leq \frac{2m}{B_1}$
\begin{equation*}
\begin{split}
&\E\left[\|\u_{t+1}-\u_t\|^2 \right] \\
=& \frac{B_1}{m} \sum_{i=1}^{m} \E\left[\left\|\beta_{t+1}\left(g_{i}(\w_{t+1};\xi_{t+1}^i)-\u_{t}^{i}\right)+\gamma_{t+1}\left(g_{i}(\w_{t+1};\xi_{t+1}^{i})-g_{i}(\w_{t};\xi_{t+1}^{i})\right)\right\|^{2}\right] \\
\leq& \frac{B_1}{m} \sum_{i=1}^{m} \E\left[2\beta_{t+1}^{2} \left\|g_{i}(\w_{t+1};\xi_{t+1}^{i})-\u_{t}^{i}\right\|^{2}+2\gamma_{t+1}^{2} \|g_{i}(\w_{t+1};\xi_{t+1}^{i})-g_{i}(\w_{t};\xi_{t+1}^{i})\|^{2}\right] \\
\leq& \E\left[\frac{2B_1\beta_{t+1}^{2}}{m} \sum_{i=1}^{m}  \left\|g_{i}(\w_{t+1};\xi_{t+1}^{i})  -\u_{t}^{i}\right\|^{2}+{2B_1\gamma_{t+1}^{2} C_g^2}\|\w_{t+1}-\w_{t}\|^{2}\right] \\
\leq& \frac{2B_1\beta_{t+1}^{2}}{m} \sum_{i=1}^{m} \left(\E\left[ \left\|g_{i}(\w_{t+1};\xi_{t+1}^{i}) - g_{i}(\w_{t+1})\right\|^{2}\right]+ \E\left[ \left\| g_{i}(\w_{t+1}) -\u_{t}^{i}\right\|^{2}\right]\right)\\
&+{2B_1\gamma_{t+1}^{2} C_g^2} \E\left[ \|\w_{t+1}-\w_{t}\|^{2}\right] \\
\leq& 2 B_1 \beta_{t+1}^{2} \sigma^{2} +\frac{2 B_1 \beta_{t+1}^{2}}{m} \sum_{i=1}^{m}\E\left[\left\|g_{i}(\w_{t+1})-\u^{i}_t\right\|^{2}\right]+\frac{8m^2 C_g^2}{B_1}\E\left[\|\w_{t+1}-\w_{t}\|^{2}\right] \\
\leq& 2 B_1 \beta_{t+1}^{2} \sigma^{2} +\frac{4 B_1 \beta_{t+1}^{2}}{m} \left\| g(\w_{t+1}) - g(\w_{t})\right\|^{2}+\frac{4 B_1 \beta_{t+1}^{2}}{m} \E\left[\left\| g(\w_{t})-\u_t\right\|^{2}\right]\\
&+\frac{8m^2 C_g^2}{B_1}\E\left[\|\w_{t+1}-\w_{t}\|^{2}\right] \\
\leq& 2 B_1 \beta_{t+1}^{2} \sigma^{2} +\frac{4 B_1 \beta_{t+1}^{2}}{m} \E\left[\left\| g(\w_{t}) - \u_t\right\|^{2}\right]+\frac{9m^2 C_g^2}{B_1}\E\left[\|\w_{t+1}-\w_{t}\|^{2}\right].
\end{split}
\end{equation*}
%The third inequality is due to $\E \left[g_{i}(\w_{t+1};\xi_{t+1}^{i}) - g_{i}(\w_{t+1})\right] = 0$. 

%===================   Lemma  3   ========================
\begin{lemma} \label{lem:1}(Lemma 2 in \cite{pmlr-v139-li21a}) Suppose function F is ${L_F}$-smooth and consider the update $\w_{t+1}:=\w_{t}-\eta_t \z_{t}$. With $\eta_t L\leq \frac{1}{2}$, we have: 
\begin{equation*}
\begin{split}
 F(\w_{t+1}) \leq F(\w_t) - \frac{\eta_t}{2} \|\nabla F(\w_t)\|^2 + \frac{\eta_t}{2} \|\z_{t} - \nabla F(\w_t)\|^2 - \frac{\eta_t}{4} \Norm{\z_t}^2
\end{split}
\end{equation*}
\end{lemma}
We denote constant $C = \max\left\{1,C_g^2,L_F^2, C_F^2,\sigma^2, L_f^2C_g^2,L_g^2C_f^2,L_f^2C_g^4, L_f^2C_g^2\sigma^2, C_f^2(\sigma^2+C_g^2)  \right\}$.
\textbf{The rest proof of Theorem~2}  Let  $\Gamma_{t}=F(\w_t)+\frac{B_1}{c_0 \eta_{t-1} m^2}\Norm{\u_{t} - g(\w_{t})}^2+  \frac{1}{c_0}\|\z_{t} - \nabla F(\w_{t})\|^2 $. By setting $\eta_t =\frac{2 \alpha_{t+1}}{c_0}$, $C_0 = 144 C$, $\eta_t \leq \frac{B_1}{4 m}$ we have:
\begin{equation*}
    \begin{split}
        &\E\left[\Gamma_{t+1}-\Gamma_{t} \right]\\
        = &\E\bigg[F(\w_{t+1}) - F(\w_t)+\frac{B_1}{c_0 \eta_{t} m^2}\Norm{\u_{t+1} - g(\w_{t+1})}^2 + \frac{1}{c_0}\|\z_{t+1} - \nabla F(\w_{t+1})\|^2 \\
        &  - \frac{B_1}{c_0 \eta_{t-1} m^2 }\Norm{\u_{t} - g(\w_{t})}^2 - \frac{1}{c_0}\|\z_{t} - \nabla F(\w_{t})\|^2 \bigg]\\
        \leq & \E\bigg[- \frac{\eta_t}{2} \|\nabla F(\w_t)\|^2 + \frac{\eta_t}{2} \|\z_{t} - \nabla F(\w_t)\|^2 - \frac{\eta_t}{4} \Norm{\z_t}^2  -\frac{\alpha_{t+1}}{c_0}  \|\z_{t} - \nabla F(\w_{t})\|^2\\
        &  +\frac{3 C \eta_t^{2} }{\alpha_{t+1} c_0} \left\|\mathbf{z}_{t}\right\|^{2}+\frac{4 L_{f}^{2} C_{g}^{2}}{m c_0} \left\|\u_{t+1}-\u_{t}\right\|^{2}+\frac{2 \alpha_{t+1}^{2} C_{f}^{2}\left(\sigma^{2}+C_{g}^{2}\right)}{B_{1} c_0} +\frac{13 C \eta_{t}}{c_0} \Norm{\z_t}^2\\
        & +  \left(\frac{5\alpha_{t+1} L_{f}^{2} C_{g}^{2}}{m c_0} +\frac{B_1}{c_0 \eta_{t} m^2} - \frac{B_1^2 \beta_{t+1}}{m^3 c_0 \eta_{t}}-\frac{B_1}{c_0 \eta_{t-1} m^2 }\right) \Norm{\u_{t} - g(\w_{t})}^2  + \frac{2 B_1^2 \beta_{t+1}^{2} \sigma^{2}}{m^2 c_0 \eta_{t}} \bigg] \\
        \leq & \E\bigg[ - \frac{\eta_t}{2} \|\nabla F(\w_t)\|^2  +\frac{4 L_{f}^{2} C_{g}^{2}}{m c_0} \left\|\u_{t+1}-\u_{t}\right\|^{2}+\frac{2 \alpha_{t+1}^{2} C_{f}^{2}\left(\sigma^{2}+C_{g}^{2}\right)}{B_{1} c_0}  - \frac{\eta_t}{8} \Norm{\z_t}^2\\
        &  \quad+  \left(\frac{5\alpha_{t+1} C}{m c_0}  + \frac{B_1}{c_0 \eta_{t} m^2} - \frac{B_1^2 \beta_{t+1}}{m^3 c_0 \eta_{t}}-\frac{B_1}{c_0 \eta_{t-1} m^2 }\right) \Norm{\u_{t} - g(\w_{t})}^2 + \frac{2 B_1^2 \beta_{t+1}^{2} \sigma^{2}}{m^2 c_0 \eta_{t}} \bigg]\\
        \leq & \E\bigg[- \frac{\eta_t}{2} \|\nabla F(\w_t)\|^2  +\frac{2 \alpha_{t+1}^{2} C}{B_{1}  c_0}  + \frac{4 B_1^2 \beta_{t+1}^{2} C}{m^2 c_0 \eta_{t}}\\
        &  \quad +  \left(\frac{5\alpha_{t+1} C}{m c_0}  + \frac{B_1}{c_0 \eta_{t} m^2} - \frac{B_1^2 \beta_{t+1}}{m^3 c_0 \eta_{t}}-\frac{B_1}{c_0 \eta_{t-1} m^2 } + \frac{16 B_1 \beta_{t+1}^2 C}{m^2 c_0}\right) \Norm{\u_{t} - g(\w_{t})}^2 \bigg]
    \end{split}
\end{equation*}
By setting $\beta_{t+1} =  \frac{256 m^2 C^2 \eta_t^2}{B_1^2}$ (and note that $c_0 = 72C$, $\alpha_{t+1}=36C \eta_t$), we have:
\begin{equation*}
    \begin{split}
        \E\left[\Gamma_{t+1}-\Gamma_{t}\right] 
        \leq & \E\left[- \frac{\eta_t}{2} \|\nabla F(\w_t)\|^2  +\frac{2 \alpha_{t+1}^{2} C}{B_{1} c_0}  + \frac{4 B_1^2 \beta_{t+1}^{2}  C}{m^2 c_0 \eta_{t}}\right] \\
        \leq & \E\left[- \frac{\eta_t}{2} \|\nabla F(\w_t)\|^2  +\frac{36 C^2 \eta_t^2}{B_{1}}  + \frac{16^4 m^2 C^4 \eta_{t}^{3}}{18 B_1^2} \right]
    \end{split}
\end{equation*}
This means that, by setting $\eta_t = \min\{\sqrt{B_1}(a+t)^{-1/2}, \left(\frac{B_1}{m}\right)^{2/3}(a+t)^{-1/3}\}$:
\begin{equation*}
\begin{split}
    &\frac{\eta_T}{2}\E\left[\sum_{t=1}^{T}\|\nabla F(\w_t) \|^2\right] \leq \E\left[\Gamma_{1}-\Gamma_{T+1}\right] +\frac{36 C^2}{B_{1}}\E\left[ \sum_{t=1}^{T} \eta_{t}^{2}\right]   + \frac{16^4 m^2 C^4}{18 B_1^2}\E\left[ \sum_{t=1}^{T} \eta_{t}^{3}\right]  \\
    \leq& \E\left[\Gamma_{1}-\Gamma_{T+1}\right] + 16^3 C^4 \E\left[ \sum_{t=1}^{T} (a+t)^{-1}\right] 
    \leq \Delta_F + \frac{2 C}{c_0 \eta_0} + 16^3 C^4 \ln{(1+T)} \\
\end{split}
\end{equation*}
Denote $M= \Delta_F + \frac{2 C}{c_0 \eta_0} + 16^3 C^4 \ln{(1+T)}$. Using Cauchy-Schwarz inequality, we have:
\begin{align*}
     &\mathbb{E}\left[\sqrt{\sum_{t=1}^{T}\left\|\nabla F\left(\boldsymbol{\w}_{t}\right)\right\|^{2}}\right]^{2} \leq \mathbb{E}\left[1 / \eta_{T}\right] \mathbb{E}\left[\eta_{T} \sum_{t=1}^{T}\left\|\nabla F\left(\boldsymbol{\w}_{t}\right)\right\|^{2}\right]  \leq \mathbb{E}\left[\frac{M}{\eta_{T}}\right]\\ 
     &\quad\quad\quad\quad \leq \mathbb{E}\left[M \max\left\{ \frac{1}{\sqrt{B_1}} \left(a+T\right)^{1 / 2},\left(\frac{m}{B_1 } \right)^{2/3}\left(a+T\right)^{1 / 3}\right\}\right],
\end{align*}
which indicates that
\begin{align*}
    \mathbb{E}\left[\sqrt{\sum_{t=1}^{T}\left\|\nabla F\left(\w_{t}\right)\right\|^{2}}\right] 
    \leq &\sqrt{M} \max\left\{B_1^{-1/4}  \left(a+T \right)^{1 / 4}, \left(\frac{m}{B_1 } \right)^{1/3} \left(a+T \right)^{1 / 6}\right\}.
\end{align*}
Using Cauchy-Schwarz, we have $
\sum_{t=1}^{T}\left\|\nabla F\left(\w_{t}\right)\right\| / T \leq \sqrt{\sum_{t=1}^{T}\left\|\nabla F\left(\w_{t}\right)\right\|^{2}} / \sqrt{T}$ so that:
\begin{align*}
    \mathbb{E}\left[\sum_{t=1}^{T} \frac{\left\|\nabla F\left(\boldsymbol{\w}_{t}\right)\right\|}{T}\right] 
    \leq&  \max\left\{\sqrt{M} \left(B_1\right)^{-1/4}  \frac{\left(a+T \right)^{1 / 4}}{\sqrt{T}}, \sqrt{M}\left(\frac{m}{B_1 } \right)^{1/3} \frac{\left(a+T \right)^{1 / 6}}{\sqrt{T}}\right\}\\
    \leq& \max\left\{\sqrt{M}\left( B_1\right)^{-1/4} \left( \frac{a^{1 / 4}}{\sqrt{T}} + \frac{1}{T^{1/4}}\right), \sqrt{ M}\left(\frac{m}{B_1} \right)^{1/3}  \left(\frac{a^{1 / 6} }{\sqrt{T}}+\frac{1}{T^{1 / 3}}\right)\right\} \\
    \leq& \O\left( \max\left\{\left(\frac{1}{B_1 T}\right)^{1/4}, \left(\frac{m}{B_1  T} \right)^{1/3}  \right\} \right).
\end{align*}
So, we achieve the complexity $\mathcal{O}\left(\max\left\{  \frac{m}{ B_1  \epsilon^3 },\frac{1}{  B_{1} \epsilon^4  } \right\}\right)$.

\textbf{MSVR-SP estimator:}
Note that the convergence property of MSVR-SP is almost the same as the MSVR estimator as long as ${\Norm{\w_t-\w_{t+1}}^2}/{\beta_t} = \mathcal{O}(1)$. This is valid for the above setup of hyperparameters and using the projection operation such that $\w_{t+1} = \w_t - \eta_t \Pi_{C_F}\left[ \z_t \right]$, since $\frac{ \|\w_{t+1} - \w_{t}\|^2}{\beta_{t+1} } = \frac{\eta_t^2 C_F^2}{\beta}\leq  1$.
This indicates that the result is valid for the MSVR-SP. 

\section{Proof of Theorem~3}
%Denote $\|\u_{t}  - g(\w_t)\|^2 = \sum_{i=1}^m\|\u^{i}_{t}  - g_{i}(\w_t)\|^2$ and $\|\u_{t}  - \u_{t-1}\|^2 = \sum_{i=1}^m\|\u^{i}_{t}  - \u^{i}_{t-1}\|^2$. 
We can first decompose the gradient estimation error as follows.
\begin{equation*}
\begin{split}
&\E\left[\|\z_{t} - \nabla F(\w_t)\|^2 \right]\\
=& \E\Bigg[2\bigg\|\z_{t} -  \frac{1}{m} \sum_{i=1}^m \nabla g_{i}(\w_t)\nabla f_i(\u_t^i)\bigg\|^2 + 2\bigg\| \frac{1}{m} \sum_{i=1}^m \nabla g_{i}(\w_t)\nabla f_i(\u_t^i) -  \frac{1}{m} \sum_{i=1}^m \nabla g_i(\w_{t})\nabla f_i(g_i(\w_{t})) \bigg\|^2\Bigg]\\
\leq&\E\left[ 2\bigg\|\z_{t} -  \frac{1}{m} \sum_{i=1}^m \nabla g_{i}(\w_t)\nabla f_i(\u_t^i)\bigg\|^2 + \frac{2}{m} \sum_{i=1}^m\bigg\|  \nabla g_{i}(\w_t)\nabla f_i(\u_t^i) -   \nabla g_i(\w_{t})\nabla f_i(g_i(\w_{t})) \bigg\|^2\right]\\
\leq& \E\left[2\bigg\|\z_{t} -  \frac{1}{m} \sum_{i=1}^m \nabla g_{i}(\w_t)\nabla f_i(\u_t^i)\bigg\|^2 + \frac{2 C_g^2L_f^2}{m} \sum_{i=1}^m\bigg\|  \u_t^i - g_i(\w_t) \bigg\|^2 \right]\\
\leq& \E\left[4\bigg\|\z_{t} -  \frac{1}{m} \sum_{i=1}^m \nabla g_{i}(\w_t)\nabla f_i(\u^i_{t-1})\bigg\|^2 + \frac{4C_g^2L_f^2}{m} \bigg\|  \u_t - \u_{t-1} \bigg\|^2 +\frac{2 C_g^2L_f^2}{m} \bigg\|  \u_t - g(\w_t) \bigg\|^2\right]
\end{split}
\end{equation*}
We can then further decompose the first term as follows.
\begin{equation*}
\begin{split}
&\E\left[\|\z_{t} - \frac{1}{m} \sum_{i=1}^m \nabla g_{i}(\w_t)\nabla f_i(\u^i_{t-1})\|^2 \right]\\
=&\E\left[\bigg\| (1-\alpha_{t}) \left(\z_{t-1} -  \frac{1}{m} \sum_{i=1}^m \nabla g_{i}(\w_{t-1})\nabla f_i(\u^i_{t-2})\right)\right. \\
&\left. + \alpha_{t}\left(\frac{1}{B_1}\sum_{i \in \mathcal{B}_{1}^{t}}  \nabla g_{i}(\w_t;\xi_t^{i}) \nabla f_{i}(\u_{t-1}^{i}) - \frac{1}{m} \sum_{i=1}^m \nabla g_{i}(\w_t)\nabla f_i(\u^i_{t-1}) \right)\right.\\
&\left. + (1-\alpha_{t})\left(\frac{1}{B_1}\sum_{i \in \mathcal{B}_{1}^{t}} \nabla g_{i}(\w_t;\xi_t^{i}) \nabla f_{i}(\u_{t-1}^{i})- \frac{1}{B_1}\sum_{i \in \mathcal{B}_{1}^{t}}\nabla g_{i}(\w_{t-1};\xi_t^{i})\nabla f_{i}(\u_{t-2}^{i})\right.\right.\\
&\left.\left. \quad\quad\quad\quad\quad  - \frac{1}{m} \sum_{i=1}^m \nabla g_{i}(\w_{t})\nabla f_i(\u^i_{t-1}) + \frac{1}{m} \sum_{i=1}^m \nabla g_{i}(\w_{t-1})\nabla f_i(\u^i_{t-2})\right)  \bigg\|^2\right]\\
& \leq \E\Bigg[(1-\alpha_{t})^2 \left\|\z_{t-1} - \frac{1}{m} \sum_{i=1}^m \nabla g_{i}(\w_{t-1})\nabla f_i(\u^i_{t-2})\right\|^2  + \frac{2C_f^2 C_g^2}{B_1}\alpha_t^2 \\
&\quad\quad + 2(1-\alpha_{t})^2\frac{1}{B_1^2}\sum_{i \in \mathcal{B}_{1}^{t}}\left\| \nabla g_{i}(\w_t;\xi_t^{i}) \nabla f_{i}(\u_{t-1}^{i})- \nabla g_{i}(\w_{t-1};\xi_t^{i})\nabla f_{i}(\u_{t-2}^{i})\right\|^2\Bigg]\\
& \leq \E\Bigg[(1-\alpha_{t}) \left\|\z_{t-1} - \frac{1}{m} \sum_{i=1}^m \nabla g_{i}(\w_{t-1})\nabla f_i(\u^i_{t-2})\right\|^2  + \frac{2C_f^2 C_g^2}{B_1}\alpha_t^2\\
&\quad\quad+4(1-\alpha_{t})^2 \frac{1}{B_1^2}\sum_{i \in \mathcal{B}_{1}^{t}}\bigg\|\nabla g_{i}(\w_t;\xi_t^{i})\left(\nabla f_{i}(\u_{t-1}^{i})-\nabla f_{i}(\u_{t-2}^{i}) \right)\bigg\|^2 \\
&\quad\quad  +4(1-\alpha_{t})^2\frac{1}{B_1^2} \sum_{i \in \mathcal{B}_{1}^{t}}\bigg\| \nabla f_{i}(\u_{t-2}^{i}) \left(\nabla g_{i}(\w_t;\xi_t^{i}) - \nabla g_{i}(\w_{t-1};\xi_t^{i})\right)\bigg\|^2\Bigg]\\
%& \leq \E\Bigg[(1-\alpha_{t}) \left\|\z_{t-1} - \frac{1}{m} \sum_{i=1}^m \nabla g_{i}(\w_{t-1})\nabla f_i(\u^i_{t-2})\right\|^2  + \frac{2\alpha_{t}^2\sigma^2}{B_1}+ 4C_g^2L_f^2 \|\u_{t-1}^{i} - \u_{t-2}^{i}\|^2 \\
%&\quad\quad +4C_f^2 L_g^2\|\w_{t} - \w_{t-1}\|^2\Bigg]\\
& \leq \E\Bigg[(1-\alpha_{t}) \left\|\z_{t-1} - \frac{1}{m} \sum_{i=1}^m \nabla g_{i}(\w_{t-1})\nabla f_i(\u^i_{t-2})\right\|^2  + \frac{2C_f^2 C_g^2}{B_1}\alpha_t^2 + \frac{4C_g^2L_f^2}{m B_1} \|\u_{t-1} - \u_{t-2}\|^2 \\
&\quad \quad + \frac{4C_f^2 L_g^2}{B_1}\|\w_{t} - \w_{t-1}\|^2 \Bigg]
\end{split}
\end{equation*}
The first inequality is due to the fact that 
\begin{align*}
    \E\left[\Norm{\frac{1}{B_1}\sum_{i \in \mathcal{B}_{1}^{t}}  \nabla g_{i}(\w_t;\xi_t^{i}) \nabla f_{i}(\u_{t-1}^{i}) - \frac{1}{m} \sum_{i=1}^m \nabla g_{i}(\w_t)\nabla f_i(\u^i_{t-1})}^2\right]\leq \frac{C_f^2 C_g^2}{B_1},
\end{align*}
as well as the expectation over the last two terms equals zero.

According to previous analyis, and supposing $\beta \leq \frac{1}{32 C}$ and $B_1 \beta_{t+1} \leq m \alpha_{t+1}$, we have:
%\begin{equation*}
%    \begin{split}
%    &\E\left[\frac{1}{m}\left\|\u_{t+1}-g\left(\w_{t+1}\right)\right\|^{2} + \left\|\z_{t+1} - \frac{1}{m} \sum_{i=1}^m \nabla g_{i}(\w_{t+1})\nabla f_i(\u^i_{t})\right\|^2\right]    \\
%     &\leq  (1-\frac{B_1\beta_{t+1}}{2m})\left[\frac{1}{m}\E\left[\left\|\u_{t}-g\left(\w_{t}\right)\right\|^{2}\right] + \left\|\z_{t} - \frac{1}{m} \sum_{i=1}^m \nabla g_{i}(\w_{t})\nabla f_i(\u^i_{t-1})\right\|^2 \right]\\ 
%    &\quad +\frac{48 m C}{B_1}\left\|\w_{t+1}-\w_{t}\right\|^{2}+ \frac{10 B_1\beta_{t+1}^{2} \sigma^{2}}{ m} +\frac{2\alpha_{t+1}^2\sigma^2}{B_1}
%    \end{split}
%\end{equation*}
%\end{lemma}
%\begin{proof}
\begin{equation*}
    \begin{split}
    &\E\left[\frac{1}{m}\left\|\u_{t+1}-g\left(\w_{t+1}\right)\right\|^{2} + \left\|\z_{t+1} - \frac{1}{m} \sum_{i=1}^m \nabla g_{i}(\w_{t+1})\nabla f_i(\u^i_{t})\right\|^2\right]    \\
    \leq&  (1-\frac{B_1\beta_{t+1}}{m})\frac{1}{m}\E\left[\left\|\u_{t}-g\left(\w_{t}\right)\right\|^{2}\right]+\frac{13 m C_g^2}{B_1}\left\|\w_{t+1}-\w_{t}\right\|^{2}+ \frac{2B_1\beta_{t+1}^{2} \sigma^{2}}{m}\\ 
    &\quad + (1-\alpha_{t+1}) \E\left[\left\|\z_{t} - \frac{1}{m} \sum_{i=1}^m \nabla g_{i}(\w_{t})\nabla f_i(\u^i_{t-1})\right\|^2\right] +\frac{2C_f^2C_g^2\alpha_{t+1}^2}{B_1}  \\
    &\quad + \frac{4C_g^2L_f^2}{m B_1} \E\left[\|\u_{t} - \u_{t-1}\|^2 \right] + \frac{4C_f^2 L_g^2}{B_1}\E\left[\|\w_{t+1} - \w_{t}\|^2 \right] \\
    \leq&  (1-\frac{B_1\beta_{t+1}}{m})\E\left[\frac{1}{m}\left\|\u_{t}-g\left(\w_{t}\right)\right\|^{2} + \left\|\z_{t} - \frac{1}{m} \sum_{i=1}^m \nabla g_{i}(\w_{t})\nabla f_i(\u^i_{t-1})\right\|^2 \right]\\
    &\quad +\frac{17 m C}{B_1}\left\|\w_{t+1}-\w_{t}\right\|^{2}+ \frac{2B_1\beta_{t+1}^{2} \sigma^{2}}{ m} +\frac{2C_f^2C_g^2\alpha_{t+1}^2}{B_1}  + \frac{4C_g^2L_f^2}{m B_1} \E\left[\|\u_{t} - \u_{t-1}\|^2 \right] \\
     \leq & (1-\frac{B_1\beta_{t+1}}{2m})\E\left[\frac{1}{m}\left\|\u_{t}-g\left(\w_{t}\right)\right\|^{2} + \left\|\z_{t} - \frac{1}{m} \sum_{i=1}^m \nabla g_{i}(\w_{t})\nabla f_i(\u^i_{t-1})\right\|^2 \right]\\ 
    &\quad +\frac{53 m C}{B_1}\left\|\w_{t+1}-\w_{t}\right\|^{2}+ \frac{10 B_1\beta_{t+1}^{2} C}{m} +\frac{2\alpha_{t+1}^2 C}{B_1}
    \end{split}
\end{equation*}

\textbf{The rest proof of Theorem~3}  Set $\eta_t \leq \frac{B_1}{m c_0}$. Denote that $\Gamma_{t}=F(\w_t) +\frac{B_1}{c_0 \eta_{t-1} m} \Delta_t$, where $\Delta_t = \frac{1}{m}\Norm{\u_{t} - g(\w_{t})}^2+ \left\|\z_{t} - \frac{1}{m} \sum_{i=1}^m \nabla g_{i}(\w_t)\nabla f_i(\u^i_{t-1})\right\|^2$. We have:
\begin{equation*}
    \begin{split}
        &\E\left[\Gamma_{t+1}-\Gamma_{t}\right] \\
        = & \E\left[F(\w_{t+1}) - F(\w_t)+\frac{B_1}{c_0 \eta_{t} m}\Delta_{t+1}   - \frac{B_1}{c_0 \eta_{t-1} m }\Delta_{t}\right] \\
        \leq & \E\bigg[ - \frac{\eta_t}{2} \|\nabla F(\w_t)\|^2 + \frac{\eta_t}{2} \|\z_{t} - \nabla F(\w_t)\|^2 - \frac{\eta_t}{4} \Norm{\z_t}^2  \\
        & +  \left(\frac{B_1}{c_0 \eta_{t} m} - \frac{B_1^2 \beta_{t+1}}{2 m^2 c_0 \eta_{t}}-\frac{B_1}{c_0 \eta_{t-1} m }\right)\Delta_{t}+\frac{53 C}{c_0 \eta_{t}}\left\|\w_{t+1}-\w_{t}\right\|^{2}+ \frac{10 B_1^2 \beta_{t+1}^{2} C}{m^2 c_0 \eta_{t}} +\frac{2\alpha_{t+1}^2 C}{m c_0 \eta_{t}} \bigg]\\
        \leq & \E\bigg[- \frac{\eta_t}{2} \|\nabla F(\w_t)\|^2  - \frac{\eta_t}{4} \Norm{\z_t}^2 +\frac{71 C}{c_0 \eta_{t}}\left\|\w_{t+1}-\w_{t}\right\|^{2}+ \frac{14 B_1^2 \beta_{t+1}^{2} C}{ m^2 c_0 \eta_{t}} +\frac{2\alpha_{t+1}^2 C}{m c_0 \eta_{t}}\\
        &  +  \left(2C\eta_t + \frac{B_1}{c_0 \eta_{t} m} - \frac{B_1^2 \beta_{t+1}}{2 m^2 c_0 \eta_{t}}-\frac{B_1}{c_0 \eta_{t-1} m }\right)\Delta_{t}\bigg]
    \end{split}
\end{equation*}
By setting $284 C =  c_0$, $\eta_{t}^2 = \frac{32 B_1^2 \beta_{t+1}}{m^2 c_0^2 }$, $\alpha_{t+1} = \frac{B_1 \beta_{t+1}}{m}$, we have:
\begin{equation*}
    \begin{split}
        \E\left[\Gamma_{t+1}-\Gamma_{t} \right]
        \leq & \E\left[- \frac{\eta_t}{2} \|\nabla F(\w_t)\|^2  + \frac{14 B_1^2\beta_{t+1}^{2} C}{m^2 c_0 \eta_{t}} +\frac{2\alpha_{t+1}^2 C}{m c_0 \eta_{t}} \right]
        \leq \E\left[ - \frac{\eta_t}{2} \|\nabla F(\w_t)\|^2  + \frac{m^2 \eta_{t}^{3} c_0^4}{512 B_1^2 } \right]
    \end{split}
\end{equation*}
This means that, by setting $\eta_t = (\frac{B_1 }{m})^{\frac{2}{3}}(a+t)^{-\frac{1}{3}}$
\begin{equation*}
\begin{split}
    \frac{\eta_T}{2}\E\left[\sum_{t=1}^{T}\|\nabla F(\w_t) \|^2\right] \leq& \E\left[\Gamma_{1}-\Gamma_{T+1} + \frac{m^2 c_0^4}{512 B_1^2 } \sum_{t=1}^{T} \eta_{t}^{3}\right] \\
    \leq& \E\left[\Gamma_{1} + \frac{c_0^4}{512}  \sum_{t=1}^{T} (a+t)^{-1}\right] 
    \leq \Delta_F + \frac{1}{8 \eta_0} + \frac{c_0^4}{512 } \ln{(1+T)}.
\end{split}
\end{equation*}
Denote $M=\Delta_F + \frac{1}{8 \eta_0} + \frac{c_0^4}{512 } \ln{(1+T)}$. Using Cauchy-Schwarz inequality, we have:
\begin{align*}
     \mathbb{E}\left[\sqrt{\sum_{t=1}^{T}\left\|\nabla F\left(\boldsymbol{\w}_{t}\right)\right\|^{2}}\right]^{2} &\leq \mathbb{E}\left[ \sum_{t=1}^{T}\left\|\nabla F\left(\boldsymbol{\w}_{t}\right)\right\|^{2}\right]  \leq \mathbb{E}\left[\frac{M}{\eta_{T}}\right] \leq \mathbb{E}\left[M \left(\frac{m}{B_1 } \right)^{2/3}\left(a+T\right)^{1 / 3}\right],
\end{align*}
which indicates that
\begin{align*}
    \mathbb{E}\left[\sqrt{\sum_{t=1}^{T}\left\|\nabla F\left(\w_{t}\right)\right\|^{2}}\right] \leq \sqrt{M} \left(\frac{m}{B_1} \right)^{1/3} \left(a+T \right)^{1 / 6}.
\end{align*}
Using Cauchy-Schwarz, we have $
\sum_{t=1}^{T}\left\|\nabla F\left(\w_{t}\right)\right\| / T \leq \sqrt{\sum_{t=1}^{T}\left\|\nabla F\left(\w_{t}\right)\right\|^{2}} / \sqrt{T}$ so that:
\begin{align*}
    \mathbb{E}\left[\sum_{t=1}^{T} \frac{\left\|\nabla F\left(\boldsymbol{\w}_{t}\right)\right\|}{T}\right] &\leq \frac{\sqrt{M} \left(a+T\right)^{1 / 6}m^{1/3}}{\sqrt{T}B_1^{1/3}} 
    \leq \mathcal{O} \left(\frac{a^{1 / 6} \sqrt{ M}}{\sqrt{T}}+\frac{m^{1/3}}{B_1^{1/3}  T^{1/3}} \right) = \mathcal{O}\left(\frac{m^{1/3}}{T^{1/3} B_1^{1/3} } \right),
\end{align*}
where the last inequality is due to $(a+b)^{1 / 3} \leq a^{1 / 3}+b^{1 / 3}$. So, we can achieve the stationary point with $T=\mathcal{O}\left(m /B_1  \epsilon^{3}\right)$.

\textbf{MSVR-SP estimator:}
Note that the convergence property of MSVR-SP is almost the same as the MSVR estimator as long as ${\Norm{\w_t-\w_{t+1}}^2}/{\beta_t} = \mathcal{O}(1)$. This is valid for the above setup of hyperparameters and using the projection operation such that $\w_{t+1} = \w_t - \eta_t \Pi_{C_F}\left[ \z_t \right]$, since $ \frac{ \|\w_{t+1} - \w_{t}\|^2}{\beta_{t+1} } = \frac{\eta_t^2 C_F^2}{\beta}\leq 1$.
This indicates the result is valid for the MSVR-SP estimator. 

\section{MSVRM with Adam-style Learning Rates} \label{adaptive}
We now demonstrate the proposed MSVRM method can be extended to Adam-style learning rates and retains the same sample complexity. To use Adam-style learning rates, we can revise the weight update step $\w_{t+1} = \w_t - \eta_t \z_t$ in the original MSVRM method as follows:
\begin{equation}\label{rule1}
    \begin{split}
        \w_{t+1} = \w_t - \frac{\eta_t}{\sqrt{\h_{t}}+\delta} \Pi_{L_f}[\z_t], \quad
        \mathbf{h}_{t}^{\prime}=\left(1-\beta_{t}^{\prime}\right) \mathbf{h}_{t-1}^{\prime}+\beta_{t}^{\prime} \z_{t}^{2},
    \end{split}
\end{equation}
where $\delta > 0$ is a parameter to avoid dividing zero, $\Pi_{L_f}$ denotes the projection onto the ball with radius $L_f$ and $\mathbf{h}_{t} = \mathbf{h}_{t}^{\prime}$ (Adam-style) or $\mathbf{h}_{t} =\max \left(\mathbf{h}_{t-1}, \mathbf{h}_{t}^{\prime}\right)$ (AMSGrad-style).  Inspired by the recent study of Adam-style methods~\cite{guo2022stochastic}, we can give the sample complexity of the Adam-style MSVR using a similar analysis. We show the proof of Adam-style MSVR-v2, for example:
\begin{thm}
    If we choose that $\alpha_{t+1} = \O(\frac{m \eta_{t}^2}{B_1})$, $\beta_{t+1} = \O\left(\frac{m^2 \eta_t^2}{B_1^2}\right)$, $a=O(\frac{m }{B_1}$) and $\eta_t = \O\left((\frac{B_1 }{m})^{2/3}(a+t)^{-1/3} \right)$, Adam-style MSVRM-v2 with learning rate defined in (\ref{rule1}), can obtain a stationary point in $ \mathcal{O}\left(\frac{m \epsilon^{-3} }{ B_1  } \right)$ iterations.
\end{thm}
\textbf{Remark:} The sample complexity is still at the order of $\mathcal{O}\left(\epsilon^{-3}\right)$. For MSVR-v1 and MSVR-v3, or under the convexity or PL condition, the Adam-style method can still get the same complexity as the original rate using a very similar analysis.

\noindent \textbf{Proof}
Note that since the norm of estimated gradient $\left\| \z_{t} \right\|$ is bounded, the value of the learning rate scaling factor $\mathbf{c}=1 /\left(\sqrt{\mathbf{h}_{t}}+\delta\right)$ is also upper bounded and lower bounded, which can be presented as $c_{l} \leq\left\|\mathbf{c}\right\|_{\infty} \leq c_{u}$. (Note that projection onto a ball of radius $C_F$ does not change the analysis, since $\nabla F$ is also in this ball.) With this property, we have: 
\begin{lemma}\label{lem:starter_} (Lemma 3 in \cite{guo2022stochastic})
    For $\mathbf{w}_{t+1}=\mathbf{w}_{t}-\tilde{\eta}_{t} \mathbf{z}_{t}$, with $\eta_t c_{l} \leq \tilde{\eta}_{t} \leq \eta_t c_{u} $ and $ \eta_t L_F\leq {c_l}/{2 c_u^2 }$, we have following guarantee:
    \begin{align*}
        F(\w_{t+1}) \leq F(\w_t) + \frac{\eta_t c_{u}}{2}\Norm{\nabla F(\w_t) - \z_t}^2 - \frac{\eta_t c_{l}}{2}\Norm{\nabla F(\w_t)}^2 - \frac{\eta_t c_{l}}{4}\Norm{\z_t}^2.
    \end{align*}
\end{lemma} 
Then very similar to the proof to Theorem~3.
Denote  $\Gamma_{t}=F(\w_t) + \frac{B_1}{c_0 \eta_{t-1} m}\Delta_t$, where $\Delta_t = +\frac{1}{m}\Norm{\u_{t} - g(\w_{t})}^2+ \left\|\z_{t} - \frac{1}{m} \sum_{i=1}^m \nabla g_{i}(\w_t)\nabla f_i(\u^i_{t-1})\right\|^2$. We have:
\begin{equation*}
    \begin{split}
        &\Gamma_{t+1}-\Gamma_{t} \\
        = &F(\w_{t+1}) - F(\w_t)+\frac{B_1}{c_0 \eta_{t} m}\Delta_{t+1}  - \frac{B_1}{c_0 \eta_{t-1} m }\Delta_t \\
        \leq & - \frac{\eta_t c_l}{2} \|\nabla F(\w_t)\|^2 + \frac{\eta_t c_u}{2} \|\z_{t} - \nabla F(\w_t)\|^2 - \frac{\eta_t c_l}{4} \Norm{\z_t}^2 \\
        &  +  \left(\frac{B_1}{c_0 \eta_{t} m} - \frac{B_1^2 \beta_{t+1}}{2 m^2 c_0 \eta_{t}}-\frac{B_1}{c_0 \eta_{t-1} m }\right)\Delta_t  +\frac{48 C}{c_0 \eta_{t}}\left\|\w_{t+1}-\w_{t}\right\|^{2}+ \frac{10 B_1^2 \beta_{t+1}^{2} C}{ m^2 c_0 \eta_{t}} +\frac{2\alpha_{t+1}^2 C}{m c_0 \eta_{t}}\\
        \leq & - \frac{\eta_t c_l}{2} \|\nabla F(\w_t)\|^2  - \frac{\eta_t c_l}{4} \Norm{\z_t}^2 +\frac{64 C c_u}{c_0 \eta_{t}}\left\|\w_{t+1}-\w_{t}\right\|^{2}+ \frac{14 B_1^2 \beta_{t+1}^{2} C c_u}{m^2 c_0 \eta_{t}} +\frac{2\alpha_{t+1}^2 C c_u}{m c_0 \eta_{t}}\\
        &  +  \left(2C c_u \eta_t + \frac{B_1}{c_0 \eta_{t} m} - \frac{B_1^2 \beta_{t+1}}{2 m^2 c_0 \eta_{t}}-\frac{B_1}{c_0 \eta_{t-1} m }\right)\Delta_t.
    \end{split}
\end{equation*}
By setting $256 C c_u / c_l =  c_0$, $\eta_{t}^2 = \frac{32 B_1^2 \beta_{t+1}}{m^2 c_0^2 c_l}$, $\alpha_{t+1} = \frac{B_1 \beta_{t+1}}{m}$, we have:
\begin{equation*}
    \begin{split}
        \Gamma_{t+1}-\Gamma_{t} 
        &\leq - \frac{\eta_t c_l}{2} \|\nabla F(\w_t)\|^2  + \frac{14 B_1^2\beta_{t+1}^{2} C c_u}{m^2 c_0 \eta_{t}} +\frac{2\alpha_{t+1}^2 C c_u}{m c_0 \eta_{t}}  \leq  - \frac{\eta_t c_l}{2} \|\nabla F(\w_t)\|^2  + \frac{m^2 \eta_{t}^{3} c_0^4 c_l^3}{512 B_1^2}. 
    \end{split}
\end{equation*}
This means that, by setting $\eta_t = (\frac{B_1 }{m})^{\frac{2}{3}}(a+t)^{-\frac{1}{3}}$
\begin{equation*}
\begin{split}
    \frac{\eta_T}{2}\E\left[\sum_{t=1}^{T}\|\nabla F(\w_t) \|^2\right] \leq& \frac{\Gamma_{1}-\Gamma_{T+1}}{c_l} + \frac{m^2  c_0^4 c_l^2}{512  B_1^2 } \E\left[ \sum_{t=1}^{T} \eta_{t}^{3}\right] \\
    \leq& \frac{\Gamma_{1}}{c_l} + \frac{c_0^4 c_l^2}{16^5} \E\left[ \sum_{t=1}^{T} (a+t)^{-1}\right]
    \leq \frac{\Delta_F}{c_l} + \frac{1}{8 \eta_0 c_l} + \frac{c_0^4 c_l^2}{512}\ln{(1+T)}.
\end{split}
\end{equation*}
Denote $M=\frac{\Delta_F}{c_l} + \frac{1}{8 \eta_0 c_l} + \frac{c_0^4 c_l^2}{16^5}\ln{(1+T)}$. Using Cauchy-Schwarz inequality, we have:
\begin{align*}
     \mathbb{E}\left[\sqrt{\sum_{t=1}^{T}\left\|\nabla F\left(\boldsymbol{\w}_{t}\right)\right\|^{2}}\right]^{2} &\leq \mathbb{E}\left[1 / \eta_{T}\right] \mathbb{E}\left[\eta_{T} \sum_{t=1}^{T}\left\|\nabla F\left(\boldsymbol{\w}_{t}\right)\right\|^{2}\right] \leq \mathbb{E}\left[M \left(\frac{m}{B_1 } \right)^{2/3}\left(a+T\right)^{1 / 3}\right].
\end{align*}
Then, following the same analysis, we will finally have :
\begin{align*}
    \mathbb{E}\left[\sum_{t=1}^{T} \frac{\left\|\nabla F\left(\boldsymbol{\w}_{t}\right)\right\|}{T}\right] &\leq \frac{\sqrt{M} \left(a+T\right)^{1 / 6}}{\sqrt{T}} \left(\frac{m}{B_1 } \right)^{1/3}
    \leq \mathcal{O} \left(\frac{a^{1 / 6} \sqrt{ M}}{\sqrt{T}}+\left(\frac{m}{B_1  T} \right)^{1/3}\right) \\ &= \mathcal{O}\left(\left(\frac{m}{T B_1 } \right)^{1 / 3}\right),
\end{align*}
where the last inequality is due to $(a+b)^{1 / 3} \leq a^{1 / 3}+b^{1 / 3}$. So, we can achieve the stationary point with $T=\mathcal{O}\left(m /B_1 \epsilon^{3}\right)$.

\section{Proof of Theorem~4}
\paragraph{Non-linear $\nabla f$}
Note that we already know that
\begin{align*}
     \frac{1}{T}\sum_{t=1}^T \E\left[\Norm{\nabla F(\w_t)}\right] & \leq \frac{F(\w_{1})-F(\w_{T+1})}{\eta T} + \E\left[\frac{2}{T}\sum_{t=1}^T \Norm{\v_t - \nabla F(\w_t)}\right] + \frac{\eta L_F}{2}.
\end{align*}
Next, we bound the gradient estimation error as follows.
Since we already know that
\begin{equation*}
\begin{split}
      \E\left[\left\|\z_{t} - \frac{1}{m} \sum_{i=1}^m \nabla g_{i}(\w_t)\nabla f_i(\u^i_{t-1})\right\|^2 \right]
    \leq   \E\Bigg[(1-\alpha) \left\|\z_{t-1} - \frac{1}{m} \sum_{i=1}^m \nabla g_{i}(\w_{t-1})\nabla f_i(\u^i_{t-2})\right\|^2 \\ 
  +\frac{2C_f^2 C_g^2}{B_1}\alpha^2 + \frac{4C_g^2L_f^2}{mB_1} \|\u_{t-1} - \u_{t-2}\|^2 + \frac{4C_f^2 L_g^2}{B_1}\|\w_{t} - \w_{t-1}\|^2 \Bigg]
\end{split}
\end{equation*}
Summing up and rearranging, we have
\begin{equation*}
\begin{split}
      &\frac{1}{T}\sum_{t=1}^T\E\left[\left\|\z_{t} - \frac{1}{m} \sum_{i=1}^m \nabla g_{i}(\w_t)\nabla f_i(\u^i_{t-1})\right\|^2 \right]\\
    \leq &  \frac{C_g^2 C_f^2}{B_0 \alpha T}
  +\frac{2C_f^2 C_g^2}{B_1}\alpha + \frac{4C_g^2L_f^2}{mB_1 \alpha T} \sum_{t=1}^T \|\u_{t-1} - \u_{t}\|^2 + \frac{4C_f^2 L_g^2}{B_1 \alpha}\eta^2 
\end{split}
\end{equation*}
Note that we have also proved that 
\begin{equation*}
\begin{split}
  \E\left[\|\u_{t+1}-\u_{t}\|^2\right]   &\leq {2B_1\beta_{t+1}^{2}  \sigma^{2} }   + \frac{4B_1 \beta_{t+1}^{2} }{m} \E\left[ \left\|\u_{t}- g(\w_{t})\right\|^{2}\right]+\frac{9 m^2 C_g^2 \eta^2}{B_1}
\end{split}
\end{equation*}
Also, by setting $B_1 \beta^2 \leq m \alpha $, we know that 
\begin{equation*}
\begin{split}
&\frac{1}{T}\sum_{t=1}^T\E\left[\|\z_{t} - \nabla F(\w_t)\|^2\right] \\
\leq& \frac{4}{T}\sum_{t=1}^T\E\left[\bigg\|\z_{t} -  \frac{1}{m} \sum_{i=1}^m \nabla g_{i}(\w_t)\nabla f_i(\u^i_{t-1})\bigg\|^2\right]  + \frac{4C_g^2L_f^2}{m} \frac{1}{T}\sum_{t=1}^T\E\left[\|  \u_t - \u_{t-1} \|^2\right]\\
&+\frac{2 C_g^2L_f^2}{m} \frac{1}{T}\sum_{t=1}^T\E\left[\|  \u_t - g(\w_t) \|^2 \right] \\
\leq &\frac{4C_g^2 C_f^2}{B_0 \alpha T}
  +\frac{8C_f^2 C_g^2}{B_1}\alpha + \left(\frac{16C_g^2L_f^2}{mB_1 \alpha T} +\frac{4C_g^2L_f^2}{m T}\right)\sum_{t=1}^T \|\u_{t-1} - \u_{t}\|^2 + \frac{16C_f^2 L_g^2}{B_1 \alpha}\eta^2\\
  &+\frac{2 C_g^2L_f^2}{m} \frac{1}{T}\sum_{t=1}^T\E\left[\|  \u_t - g(\w_t) \|^2 \right] \\
  \leq &\frac{4C_g^2 C_f^2}{B_0 \alpha T}
  +\frac{8C_f^2 C_g^2}{B_1}\alpha + \frac{20C_g^2L_f^2}{m \alpha }\left(2B_1\beta^{2}  \sigma^{2}    +\frac{9 m^2 C_g^2 \eta^2}{B_1}\right ) + \frac{16C_f^2 L_g^2}{B_1 \alpha}\eta^2 \\
  &+\frac{82 C_g^2L_f^2}{m} \frac{1}{T}\sum_{t=1}^T\E\left[\|  \u_t - g(\w_t) \|^2 \right]
\end{split}
\end{equation*}
We set that $B_0=\frac{T^{1/3}}{m}$, $\alpha = \frac{m^{2/3}B_1^{1/3}}{T^{2/3}}$, $\beta = \frac{m^{2/3}}{B_1^{2/3}T^{2/3}}$, $\eta = \frac{B_1^{1/3}}{m^{1/3}T^{2/3}}$, and we can obtain
\begin{align}
   \frac{1}{T}\sum_{t=1}^T \E\left[\Norm{\nabla F(\w_t)}\right] \leq  \left(\frac{m}{B_1T}\right)^{1/3}
\end{align}
\textbf{MSVR-SP estimator:}
The convergence of MSVR-SP is similar when ${\Norm{\w_t-\w_{t+1}}^2}/{\beta_t} = \mathcal{O}(1)$. This is valid for the setup of hyperparameters, since $
    \frac{ \|\w_{t+1} - \w_{t}\|^2}{\beta_{t+1} } = \frac{\eta^2}{\beta}\leq \frac{B_1^{4/3}}{m^{4/3}T^{2/3}} \leq 1$.
%This indicates that the result is valid for the MSVR-SP estimator.

\paragraph{Linear $\nabla f$} 
%Suppose that $\nabla f$ is linear such that $\nabla f(\x) = k\x$. 
We first decompose the gradient estimation error as follows:
\begin{equation*}
\begin{split}
&\E\left[\|\z_{t} - \nabla F(\w_t)\|^2 \right]\\
\leq & \E\Bigg[2\bigg\|\z_{t} -  \frac{1}{m} \sum_{i=1}^m \nabla g_{i}(\w_t)\nabla f_i(g_i(\w_{t-1}))\bigg\|^2 \\
&+ 2\bigg\| \frac{1}{m} \sum_{i=1}^m \nabla g_{i}(\w_t)\nabla f_i(g_i(\w_{t-1})) -  \frac{1}{m} \sum_{i=1}^m \nabla g_i(\w_{t})\nabla f_i(g_i(\w_{t})) \bigg\|^2\Bigg]\\
\leq & \E\Bigg[2\bigg\|\z_{t} -  \frac{1}{m} \sum_{i=1}^m \nabla g_{i}(\w_t)\nabla f_i(g_i(\w_{t-1}))\bigg\|^2 + \frac{2C_g^2k^2}{m} \sum_{i=1}^m\bigg\|  g_i(\w_{t-1}) -  g_i(\w_{t}) \bigg\|^2\Bigg]\\
\leq & \E\Bigg[2\bigg\|\z_{t} -  \frac{1}{m} \sum_{i=1}^m \nabla g_{i}(\w_t)\nabla f_i(g_i(\w_{t-1}))\bigg\|^2 + 2C_g^4k^2 \bigg\|  \w_{t-1} -  \w_{t} \bigg\|^2\Bigg]\\
\leq & 2\E\left[\bigg\|\z_{t} -  \frac{1}{m} \sum_{i=1}^m \nabla g_{i}(\w_t)\nabla f_i(g_i(\w_{t-1}))\bigg\|^2\right] + 2C_g^4k^2 \eta^2
\end{split}
\end{equation*}
Next, setting $\gamma_t= 0$ and $\beta_t=1$, we know $\bar{\u}_i^t = g_i(\w_t;\xi_t^i)$ and we have that
\begin{equation*}
\begin{split}
&\E\left[\|\z_{t} - \frac{1}{m} \sum_{i=1}^m \nabla g_{i}(\w_t)\nabla f_i(g_{i}(\w_{t-1}))\|^2 \right]\\
=&\E\left[\bigg\| (1-\alpha) \left(\z_{t-1} -  \frac{1}{m} \sum_{i=1}^m \nabla g_{i}(\w_{t-1})\nabla f_i(g_{i}(\w_{t-2}))\right)\right. \\
&\left. + \alpha\left(\frac{1}{B_1}\sum_{i \in \mathcal{B}_{1}^{t}}  \nabla g_{i}(\w_t;\xi_t^{i}) \nabla f_{i}(g_{i}(\w_{t-1};\xi_{t-1}^{i})) - \frac{1}{m} \sum_{i=1}^m \nabla g_{i}(\w_t)\nabla f_i(g_{i}(\w_{t-1})) \right)\right.\\
&\left. + (1-\alpha)\left(\frac{1}{B_1}\sum_{i \in \mathcal{B}_{1}^{t}} \nabla g_{i}(\w_t;\xi_t^{i}) \nabla f_{i}(g_{i}(\w_{t-1};\xi_{t-1}^{i}))- \frac{1}{B_1}\sum_{i \in \mathcal{B}_{1}^{t}}\nabla g_{i}(\w_{t-1};\xi_t^{i})\nabla f_{i}(g_{i}(\w_{t-2};\xi_{t-2}^{i}))\right.\right.\\
&\left.\left. \quad\quad\quad\quad\quad  - \frac{1}{m} \sum_{i=1}^m \nabla g_{i}(\w_{t})\nabla f_i(g_{i}(\w_{t-1})) + \frac{1}{m} \sum_{i=1}^m \nabla g_{i}(\w_{t-1})\nabla f_i(g_{i}(\w_{t-2}))\right)  \bigg\|^2\right]\\
\leq &(1-\alpha)\E\left[\left\|  \z_{t-1} -  \frac{1}{m} \sum_{i=1}^m \nabla g_{i}(\w_{t-1})\nabla f_i(g_{i}(\w_{t-2}))\right\|^2\right] \\
& +2\alpha^2 k^2\E\left[\left\| \frac{1}{B_1}\sum_{i \in \mathcal{B}_{1}^{t}}  \nabla g_{i}(\w_t;\xi_t^{i}) g_{i}(\w_{t-1};\xi_{t-1}^{i}) - \frac{1}{m} \sum_{i=1}^m \nabla g_{i}(\w_t)g_{i}(\w_{t-1}) \right\|^2\right]\\
&+ \frac{2(1-\alpha)^2k^2}{B_1^2}\sum_{i \in \mathcal{B}_{1}^{t}}\E\left[\left\|   \nabla g_{i}(\w_t;\xi_t^{i}) g_{i}(\w_{t-1};\xi_{t-1}^{i})- \nabla g_{i}(\w_{t-1};\xi_t^{i})g_{i}(\w_{t-2};\xi_{t-2}^{i}) \right\|^2\right]\\
\leq &(1-\alpha)\E\left[\left\|  \z_{t-1} -  \frac{1}{m} \sum_{i=1}^m \nabla g_{i}(\w_{t-1})\nabla f_i(g_{i}(\w_{t-2}))\right\|^2\right]  +\frac{2\alpha^2  C_g^2C_f^2}{B_1}+ \frac{2\left(2C_g^2L_f^2+2C_f^2L_g^2\right)}{B_1}\eta^2\\
\leq &(1-\alpha)\E\left[\left\|  \z_{t-1} -  \frac{1}{m} \sum_{i=1}^m \nabla g_{i}(\w_{t-1})\nabla f_i(g_{i}(\w_{t-2}))\right\|^2\right]  +\frac{2\alpha^2  C_g^2C_f^2}{B_1}+ \frac{4\left(C_g^2L_f^2+C_f^2L_g^2\right)}{B_1}\eta^2
\end{split}
\end{equation*}
As a result, we have that
\begin{equation*}
\begin{split}
&\frac{1}{T}\sum_{t=1}^T \E\left[\|\z_{t} - \frac{1}{m} \sum_{i=1}^m \nabla g_{i}(\w_t)\nabla f_i(g_{i}(\w_{t-1}))\|^2 \right]
\leq \frac{C_g^2 C_f^2}{B_0 \alpha T}  +\frac{2\alpha  C_g^2C_f^2}{B_1}+ \frac{4\left(C_g^2L_f^2+C_f^2L_g^2\right)}{B_1}\frac{\eta^2}{\alpha}
\end{split}
\end{equation*}
as well as
\begin{equation*}
\begin{split}
\E\left[\|\z_{t} - \nabla F(\w_t)\|^2 \right]
\leq & \frac{2C_g^2 C_f^2}{B_0 \alpha T}  +\frac{4\alpha  C_g^2C_f^2}{B_1}+ \frac{8\left(C_g^2L_f^2+C_f^2L_g^2\right)}{B_1}\frac{\eta^2}{\alpha} + 2C_g^4k^2 \eta^2
\end{split}
\end{equation*}
Finally, by setting $B_0=T^{1/3}$, $\alpha = \frac{B_1^{1/3}}{T^{2/3}}$, $\eta = \frac{B_1^{1/3}}{T^{2/3}}$, we know that
\begin{align*}
     \frac{1}{T}\sum_{t=1}^T \E\left[\Norm{\nabla F(\w_t)}\right] & \leq \frac{F(\w_{1})-F(\w_{T+1})}{\eta T} + \E\left[\frac{2}{T}\sum_{t=1}^T \Norm{\v_t - \nabla F(\w_t)}\right] + \frac{\eta L_F}{2}\leq  \left(\frac{1}{B_1 T}\right)^{1/3}.
\end{align*}

\textbf{MSVR-SP estimator:}
Note that the convergence property of MSVR-SP is almost the same as the MSVR estimator as long as ${\Norm{\w_t-\w_{t+1}}^2}/{\beta_t} = \mathcal{O}(1)$. This is valid for the above setup of hyperparameters, since $
    \frac{ \|\w_{t+1} - \w_{t}\|^2}{\beta_{t+1} } = \eta^2\leq 1$.
This indicates that the result is valid for the MSVR-SP estimator.

\section{Proof of Theorem~5 }
%===================   Lemma  7   ========================
\begin{lemma}\label{lem:7} If $\beta \leq \frac{1}{2}$ and $\beta I \leq \frac{m}{B_1}$, we have:
\begin{equation*}
\begin{split}
  \E\left[\sum_{t=1}^{T} \|\u_{t+1}-\u_{t}\|^2\right]   &\leq   \frac{4B_1 \beta^{2} }{m} \E\left[ \sum_{t=1}^{T}\left\|\u_{t}- g(\w_{t})\right\|^{2}\right] +\frac{11 m^2 C_g^2}{B_1}\sum_{t=1}^{T}\|\w_{t+1}-\w_{t}\|^{2}
\end{split}
\end{equation*}
\end{lemma}

\textbf{Proof}
Following the analysis of Lemma~\ref{lem:3}, we have:
\begin{equation*}
\begin{split}
&\E\left[\|\u_{t+1}-\u_t\|^2 \right] \\
\leq& \frac{2B_1\beta^{2}}{m} \sum_{i=1}^{m} \left(\E\left[ \left\| \widehat g_{i}^{t+1} - g_{i}(\w_{t+1})\right\|^{2}\right]+ \E\left[ \left\| g_{i}(\w_{t+1}) -\u_{t}^{i}\right\|^{2}\right]\right)+\frac{8m^2 C_g^2}{B_1}\|\w_{t+1}-\w_{t}\|^{2}\\ 
\leq& 2B_1\beta^{2} C_g^2\left\|\w_{t+1}-\w_{\tau}\right\|^{2} + \frac{2B_1\beta^{2}}{m} \E\left[ \left\| g(\w_{t+1}) -\u_{t}\right\|^{2}\right]+\frac{8m^2 C_g^2}{B_1}\|\w_{t+1}-\w_{t}\|^{2}
\end{split}
\end{equation*}
So, with $\beta I \leq m / B_1$, we have:
\begin{equation*}
\begin{split}
&\E\left[\sum_{t=1}^{T}\|\u_{t+1}-\u_t\|^2 \right] \\
\leq&  \frac{2B_1\beta^{2}}{m} \E\left[ \sum_{t=1}^{T}\left\| g(\w_{t+1}) -\u_{t}\right\|^{2}\right]+2B_1\beta^{2} C_g^2\sum_{t=1}^{T}\left\|\w_{t+1}-\w_{\tau}\right\|^{2} +\frac{8m^2 C_g^2}{B_1}\sum_{t=1}^{T}\|\w_{t+1}-\w_{t}\|^{2} \\
\leq&  \frac{4B_1\beta^{2}}{m} \E\left[ \sum_{t=1}^{T}\left\| g(\w_{t}) -\u_{t}\right\|^{2}\right]+ 2B_1\beta^{2} C_g^2 I^2 \sum_{t=1}^T   \left\| \w_{t+1}-\w_t\right\|^2+\frac{9m^2 C_g^2}{B_1}\sum_{t=1}^{T}\|\w_{t+1}-\w_{t}\|^{2} \\
\leq&  \frac{4B_1\beta^{2}}{m} \E\left[ \sum_{t=1}^{T}\left\| g(\w_{t}) -\u_{t}\right\|^{2}\right]+\frac{11m^2 C_g^2}{B_1}\sum_{t=1}^{T}\|\w_{t+1}-\w_{t}\|^{2}
\end{split}
\end{equation*}

%We can also replace Lemma~\ref{lem:6} with following lemma.
%===================   Lemma  9   ========================
\begin{lemma} \label{lem:9} 
With $\alpha I \leq 1$ ,  we have:
\begin{equation*}
\begin{split}
     &\E\left[\sum_{t=1}^T \left\|\z_{t} - \frac{1}{m} \sum_{i=1}^m \nabla g_{i}(\w_t)\nabla f_i(\u^i_{t-1})\right\|^2 \right]
    \leq  \frac{1}{\alpha}\left\|\z_{1} - \frac{1}{m} \sum_{i=1}^m \nabla g_{i}(\w_1)\nabla f_i(\u^i_{0})\right\|^2  \\
&\quad\quad\quad\quad\quad\quad+\frac{8C_g^2L_f^2}{mB_1 \alpha} \E\left[\sum_{t=1}^T\|\u_{t} - \u_{t-1}\|^2 \right]+ \frac{8C_f^2 L_g^2}{ B_1\alpha}\E\left[\sum_{t=1}^T\|\w_{t+1} - \w_t\|^2 \right]
\end{split}
\end{equation*}
\end{lemma}
\textbf{Proof}
First, since $\h_t$ is an unbiased estimation of $\frac{1}{m} \sum_{i=1}^m \nabla g_{i}(\w_t)\nabla f_i(\u^i_{t-1})$, we have:
\begin{equation*}
\begin{split} 
    &\E\left[\left\| \h_t - \frac{1}{m} \sum_{i=1}^m \nabla g_{i}(\w_t)\nabla f_i(\u^i_{t-1})\right\|^{2}\right]  \\
    =& \E\left[\left\| \frac{1}{B_1}\sum_{i \in \mathcal{B}_{1}^{t}}  \nabla f_{i}(\u_{t-1}^{i}) \nabla g_{i}(\w_t;\xi_t^{i}) - \frac{1}{B_1}\sum_{i \in \mathcal{B}_{1}^{t}}  \nabla f_{i}(\u_{\tau-1}^{i}) \nabla g_{i}(\w_{\tau};\xi_t^{i}) \right.\right. \\
    &\quad\quad\quad \left.\left. + \frac{1}{m}\sum_{i=1}^{m}  \nabla f_{i}(\u_{\tau-1}^{i}) \nabla g_{i}(\w_{\tau})- \frac{1}{m} \sum_{i=1}^m \nabla g_{i}(\w_t)\nabla f_i(\u^i_{t-1})\right\|^{2}\right] \\
    \leq& \frac{1}{B_1^2}\sum_{i \in \mathcal{B}_{1}^{t}} \E\left[\left\|  \nabla f_{i}(\u_{t-1}^{i}) \nabla g_{i}(\w_t;\xi_t^{i}) -  \nabla f_{i}(\u_{\tau-1}^{i}) \nabla g_{i}(\w_{\tau};\xi_t^{i}) \right\|^{2}\right]\\
    =& \frac{2C_f^2L_g^2}{B_1} \left\|\w_{t} -\w_{\tau} \right\|^{2} +  \frac{2C_g^2L_f^2}{m B_1}   \left\|\u_{t-1} -\u_{\tau-1} \right\|^{2}
\end{split}
\end{equation*}
Next, we have:
\begin{equation*}
\begin{split}
&\E\left[\|\z_{t} - \frac{1}{m} \sum_{i=1}^m \nabla g_{i}(\w_t)\nabla f_i(\u^i_{t-1})\|^2 \right]\\
=&\E\left[\bigg\| (1-\alpha) \left(\z_{t-1} -  \frac{1}{m} \sum_{i=1}^m \nabla g_{i}(\w_{t-1})\nabla f_i(\u^i_{t-2})\right) + \alpha\left(\h_t - \frac{1}{m} \sum_{i=1}^m \nabla g_{i}(\w_t)\nabla f_i(\u^i_{t-1}) \right)\right.\\
&\left. + (1-\alpha)\left(\frac{1}{B_1}\sum_{i \in \mathcal{B}_{1}^{t}} \nabla g_{i}(\w_t;\xi_t^{i}) \nabla f_{i}(\u_{t-1}^{i})- \frac{1}{B_1}\sum_{i \in \mathcal{B}_{1}^{t}}\nabla g_{i}(\w_{t-1};\xi_t^{i})\nabla f_{i}(\u_{t-2}^{i})\right.\right.\\
&\left.\left. \quad\quad\quad\quad\quad  - \frac{1}{m} \sum_{i=1}^m \nabla g_{i}(\w_{t})\nabla f_i(\u^i_{t-1}) + \frac{1}{m} \sum_{i=1}^m \nabla g_{i}(\w_{t-1})\nabla f_i(\u^i_{t-2})\right)  \bigg\|^2\right]\\
& \leq \E\bigg[(1-\alpha)^2 \left\|\z_{t-1} - \frac{1}{m} \sum_{i=1}^m \nabla g_{i}(\w_{t-1})\nabla f_i(\u^i_{t-2})\right\|^2  + \frac{4\alpha^2 C_f^2L_g^2}{B_1} \left\|\w_{t} -\w_{\tau} \right\|^{2}  \\
&\quad +  \frac{4\alpha^2 C_g^2L_f^2}{m B_1}   \left\|\u_{t-1} -\u_{\tau-1} \right\|^{2} + \frac{2(1-\alpha)^2}{B_1^2}\sum_{i \in \mathcal{B}_{1}^{t}}\left\| \nabla g_{i}(\w_t;\xi_t^{i}) \nabla f_{i}(\u_{t-1}^{i})- \nabla g_{i}(\w_{t-1};\xi_t^{i})\nabla f_{i}(\u_{t-2}^{i})\right\|^2\bigg]\\
& \leq\E\bigg[ (1-\alpha) \|\z_{t-1} - \frac{1}{m} \sum_{i=1}^m \nabla g_{i}(\w_{t-1})\nabla f_i(\u^i_{t-2})\|^2  + \frac{4\alpha^2 C_f^2L_g^2}{B_1} \left\|\w_{t} -\w_{\tau} \right\|^{2}\\
&\quad+  \frac{4\alpha^2 C_g^2L_f^2}{mB_1}   \left\|\u_{t-1} -\u_{\tau-1} \right\|^{2}  + \frac{4C_g^2L_f^2}{mB_1} \|\u_{t-1} - \u_{t-2}\|^2 + \frac{4C_f^2 L_g^2}{B_1}\|\w_{t} - \w_{t-1}\|^2 \bigg]
\end{split}
\end{equation*}
The first inequality is due to the fact that the last two terms equal zero in expectation.

Summing up, we have:
\begin{equation*}
\begin{split}
    &\sum_{t=1}^T \left\|\z_{t} - \frac{1}{m} \sum_{i=1}^m \nabla g_{i}(\w_t)\nabla f_i(\u^i_{t-1})\right\|^2 \\
    \leq & \frac{1}{\alpha}\left\|\z_{1} - \frac{1}{m} \sum_{i=1}^m \nabla g_{i}(\w_1)\nabla f_i(\u^i_{0})\right\|^2  + \frac{4\alpha C_f^2L_g^2}{B_1} \sum_{t=1}^T \left\|\w_{t+1} -\w_{\tau} \right\|^{2}  \\
      &+\frac{4\alpha C_g^2L_f^2}{mB_1}   \sum_{t=1}^T \left\|\u_{t} -\u_{\tau-1} \right\|^{2} + \frac{4C_g^2L_f^2}{m B_1\alpha} \sum_{t=1}^T\|\u_{t} - \u_{t-1}\|^2 + \frac{4C_f^2 L_g^2}{B_1 \alpha}\sum_{t=1}^T\|\w_{t+1} - \w_t\|^2  \\
        \leq & \frac{1}{\alpha}\left\|\z_{1} - \frac{1}{m} \sum_{i=1}^m \nabla g_{i}(\w_1)\nabla f_i(\u^i_{0})\right\|^2  +
      \frac{4\alpha C_f^2L_g^2 I^2}{B_1} \sum_{t=1}^T \left\|\w_{t+1} -\w_{t} \right\|^{2}  \\
      &+\frac{4\alpha C_g^2L_f^2 I^2}{mB_1}   \sum_{t=1}^T \left\|\u_{t} -\u_{t-1} \right\|^{2}  + \frac{4C_g^2L_f^2}{m B_1\alpha } \sum_{t=1}^T\|\u_{t} - \u_{t-1}\|^2 + \frac{4C_f^2 L_g^2}{B_1 \alpha}\sum_{t=1}^T\|\w_{t+1} - \w_t\|^2  \\
       \leq & \frac{1}{\alpha}\left\|\z_{1} - \frac{1}{m} \sum_{i=1}^m \nabla g_{i}(\w_1)\nabla f_i(\u^i_{0})\right\|^2  + \frac{8C_g^2L_f^2}{m B_1\alpha} \sum_{t=1}^T\|\u_{t} - \u_{t-1}\|^2 + \frac{8C_f^2 L_g^2}{ B_1\alpha}\sum_{t=1}^T\|\w_{t+1} - \w_t\|^2 
\end{split}
\end{equation*}
The last inequality is due to $\alpha I \leq 1$.
%\end{proof}

\noindent \textbf{The rest proof of Theorem~5:} 
First, we have that:
\begin{equation*}
\begin{split}
\sum_{t=1}^{T}\|\z_{t} - \nabla F(\w_t)\|^2 
&\leq 4\sum_{t=1}^{T}\bigg\|\z_{t} -  \frac{1}{m} \sum_{i=1}^m \nabla g_{i}(\w_t)\nabla f_i(\u^i_{t-1})\bigg\|^2 \\
&\quad + \frac{4C_g^2L_f^2}{m}\sum_{t=1}^{T} \|  \u_t - \u_{t-1} \|^2 +\frac{2 C_g^2L_f^2}{m} \sum_{t=1}^{T}\|  \u_t - g(\w_t) \|^2 
\end{split}
\end{equation*}
We use Lemma~\ref{lem:9} to replace $\sum_{t=1}^T \left\|\z_{t} - \frac{1}{m} \sum_{i=1}^m \nabla g_{i}(\w_t)\nabla f_i(\u^i_{t-1})\right\|^2$:
\begin{equation*}
\begin{split}
    &\E\left[\sum_{t=1}^{T}\|\z_{t} - \nabla F(\w_t)\|^2 \right] \\
    \leq& \frac{4}{\alpha}\left\|\z_{1} - \frac{1}{m} \sum_{i=1}^m \nabla g_{i}(\w_1)\nabla f_i(\u^i_{0})\right\|^2 +\frac{32C_g^2L_f^2}{m B_1\alpha} \sum_{t=1}^T\|\u_{t} - \u_{t-1}\|^2   \\
    &\quad + \frac{32C_f^2 L_g^2}{ B_1\alpha}\sum_{t=1}^T\|\w_{t+1} - \w_t\|^2 + \frac{4C_g^2L_f^2}{m} \sum_{t=1}^{T} \|  \u_t - \u_{t-1} \|^2 +\frac{2 C_g^2L_f^2}{m} \sum_{t=1}^{T}\|  \u_t - g(\w_t) \|^2 \\
    \leq &\frac{4}{\alpha}\left\|\z_{1} - \frac{1}{m} \sum_{i=1}^m \nabla g_{i}(\w_1)\nabla f_i(\u^i_{0})\right\|^2 +\frac{36C_g^2L_f^2}{m \alpha} \E\left[\sum_{t=1}^T\|\u_{t} - \u_{t-1}\|^2 \right]  \\
    &\quad  +\frac{32C_f^2 L_g^2}{B_1 \alpha}\E\left[\sum_{t=1}^T\|\w_{t+1} - \w_t\|^2 \right]+\frac{2 C_g^2L_f^2}{m} \E\left[\sum_{t=1}^{T}\|  \u_t - g(\w_t) \|^2 \right]
\end{split}
\end{equation*}
Set $\beta B_1 \leq m \alpha$. We use Lemma~\ref{lem:7} to replace $\E\left[\sum_{t=1}^{T} \|\u_{t}-\u_{t-1}\|^2\right]$ (set $\u_0 = \u_1$):
\begin{equation*}
\begin{split}
    &\E\left[\sum_{t=1}^{T}\|\z_{t} - \nabla F(\w_t)\|^2 \right] \\
    \leq &\frac{4}{\alpha}\left\|\z_{1} - \frac{1}{m} \sum_{i=1}^m \nabla g_{i}(\w_1)\nabla f_i(\u^i_{0})\right\|^2 + \frac{144(C_g^2L_f^2)B_1 \beta^{2}}{m^2 \alpha}\E\left[ \sum_{t=1}^{T}\left\|\u_{t}- g(\w_{t})\right\|^{2}\right]   \\
    & +\frac{396 m(C_g^4L_f^2)}{\alpha B_1}\sum_{t=1}^{T}\|\w_{t+1}-\w_{t}\|^{2} + \frac{32C_f^2 L_g^2}{B_1 \alpha}\E\left[\sum_{t=1}^T\|\w_{t+1} - \w_t\|^2 \right] \\
    &+\frac{2 C_g^2L_f^2}{m} \E\left[\sum_{t=1}^{T}\|  \u_t - g(\w_t) \|^2 \right] \\
    \leq& \frac{4}{\alpha}\left\|\z_{1} - \frac{1}{m} \sum_{i=1}^m \nabla g_{i}(\w_1)\nabla f_i(\u^i_{0})\right\|^2+\frac{428 m C}{\alpha B_1}\E\left[\sum_{t=1}^T\|\w_{t+1} - \w_t\|^2 \right]\\
    &+\frac{146 C_g^2L_f^2}{m} \E\left[\sum_{t=1}^{T}\|  \u_t - g(\w_t) \|^2 \right] \\
    \leq &\frac{4}{\alpha}\left\|\z_{1} - \frac{1}{m} \sum_{i=1}^m \nabla g_{i}(\w_1)\nabla f_i(\u^i_{0})\right\|^2 +\frac{428 m C}{\alpha B_1} \sum_{t=1}^T\|\w_{t+1} - \w_t\|^2 \\
    &+ \frac{2190 m^{2} C_g^4L_f^2}{B_1^2 \beta} \sum_{t=1}^{T}\left\|\w_{t+1} - \w_{t}\right\|^{2}\\
    \leq &\frac{4}{\alpha}\left\|\z_{1} - \frac{1}{m} \sum_{i=1}^m \nabla g_{i}(\w_1)\nabla f_i(\u^i_{0})\right\|^2  +\frac{2618 m^{2} C}{B_1^2 \beta} \sum_{t=1}^{T}\left\|\w_{t+1} - \w_{t}\right\|^{2}\\
    \leq &\frac{2618 m^{2} C}{B_1^2 \beta} \sum_{t=1}^{T}\left\|\w_{t+1} - \w_{t}\right\|^{2}\\
\end{split}
\end{equation*}
Set $\frac{2618 m^{2} C \eta^{2}}{B_1^2 \beta} \leq \frac{1}{2}$. We have $
    \mathbb{E}\left[ \sum_{t=1}^{T} \|\z_{t} - \nabla F(\w_{t})\|^2 \right] \leq \frac{1}{2} \sum_{t=1}^{T}\left\|\z_{t}\right\|^{2}$.
    
\noindent According to Lemma~\ref{lem:1}, we have:
\begin{equation*}
\begin{split}
    \frac{1}{T}\E\left[\sum_{t=1}^{T}\|\nabla F(\w_t) \|^2\right] \leq& \frac{2 F(\w_1)}{\eta T} + \frac{1}{T}\sum_{t=1}^{T} \E\left[\|\z_{t}- \nabla F(\w_t) \|^2\right] -\frac{1}{2T} \sum_{t=1}^{T}\left\|\z_{t}\right\|^{2} 
    \leq \frac{2 F(\w_1)}{\eta T} 
\end{split}
\end{equation*}
Note that the sample complexity is $\left(B_1 T + \frac{m n T}{I}\right)$. To ensure the first term and the second term are in the same order, we set $I = \left(\frac{m n}{B_1 }\right)$. Also, since we assume that $\alpha I \leq 1$ and $\beta I \leq \frac{m}{B_1}$, we directly set $\alpha = \frac{B_1}{m n}$ and $\beta = \frac{1}{n}$. This setting also satisfies the requirement $B_1 \beta \leq m \alpha$. We also require $\frac{2618 m^2 C \eta^2}{B_1^2 \beta} \leq \frac{1}{2}$. So, we set $\eta = \mathcal{O}(\frac{B_1 }{m \sqrt{n}})$ and we can ensure that 
\begin{equation*}
\begin{split}
    \E\left[\frac{1}{T}\sum_{t=1}^{T}\|\nabla F(\w_t) \|\right] 
    \leq \mathcal{O}\left(\frac{m^{1/2}n^{1/4}}{ (B_1 T)^{1/2}}\right) 
\end{split}
\end{equation*}

\newpage
\section{Proof of Theorem~6}
\paragraph{Non-linear $\nabla$ f} According to the previous analysis, if $\beta \leq \frac{1}{2}$ and $\beta I \leq \frac{m}{B_1}$, we have:
\begin{equation*}
\begin{split}
  \E\left[\sum_{t=1}^{T} \|\u_{t+1}-\u_{t}\|^2\right]   &\leq   \frac{4B_1 \beta^{2} }{m} \E\left[ \sum_{t=1}^{T}\left\|\u_{t}- g(\w_{t})\right\|^{2}\right] +\frac{11 m^2 C_g^2}{B_1}\sum_{t=1}^{T}\|\w_{t+1}-\w_{t}\|^{2} \\
  & \leq \frac{60 m^2 C_g^2 \eta^2 T \beta}{B_1 }  +\frac{11 m^2 C_g^2 \eta^2 T}{B_1} \leq \frac{71 m^2 C_g^2 \eta^2 T}{B_1 }
\end{split}
\end{equation*}
Also, with $\alpha I \leq 1$ ,  we have:
\begin{equation*}
\begin{split}
     &\E\left[\sum_{t=1}^T \left\|\z_{t} - \frac{1}{m} \sum_{i=1}^m \nabla g_{i}(\w_t)\nabla f_i(\u^i_{t-1})\right\|^2 \right]\\
    \leq & \frac{1}{\alpha}\left\|\z_{1} - \frac{1}{m} \sum_{i=1}^m \nabla g_{i}(\w_1)\nabla f_i(\u^i_{0})\right\|^2+\frac{8C_g^2L_f^2}{m B_1\alpha} \E\left[\sum_{t=1}^T\|\u_{t} - \u_{t-1}\|^2 \right]+ \frac{8C_f^2 L_g^2}{B_1 \alpha}\E\left[\sum_{t=1}^T\|\w_{t+1} - \w_t\|^2 \right]
\end{split}
\end{equation*}
by setting $\beta B_1 \leq m \alpha$
\begin{equation*}
\begin{split}
&\sum_{t=1}^{T}\|\z_{t} - \nabla F(\w_t)\|^2 \\
\leq &4\sum_{t=1}^{T}\bigg\|\z_{t} -  \frac{1}{m} \sum_{i=1}^m \nabla g_{i}(\w_t)\nabla f_i(\u^i_{t-1})\bigg\|^2  + \frac{4C_g^2L_f^2}{m}\sum_{t=1}^{T} \|  \u_t - \u_{t-1} \|^2 +\frac{2 C_g^2L_f^2}{m} \sum_{t=1}^{T}\|  \u_t - g(\w_t) \|^2 \\
\leq & \frac{4}{\alpha}\left\|\z_{1} - \frac{1}{m} \sum_{i=1}^m \nabla g_{i}(\w_1)\nabla f_i(\u^i_{0})\right\|^2+\frac{32C_g^2L_f^2}{m B_1\alpha} \E\left[\sum_{t=1}^T\|\u_{t} - \u_{t-1}\|^2 \right]+ \frac{32C_f^2 L_g^2}{B_1 \alpha}\E\left[\sum_{t=1}^T\|\w_{t+1} - \w_t\|^2 \right] \\
& + \frac{4C_g^2L_f^2}{m}\sum_{t=1}^{T} \|  \u_t - \u_{t-1} \|^2 +\frac{2 C_g^2L_f^2}{m} \sum_{t=1}^{T}\|  \u_t - g(\w_t) \|^2 \\
\leq & \frac{4C_g^2C_f^2}{\alpha B_0 } + \frac{36C_g^2L_f^2}{m \alpha} \E\left[\sum_{t=1}^T\|\u_{t} - \u_{t-1}\|^2 \right] + \frac{32C_f^2 L_g^2 \eta^2 T}{ B_1\alpha}+\frac{2 C_g^2L_f^2}{m} \sum_{t=1}^{T}\|  \u_t - g(\w_t) \|^2 \\
\leq & \frac{4C_g^2C_f^2}{\alpha B_0 } + \frac{2556C_g^4L_f^2 m  \eta^2 T}{ \alpha B_1}  + \frac{32C_f^2 L_g^2 \eta^2 T}{B_1 \alpha}+\frac{30 m^2 C_g^4L_f^2 \eta^2 T}{B_1^2 \beta} \\
\leq & \frac{4C_g^2C_f^2}{\alpha B_0 } + \frac{2556(C_g^4L_f^2+C_fL_g^2) m  \eta^2 T}{ \alpha B_1} +\frac{30 m^2 C_g^4L_f^2 \eta^2 T}{B_1^2 \beta} \\
\leq & \frac{4C_g^2C_f^2}{\alpha B_0 } + \frac{2586(C_g^4L_f^2+C_fL_g^2) m^2   \eta^2 T}{  B_1^2 \beta} 
\end{split}
\end{equation*}
Finally, by setting $I=\frac{mn}{B_1}$, $\alpha = \frac{B_1}{mn}$, $\beta = \frac{1}{n}$, $\eta = \frac{B_1^{1/2}}{m^{1/2}n^{1/4}T^{1/2}}$, we know that
\begin{align*}
     \frac{1}{T}\sum_{t=1}^T \E\left[\Norm{\nabla F(\w_t)}\right] & \leq \frac{F(\w_{1})-F(\w_{T+1})}{\eta T} + \E\left[\frac{2}{T}\sum_{t=1}^T \Norm{\v_t - \nabla F(\w_t)}\right] + \frac{\eta L_F}{2}\leq  \left(\frac{m\sqrt{n}}{B_1 T}\right)^{1/2}.
\end{align*}

\newpage
\paragraph{Linear $\nabla f$ } Suppose that $\nabla f$ is linear such that $\nabla f(\x) = k\x$. We first decompose the gradient estimation error as follows:
\begin{equation*}
\begin{split}
&\E\left[\|\z_{t} - \nabla F(\w_t)\|^2 \right]\\
\leq & \E\Bigg[2\bigg\|\z_{t} -  \frac{1}{m} \sum_{i=1}^m \nabla g_{i}(\w_t)\nabla f_i(g_i(\w_{t-1}))\bigg\|^2 + 2\bigg\| \frac{1}{m} \sum_{i=1}^m \nabla g_{i}(\w_t)\nabla f_i(g_i(\w_{t-1})) -  \frac{1}{m} \sum_{i=1}^m \nabla g_i(\w_{t})\nabla f_i(g_i(\w_{t})) \bigg\|^2\Bigg]\\\leq & \E\Bigg[2\bigg\|\z_{t} -  \frac{1}{m} \sum_{i=1}^m \nabla g_{i}(\w_t)\nabla f_i(g_i(\w_{t-1}))\bigg\|^2 + \frac{2C_g^2}{m} \sum_{i=1}^m\bigg\|  \nabla f_i(g_i(\w_{t-1})) -  \nabla f_i(g_i(\w_{t})) \bigg\|^2\Bigg]\\
\leq & \E\Bigg[2\bigg\|\z_{t} -  \frac{1}{m} \sum_{i=1}^m \nabla g_{i}(\w_t)\nabla f_i(g_i(\w_{t-1}))\bigg\|^2 + \frac{2C_g^2k^2}{m} \sum_{i=1}^m\bigg\|  g_i(\w_{t-1}) -  g_i(\w_{t}) \bigg\|^2\Bigg]\\
\leq & \E\Bigg[2\bigg\|\z_{t} -  \frac{1}{m} \sum_{i=1}^m \nabla g_{i}(\w_t)\nabla f_i(g_i(\w_{t-1}))\bigg\|^2 + 2C_g^4k^2 \bigg\|  \w_{t-1} -  \w_{t} \bigg\|^2\Bigg]\\
\leq & 2\E\left[\bigg\|\z_{t} -  \frac{1}{m} \sum_{i=1}^m \nabla g_{i}(\w_t)\nabla f_i(g_i(\w_{t-1}))\bigg\|^2\right] + 2C_g^4k^2 \eta^2
\end{split}
\end{equation*}
Note that in this case, we set $\widehat{\u}_t^i = g_i(\w_t;\xi_t^i)$.
Next, we have the following ‌recurrence.
First, since $\h_t$ is an unbiased estimation of $\frac{1}{m} \sum_{i=1}^m \nabla g_{i}(\w_t)\nabla f_i(\u^i_{t-1})$, we have:
\begin{equation*}
\begin{split} 
    &\E\left[\left\| \h_t - \frac{1}{m} \sum_{i=1}^m \nabla g_{i}(\w_t)\nabla f_i(g_i(\w_{t-1}^{i}))\right\|^{2}\right]  \\
    =& \E\left[\left\| \frac{1}{B_1}\sum_{i \in \mathcal{B}_{1}^{t}}  \nabla f_{i}(g_i(\w_{t-1}^{i};\xi_{t-1}^{i})) \nabla g_{i}(\w_t;\xi_t^{i}) - \frac{1}{B_1}\sum_{i \in \mathcal{B}_{1}^{t}}  \nabla f_{i}(g_i(\w_{\tau-1}^{i};\xi_{t-1}^{i})) \nabla g_{i}(\w_{\tau};\xi_t^{i}) \right.\right. \\
    &\quad\quad\quad \left.\left. + \frac{1}{m}\sum_{i=1}^{m}  \nabla f_{i}(g_i(\w_{\tau-1}^{i})) \nabla g_{i}(\w_{\tau})- \frac{1}{m} \sum_{i=1}^m \nabla g_{i}(\w_t)\nabla f_i(g_i(\w_{t-1}^{i}))\right\|^{2}\right] \\
    =& k^2\E\left[\left\| \frac{1}{B_1}\sum_{i \in \mathcal{B}_{1}^{t}}  g_i(\w_{t-1}^{i};\xi_{t-1}^{i}) \nabla g_{i}(\w_t;\xi_t^{i}) - \frac{1}{B_1}\sum_{i \in \mathcal{B}_{1}^{t}}  g_i(\w_{\tau-1}^{i};\xi_{t-1}^{i}) \nabla g_{i}(\w_{\tau};\xi_t^{i}) \right.\right. \\
    &\quad\quad\quad \left.\left. + \frac{1}{m}\sum_{i=1}^{m}  g_i(\w_{\tau-1}^{i}) \nabla g_{i}(\w_{\tau})- \frac{1}{m} \sum_{i=1}^m \nabla g_{i}(\w_t)g_i(\w_{t-1}^{i})\right\|^{2}\right] \\
    \leq& k^2 \frac{1}{B_1^2}\sum_{i \in \mathcal{B}_{1}^{t}}\E\left[\left\|   g_i(\w_{t-1}^{i};\xi_{t-1}^{i}) \nabla g_{i}(\w_t;\xi_t^{i}) -   g_i(\w_{\tau-1}^{i};\xi_{t-1}^{i}) \nabla g_{i}(\w_{\tau};\xi_t^{i}) \right\|^{2}\right]\\
    \leq& \frac{1}{B_1^2}\sum_{i \in \mathcal{B}_{1}^{t}}\E\left[\left\|   \nabla f_{i}(g_i(\w_{t-1}^{i};\xi_{t-1}^{i})) \nabla g_{i}(\w_t;\xi_t^{i}) -   \nabla f_{i}(g_i(\w_{\tau-1}^{i};\xi_{t-1}^{i})) \nabla g_{i}(\w_{\tau};\xi_t^{i}) \right\|^{2}\right]\\
    =& \frac{1}{B_1}\left(2C_f^2L_g^2 \left\|\w_{t} -\w_{\tau} \right\|^{2} +  {2C_g^4L_f^2}   \left\|\w_{t-1} -\w_{\tau-1} \right\|^{2}\right)
\end{split}
\end{equation*}
Summing up, we have that
\begin{equation*}
\begin{split} 
    &\frac{1}{T}\sum_{t=1}^T\E\left[\left\| \h_t - \frac{1}{m} \sum_{i=1}^m \nabla g_{i}(\w_t)\nabla f_i(g_i(\w_{t-1}^{i}))\right\|^{2}\right]  \\
    \leq &
    \frac{1}{B_1 T}\sum_{t=1}^T\left(2C_f^2L_g^2 \left\|\w_{t} -\w_{\tau} \right\|^{2} +  {2C_g^4L_f^2}   \left\|\w_{t-1} -\w_{\tau-1} \right\|^{2}\right)
    \leq 
    \frac{2}{B_1}\left(C_f^2L_g^2  +  {C_g^4L_f^2}   \right)I^2 \eta^2 
\end{split}
\end{equation*}
Next, we have that
\begin{equation*}
\begin{split}
&\E\left[\|\z_{t} - \frac{1}{m} \sum_{i=1}^m \nabla g_{i}(\w_t)\nabla f_i(g_{i}(\w_{t-1}))\|^2 \right]\\
=&\E\left[\bigg\| (1-\alpha) \left(\z_{t-1} -  \frac{1}{m} \sum_{i=1}^m \nabla g_{i}(\w_{t-1})\nabla f_i(g_{i}(\w_{t-2}))\right) + \alpha\left(\textbf{h}_t - \frac{1}{m} \sum_{i=1}^m \nabla g_{i}(\w_t)\nabla f_i(g_{i}(\w_{t-1})) \right)\right.\\
&\left. + (1-\alpha)\left(\frac{1}{B_1}\sum_{i \in \mathcal{B}_{1}^{t}} \nabla g_{i}(\w_t;\xi_t^{i}) \nabla f_{i}(g_{i}(\w_{t-1};\xi_{t-1}^{i}))- \frac{1}{B_1}\sum_{i \in \mathcal{B}_{1}^{t}}\nabla g_{i}(\w_{t-1};\xi_t^{i})\nabla f_{i}(g_{i}(\w_{t-2};\xi_{t-2}^{i}))\right.\right.\\
&\left.\left. \quad\quad\quad\quad\quad  - \frac{1}{m} \sum_{i=1}^m \nabla g_{i}(\w_{t})\nabla f_i(g_{i}(\w_{t-1})) + \frac{1}{m} \sum_{i=1}^m \nabla g_{i}(\w_{t-1})\nabla f_i(g_{i}(\w_{t-2}))\right)  \bigg\|^2\right]\\
\leq &(1-\alpha)\E\left[\left\|  \z_{t-1} -  \frac{1}{m} \sum_{i=1}^m \nabla g_{i}(\w_{t-1})\nabla f_i(g_{i}(\w_{t-2}))\right\|^2\right] +2\alpha^2 \E\left[\left\| \textbf{h}_t - \frac{1}{m} \sum_{i=1}^m \nabla g_{i}(\w_t)\nabla f_i(g_{i}(\w_{t-1})) \right\|^2\right]\\
&+ \frac{2(1-\alpha)^2k^2}{B_1^2}\sum_{i \in \mathcal{B}_{1}^{t}}\E\left[\left\|   \nabla g_{i}(\w_t;\xi_t^{i}) g_{i}(\w_{t-1};\xi_{t-1}^{i})- \nabla g_{i}(\w_{t-1};\xi_t^{i})g_{i}(\w_{t-2};\xi_{t-2}^{i}) \right\|^2\right]\\
\leq &(1-\alpha)\E\left[\left\| \z_{t-1} -  \frac{1}{m} \sum_{i=1}^m \nabla g_{i}(\w_{t-1})\nabla f_i(g_{i}(\w_{t-2}))\right\|^2\right]  + \frac{2\left(2C_f^4+2C_f^2L_g^2\right)}{B_1}\eta^2\\
&+2\alpha^2 \E\left[\left\| \textbf{h}_t - \frac{1}{m} \sum_{i=1}^m \nabla g_{i}(\w_t)\nabla f_i(g_{i}(\w_{t-1})) \right\|^2\right]
\end{split}
\end{equation*}
As a result, we have that
\begin{equation*}
\begin{split}
&\frac{1}{T}\sum_{t=1}^T \E\left[\|\z_{t} - \frac{1}{m} \sum_{i=1}^m \nabla g_{i}(\w_t)\nabla f_i(g_{i}(\w_{t-1}))\|^2 \right]
\leq \frac{C_g^2 C_f^2}{B_0 \alpha T} + \frac{4\left(C_f^4+C_f^2L_g^2\right)}{B_1}\frac{\eta^2}{\alpha} +\frac{4 I^2 \eta^2 \alpha}{B_1 }\left(C_f^2L_g^2  +  {C_g^4L_f^2}   \right) 
\end{split}
\end{equation*}
as well as
\begin{equation*}
\begin{split}
\frac{1}{T}\sum_{t=1}^T \E\left[\|\z_{t} - \nabla F(\w_t)\|^2 \right]
\leq & \frac{2C_g^2 C_f^2}{B_0 \alpha T}  + \frac{8\left(C_f^4+C_f^2L_g^2\right)}{B_1}\frac{\eta^2}{\alpha} +\frac{8 I^2 \eta^2 \alpha}{B_1 }\left(C_f^2L_g^2  +  {C_g^4L_f^2}   \right)+ 2C_g^4k^2 \eta^2
\end{split}
\end{equation*}
Finally, by setting $I=\frac{mn}{B_1}$, $\alpha = \frac{B_1}{mn}$, $\eta = \frac{B_1^{1/2}}{m^{1/4}n^{1/4}T^{1/2}}$, we know that, we know that
\begin{align*}
     \frac{1}{T}\sum_{t=1}^T \E\left[\Norm{\nabla F(\w_t)}\right] & \leq \frac{F(\w_{1})-F(\w_{T+1})}{\eta T} + \E\left[\frac{2}{T}\sum_{t=1}^T \Norm{\v_t - \nabla F(\w_t)}\right] + \frac{\eta L_F}{2}\leq  \left(\frac{\sqrt{mn}}{B_1 T}\right)^{1/2}.
\end{align*}

\section{Proof of Theorem~7}
Below, the numerical subscripts would denote the stage index $\{1, \ldots, S\}$. Denote that $\Delta_s =  \left\|\z_{s} - \frac{1}{m} \sum_{i=1}^m \nabla g_{i}(\w_s)\nabla f_i(\u^i_{s-1})\right\|^2 + \frac{1}{m} \left\|\u_{s} -  g(\w_s)\right\|^2$. Let's consider the first stage, $ \Delta_1 \leq 2C = \mu \epsilon_1$ and $F(\w_1)-F_{*} \leq \epsilon_1$, where $\epsilon_1 = \max \{\frac{2C}{\mu}, \Delta_F \}$. Starting from the second stage, we would prove by induction.

Suppose at stage $s-1$, we have $ \Delta_{s-1} \leq \mu \epsilon_{s-1}$ and $F\left(\w_{s-1}\right)-F_{*} \leq \epsilon_{s-1}$. Then at $s$ stage, by setting $264 C =  c_0$, and hyperparameters as
\begin{align*}
    \beta_s = \mathcal{O}\left(\mu \epsilon_s\right); \quad \alpha_s = \mathcal{O}\left(B_1\mu \epsilon_s\right); \quad  \eta_s = \mathcal{O}\left(\frac{B_1}{m}\mu \epsilon_s\right); \quad T_s = \mathcal{O}\left(\frac{m}{B_1 \mu^2 \epsilon_s}\right),
\end{align*}
we have:
\begin{equation*}
    \begin{split}
    \E\left[\Gamma_{t+1}-\Gamma_{t}\right]
        \leq & \E\left[- \frac{\eta_t}{2} \|\nabla F(\w_t)\|^2  +\frac{2 \alpha_{t+1}^{2} C}{B_{1} c_0}  + \frac{4 B_1^2 \beta_{t+1}^{2}  C}{m^2 c_0 \eta_{t}}\right],
    \quad \frac{1}{T}\E\left[\sum_{t=1}^{T}\|\nabla F(\w_t) \|^2\right] \leq \mu \epsilon_{s}
\end{split}
\end{equation*}
Due to the PL condition, we have:
\begin{equation*}
\begin{split}
    F(\w_s) - F_{*} \leq \frac{1}{2\mu T }\E\left[\sum_{t=1}^{T}\|\nabla F(\w_t) \|^2\right]  \leq \epsilon_s
\end{split}
\end{equation*}
On the other hand, we have:
\begin{equation*}
\begin{split}
    \Delta_s 
    \leq  \mu \epsilon_s
\end{split} 
\end{equation*}
So, we proved that $F\left(\w_{s}\right)-F_{*} \leq \epsilon_{s}$. That is to say, $F\left(\w_{s}\right)-F_{*} \leq \epsilon$ when $S = \log _{2}\left(\frac{2\epsilon_{1}}{\epsilon}\right) $, and the iteration complexity  is computed as:
\begin{equation*}
\begin{split}
    T_{1}+\sum_{s=2}^{S} T_{s} &=\mathcal{O}\left(\sum_{s=2}^{S} \frac{m}{B_1 \mu^2 \epsilon_s}\right) \\ & \leq  \mathcal{O}\left(\frac{m}{B_1  \mu^2\epsilon}\right)
\end{split}
\end{equation*}
When $F(\w)$ is convex, we define $\hat{F}(\w) = F(\w) + \frac{\mu}{2}\|\w\|^2$. We know that $\hat{F}(\w)$ is $\mu$-strongly convex, which implies $\mu$-PL condition. We have proved: for any $\delta > 0$,  there exist $T=\mathcal{O}\left(\frac{m}{B_1\mu \delta}\right)$ such that $\hat{F}(\w_{T}) - \hat{F}_{*} \leq \delta$.  It indicates that $F(\w_{T}) - F_{*} \leq \delta + \frac{\mu}{2} \|\w_{*}\|^2 -  \frac{\mu}{2} \|\w_{T}\|^2 \leq \delta + \frac{\mu}{2}D$. For any $\epsilon > 0$, if we choose $\mu = \frac{\epsilon}{D}$ and $\delta = \frac{\epsilon}{2}$, we get $F(\w_{T}) - F_{*} \leq \epsilon$, for some $T=\mathcal{O}\left(\frac{m}{B_1\epsilon^3}\right)$.

\newpage
\section{Proof of Theorem~8}
Denote $\Delta_s =  \left\|\z_{s} - \frac{1}{m} \sum_{i=1}^m \nabla g_{i}(\w_s)\nabla f_i(\u^i_{s-1})\right\|^2 + \frac{1}{m} \left\|\u_{s} -  g(\w_s)\right\|^2$. Let's consider the first stage, $ \Delta_1 \leq 2C = \mu \epsilon_1$ and $F(\w_1)-F_{*} \leq \epsilon_1$, where $\epsilon_1 = \max \{\frac{2C}{\mu}, \Delta_F \}$. Starting from the second stage, we would prove by induction.

Suppose at stage $s-1$, we have $ \Delta_{s-1} \leq \mu \epsilon_{s-1}$ and $F\left(\w_{s-1}\right)-F_{*} \leq \epsilon_{s-1}$. Then at $s$ stage, by setting $264 C =  c_0$, $\eta_{s}^2 = \frac{32 B_1^2 \beta_{s}}{m^2 c_0^2 }$, $\alpha_{s} = \frac{B_1 \beta_{s}}{m}$, we have:
\begin{equation*}
    \begin{split}
        \E\left[\Gamma_{t+1}-\Gamma_{t} \right]
         \leq \E\left[ - \frac{\eta_s}{2} \|\nabla F(\w_t)\|^2  + \frac{m^2 \eta_{s}^{3} c_0^4}{512 B_1^2 } \right]
    \end{split}
\end{equation*}
This means that by setting $T_s =  \max\left\{\frac{m c_0^2}{B_1 \mu \sqrt{ \mu \epsilon_s}},\frac{m c_0^4}{B_1 \mu \epsilon_s} \right\}$, $\eta_s = \frac{8 B_1 \sqrt{\mu \epsilon_s}}{m c_0^2}$, we have: 
\begin{equation*}
\begin{split}
    &\frac{1}{T}\E\left[\sum_{t=1}^{T}\|\nabla F(\w_t) \|^2\right] \leq \E\left[\frac{2(\Gamma_{1}-\Gamma_{T+1})}{\eta_s T} + \frac{m^2 c_0^4 \eta_{s}^{2}}{256  B_1^2 } \right]\\
    \leq &\E\left[\frac{2(F(\w_{s-1}) - F_{*})}{\eta_s T} + \frac{2 B_1\Delta_{s-1}}{c_0 \eta_s^2 T m}+ \frac{m^2 c_0^4 \eta_{s}^{2}}{256 B_1^2 } \right]
    \leq 2\mu \epsilon_{s}
\end{split}
\end{equation*}
Due to the PL condition, we have:
\begin{equation*}
\begin{split}
    F(\w_s) - F_{*} \leq \frac{1}{2\mu T }\E\left[\sum_{t=1}^{T}\|\nabla F(\w_t) \|^2\right]  \leq \epsilon_s
\end{split}
\end{equation*}
On the other hand, by setting $\beta_s = \frac{\mu \epsilon_s}{80 C}$ and $\alpha_{s} = \frac{B_1 \beta_{s}}{m}$, we have:
\begin{equation*}
\begin{split}
    \Delta_s \leq&  \frac{2m}{B_1\beta_s T }\Delta_{s-1} +\frac{96 m^2 C}{B_1^2 \beta_s T}\sum_{t=1}^T\left\|\w_{t+1}-\w_{t}\right\|^{2}+ {20 \beta_s C} +\frac{4 m \alpha_s^2 C}{B_1^2 \beta} \\
    \leq & \frac{2m \mu \epsilon_{s-1}}{B_1 \beta_s T}+\frac{96 m^2 \eta_s^2  C}{B_1^2 \beta_s T}\sum_{t=1}^T\left\|\z_{t}\right\|^{2}+ {20 \beta_s C} +\frac{4 m \alpha_s^2 C}{B_1^2 \beta} 
    \leq \mu \epsilon_s
\end{split} 
\end{equation*}
So, we proved that $F\left(\w_{s}\right)-F_{*} \leq \epsilon_{s}$. That is to say, $F\left(\w_{s}\right)-F_{*} \leq \epsilon$ when $S = \log _{2}\left(\frac{2\epsilon_{1}}{\epsilon}\right) $, and the iteration complexity  is computed as:
\begin{equation*}
\begin{split}
    T_{1}+\sum_{s=2}^{S} T_{s} &\stackrel{\mu \geq \epsilon}{=}\mathcal{O}\left(\sum_{s=2}^{S} \frac{m}{B_1 \mu \epsilon_s}\right)  \leq  \mathcal{O}\left(\frac{m}{B_1  \mu \epsilon}\right).
\end{split}
\end{equation*}
When $F(\w)$ is convex, we define $\hat{F}(\w) = F(\w) + \frac{\mu}{2}\|\w\|^2$. We know that $\hat{F}(\w)$ is $\mu$-strongly convex, which implies $\mu$-PL condition. We have proved: for any $\delta > 0$,  there exist $T=\mathcal{O}\left(\frac{m}{B_1\mu \delta}\right)$ such that $\hat{F}(\w_{T}) - \hat{F}_{*} \leq \delta$.  It indicates that $F(\w_{T}) - F_{*} \leq \delta + \frac{\mu}{2} \|\w_{*}\|^2 -  \frac{\mu}{2} \|\w_{T}\|^2 \leq \delta + \frac{\mu}{2}D$. For any $\epsilon > 0$, if we choose $\mu = \frac{\epsilon}{D}$ and $\delta = \frac{\epsilon}{2}$, we get $F(\w_{T}) - F_{*} \leq \epsilon$, for some $T=\mathcal{O}\left(\frac{m}{B_1\epsilon^2}\right)$.

\newpage
\section{Proof of Theorem~9}
The analysis is very similar to Theorem~7. We still use Algorithm~3 but employ MSVR-v3 instead. Also, we do not need to decrease $\alpha$, $\beta$, $\eta$ and increase $T$ during each stage. Let's consider the first stage, $ 4\left\|\z_{1} - \frac{1}{m} \sum_{i=1}^m \nabla g_{i}(\w_1)\nabla f_i(\u^i_{0})\right\|^2 \leq 4C \leq \mu\epsilon_1 $, $\frac{146 C_g^2L_f^2}{m}\left\|\u_{1} -  g(\w_1)\right\|^2 \leq 146 C \leq \mu\epsilon_1$ and $F(\w_1)-F_{*} \leq  \Delta_F \leq \epsilon_1$, where we set $\epsilon_1 = \max\{\Delta_F, \frac{146 C}{\mu}  \}$. Note that the numerical subscripts below denote the stage index $\{1, \ldots, S\}$. Set $\alpha = \frac{B_1}{m n}$, $\beta = \frac{1}{n}$, $\eta = \mathcal{O}(\frac{B_1 }{m \sqrt{n}})$ and $T =\O\left(\max \left\{\frac{m n}{B_1 }, \frac{m \sqrt{n}}{\mu B_1}\right\} \right)$. 

Starting from the second stage, we would prove by induction. Suppose at the stage $s-1$, we have $ F\left(\w_{s-1}\right)-F_{*} \leq \epsilon_{s-1}$, $ 4 \left\|\z_{s-1} - \frac{1}{m} \sum_{i=1}^m \nabla g_{i}(\w_1)\nabla f_i(\u^i_{s-2})\right\|^2 \leq \mu  \epsilon_{s-1}$, and $\frac{146 C_g^2 L_f^2}{m} \left\|\u_{s-1} -  g(\w_{s-1})\right\|^2 \leq \mu \epsilon_{s-1}$. Then at $s$ stage, we have:
\begin{equation*}
\begin{split}
    F(\w_s)-F_{*} &\leq \frac{1}{2\mu} \Norm{\nabla F(\w_s)}^2\leq \epsilon_s
\end{split}
\end{equation*}
On the other hand, following the very similar analysis in Theorem~7, we have:
\begin{equation*}
\begin{split}
    4 \left\|\z_{s} - \frac{1}{m} \sum_{i=1}^m \nabla g_{i}(\w_s)\nabla f_i(\u^i_{s-1})\right\|^2 \leq  \mu \epsilon_s; \quad
    \frac{146 C_g^2 L_f^2}{m} \left\|\u_{s} -  g(\w_{s})\right\|^2 \leq \mu \epsilon_{s} 
\end{split}
\end{equation*}

We proved that $F\left(\w_{s}\right)-F_{*} \leq \epsilon_{s}$. That is to say, $F\left(\w_{S}\right)-F_{*} \leq \epsilon$ when $S = \log _{2}\left(\frac{2\epsilon_{1}}{\epsilon}\right) = \log _{2}\left(\frac{L}{ \epsilon}\right)$, and the iteration complexity until this stage is computed as:
\begin{equation*}
\begin{split}
    \sum_{s=1}^{S} T_{s}  \leq \O\left(\max \left\{\frac{m n}{B_1 }, \frac{m \sqrt{n}}{\mu B_1 } \right\} \cdot
    \log\frac{1}{\epsilon}\right)
\end{split}
\end{equation*}
When $F(\w)$ is convex, we define $\hat{F}(\w) = F(\w) + \frac{\mu}{2}\|\w\|^2$. We know that $\hat{F}(\w)$ is $\mu$-strongly convex, which implies $\mu$-PL condition. We have proved: for any $\delta > 0$,  there exist $T=\O\left(\frac{m \sqrt{n}}{\mu B_1 }  \cdot \log\frac{1}{\epsilon}\right)$ such that $\hat{F}(\w_{T}) - \hat{F}_{*} \leq \delta$.  It indicates that $F(\w_{T}) - F_{*} \leq \delta + \frac{\mu}{2} \|\w_{*}\|^2 -  \frac{\mu}{2} \|\w_{T}\|^2 \leq \delta + \frac{\mu}{2}D$. For any $\epsilon > 0$, if we choose $\mu = \frac{\epsilon}{D}$ and $\delta = \frac{\epsilon}{2}$, we get $F(\w_{T}) - F_{*} \leq \epsilon$, for some $T=\O\left(\frac{m \sqrt{n}}{\epsilon B_1 }  \cdot
    \log\frac{1}{\epsilon}\right)$.

\section{Proof of Theorem~10} 
When the objective function satisfies the PL condition, we would know that $ F(\mathbf{w})-F_{*} \leq \frac{1}{2 \mu}\|\nabla F(\mathbf{w})\|^{2}$.
Let $\eta_t = \eta$, We have already proved that 
\begin{equation*}
\begin{split}
  \frac{1}{T}\sum_{t=1}^T \|\nabla F(\w_t) \|^2  \leq &\frac{2\left(F(\w_1) - F_{*}\right)}{\eta T} +   \frac{1}{T} \sum_{t=1}^T  \|\z_{t}- \nabla F(\w_t)\|^2 - \frac{1}{2T}\sum_{t=1}^T\Norm{\z_t}^2\\
 \leq &\frac{2\left(F(\w_1) - F_*\right)}{\eta T} +   \frac{1}{T} \sum_{t=1}^T  \|\z_{t}- \nabla F(\w_t)\|^2
\end{split}
\end{equation*}
That is to say
\begin{equation*}
\begin{split}
  F(\mathbf{w})-F_{*} \leq \frac{1}{2 \mu}\|\nabla F(\mathbf{w})\|^{2}  
 \leq &\frac{\left(F(\w_1) - F_{*}\right)}{\mu\eta T} +   \frac{1}{2\mu T} \sum_{t=1}^T  \|\z_{t}- \nabla F(\w_t)\|^2
\end{split}
\end{equation*}
In the analysis of Theorem~1, we know that 
\begin{equation*}
\begin{split}
&\frac{1}{T}\sum_{t=1}^T \E\left[\|\z_{t} - \nabla F(\w_t)\|^2 \right]\\
\leq &\frac{\left\|  \z_{1} - \nabla F(\w_{1})\right\|^2}{\alpha T} 
+\frac{2 C_f^2 \sigma^2 + 2 k^2 C_g^2 \sigma^2 + C_f^2 C_g^2}{B_1}\alpha+\frac{4\left(L_F^2 + 4k^2C_g^4 \right) \eta^2}{\alpha^2}\E\left[ \frac{1}{T}\sum_{t=1}^T  \|\z_{t}\|^2 \right].
\end{split}
\end{equation*}
Combining the above and  setting $16\left(L_F^2 + 4k^2C_g^4 \right) \eta^2 \leq \alpha^2$, we know that
\begin{equation*}
\begin{split}
  & \frac{1}{T}\sum_{t=1}^T  \|\z_{t}\|^2 \leq \frac{2}{T}\sum_{t=1}^T \|\nabla F(\w_t) \|^2 +\frac{2}{T} \sum_{t=1}^T  \|\z_{t}- \nabla F(\w_t)\|^2 \\ 
  \leq &\frac{4\left(F(\w_1) - F_{*}\right)}{\eta T} +   \frac{4}{T} \sum_{t=1}^T  \|\z_{t}- \nabla F(\w_t)\|^2 - \frac{1}{T}\sum_{t=1}^T\Norm{\z_t}^2\\
  \leq &\frac{4\left(F(\w_1) - F_{*}\right)}{\eta T} + \frac{4 \left\|  \z_{1} - \nabla F(\w_{1})\right\|^2}{\alpha T} 
+\frac{4(2 C_f^2 \sigma^2 + 2 k^2 C_g^2 \sigma^2 + C_f^2 C_g^2)}{B_1}\alpha\\
&+\left(\frac{16\left(L_F^2 + 4k^2C_g^4 \right) \eta^2}{\alpha^2} -1 \right)\frac{1}{T}\sum_{t=1}^T\Norm{\z_t}^2\\
  \leq &\frac{4\left(F(\w_1) - F_{*}\right)}{\eta T} + \frac{4\left\|  \z_{1} - \nabla F(\w_{1})\right\|^2}{\alpha T} 
+\frac{\alpha}{B_1}4\left(2 C_f^2 \sigma^2 + 2 k^2 C_g^2 \sigma^2 + C_f^2 C_g^2\right)
\end{split}
\end{equation*}
Then, we know that
\begin{equation*}
\begin{split}
&\frac{1}{T}\sum_{t=1}^T \E\left[\|\z_{t} - \nabla F(\w_t)\|^2 \right]\\
\leq &\frac{\left\|  \z_{1} - \nabla F(\w_{1})\right\|^2}{\alpha T} 
+\frac{2 C_f^2 \sigma^2 + 2 k^2 C_g^2 \sigma^2 + C_f^2 C_g^2}{B_1}\alpha+\frac{4\left(L_F^2 + 4k^2C_g^4 \right) \eta^2}{\alpha^2}\E\left[ \frac{1}{T}\sum_{t=1}^T  \|\z_{t}\|^2 \right]\\
\leq &\frac{\left\|  \z_{1} - \nabla F(\w_{1})\right\|^2}{\alpha T} 
+\frac{2 C_f^2 \sigma^2 + 2 k^2 C_g^2 \sigma^2 + C_f^2 C_g^2}{B_1}\alpha\\
&\quad +\frac{4\left(L_F^2 + 4k^2C_g^4 \right) \eta^2}{\alpha^2}\left( \frac{4\left(F(\w_1) - F_{*}\right)}{\eta T} + \frac{4\left\|  \z_{1} - \nabla F(\w_{1})\right\|^2}{\alpha T} +\frac{\alpha}{B_1}4\left(2 C_f^2 \sigma^2 + 2 k^2 C_g^2 \sigma^2 + C_f^2 C_g^2\right)\right)\\
\leq &\left(1+\frac{16\left(L_F^2 + 4k^2C_g^4 \right) \eta^2}{\alpha^2}\right)\left(\frac{\left\|  \z_{1} - \nabla F(\w_{1})\right\|^2}{\alpha T} 
+\frac{2 C_f^2 \sigma^2 + 2 k^2 C_g^2 \sigma^2 + C_f^2 C_g^2}{B_1}\alpha\right) \\
&+\frac{16\left(L_F^2 + 4k^2C_g^4 \right) \eta^2}{\alpha^2}\frac{\left(F(\w_1) - F_{*}\right)}{\eta T} \\
\leq &\frac{2\left\|  \z_{1} - \nabla F(\w_{1})\right\|^2}{\alpha T} 
+\frac{2(2 C_f^2 \sigma^2 + 2 k^2 C_g^2 \sigma^2 + C_f^2 C_g^2)}{B_1}\alpha +\frac{F(\w_1) - F_{*}}{\eta T}.
\end{split}
\end{equation*}
To conclude, we know that for stage $s$ we have
\begin{equation*}
\begin{split}
  F(\mathbf{w_s})-F_{*} 
 &\leq \frac{F(\w_{s-1}) - F_*}{\mu\eta_s T_s} +   \frac{1}{2\mu}   \|\z_{s}- \nabla F(\w_s)\|^2\\
 \|\z_{s}- \nabla F(\w_s)\|^2 &\leq \frac{2\left\|  \z_{s-1} - \nabla F(\w_{s-1})\right\|^2}{\alpha_s T_s} 
+\frac{2(2 C_f^2 \sigma^2 + 2 k^2 C_g^2 \sigma^2 + C_f^2 C_g^2)}{B_1}\alpha_s +\frac{F(\w_{s-1}) - F_{*}}{\eta_s T_s}
\end{split}
\end{equation*}
Suppose for stage $s-1$, we can ensure that $\E\left[ F(\w_{s-1}) - F_* \right]\leq \epsilon_{s-1}$ and $\E\left[ \|\z_{s-1}- \nabla F(\w_{s-1})\|^2 \right]\leq \mu\epsilon_{s-1} $. Next, we prove that $\E\left[F(\w_{s}) - F_* \right]\leq \epsilon_{s}=\epsilon_{s-1}/2$ and $\E\left[ \|\z_{s}- \nabla F(\w_s)\|^2 \right]\leq \mu\epsilon_{s}\leq \mu\epsilon_{s-1}/2 $ for stage $s$.

First, we can prove that
\begin{equation*}
\begin{split}
 \|\z_{s}- \nabla F(\w_s)\|^2 
 \leq &\frac{2\left\|  \z_{s-1} - \nabla F(\w_{s-1})\right\|^2}{\alpha_s T_s} 
+\frac{2(2 C_f^2 \sigma^2 + 2 k^2 C_g^2 \sigma^2 + C_f^2 C_g^2)}{B_1}\alpha_s +\frac{F(\w_{s-1}) - F_{*}}{\eta_s T_s}\\
 \leq &\frac{2 \mu\epsilon_{s-1}}{\alpha_s T_s} 
+\frac{2(2 C_f^2 \sigma^2 + 2 k^2 C_g^2 \sigma^2 + C_f^2 C_g^2)}{B_1}\alpha_s +\frac{\epsilon_{s-1}}{\eta_s T_s}\\
\leq & \frac{\mu \epsilon_s}{2} + \frac{\mu \epsilon_s}{4} + \frac{\mu \epsilon_s}{4}
\leq  \mu \epsilon_s,
\end{split}
\end{equation*}
by setting that
\begin{align*}
     \frac{2(2 C_f^2 \sigma^2 + 2 k^2 C_g^2 \sigma^2 + C_f^2 C_g^2)}{B_1}\alpha_s \leq \frac{\mu \epsilon_s}{4}; \quad 16\left(L_F^2 + 4k^2C_g^4 \right) \eta_s^2 \leq \alpha_s^2; \quad  \quad \alpha_s T_s \ge 8; \quad \mu \eta_s T_s \ge 8. 
\end{align*}
which can be achieved by
\begin{align*}
     \alpha_s =& \frac{B_1\mu  \epsilon_s}{8 \left(2 C_f^2 \sigma^2 + 2 k^2 C_g^2 \sigma^2 + C_f^2 C_g^2\right)}, \eta_s = \frac{B_1\mu  \epsilon_s}{32 \left(2 C_f^2 \sigma^2 + 2 k^2 C_g^2 \sigma^2 + C_f^2 C_g^2\right)\sqrt{\left(L_F^2 + 4k^2C_g^4 \right)}} , 
    \\
    T_s \ge& \max\left\{ \frac{64  \left(2 C_f^2 \sigma^2 + 2 k^2 C_g^2 \sigma^2 + C_f^2 C_g^2\right)}{ B_1 \mu\epsilon_s}, \frac{256 \left(2 C_f^2 \sigma^2 + 2 k^2 C_g^2 \sigma^2 + C_f^2 C_g^2\right)\sqrt{\left(L_F^2 + 4k^2C_g^4 \right)}}{B_1 \mu^2 \epsilon_s} \right\}\\
    \ge & \frac{C}{B_1 \mu^2 \epsilon_s},
\end{align*}
where $C=64 \left(2 C_f^2 \sigma^2 + 2 k^2 C_g^2 \sigma^2 + C_f^2 C_g^2\right) + 256 \left(2 C_f^2 \sigma^2 + 2 k^2 C_g^2 \sigma^2 + C_f^2 C_g^2\right)\sqrt{\left(L_F^2 + 4k^2C_g^4 \right)}$.

Second, we can prove that
\begin{equation*}
\begin{split}
  F(\mathbf{w_s})-F_{*} 
 &\leq \frac{F(\w_{s-1}) - F_*}{\mu\eta_s T_s} +   \frac{1}{2\mu}   \|\z_{s}- \nabla F(\w_s)\|^2 \leq \frac{\epsilon_{s-1}}{\mu\eta_s T_s} + \frac{\mu \epsilon_s}{2\mu} \leq \epsilon_s.
\end{split}
\end{equation*}
For the first stage $s=1$, we can ensure that
\begin{equation*}
\begin{split}
 \|\z_{1}- \nabla F(\w_1)\|^2 &\leq \frac{2\left\|  \z_{0} - \nabla F(\w_{0})\right\|^2}{\alpha_1 T_1} 
+\frac{2(2 C_f^2 \sigma^2 + 2 k^2 C_g^2 \sigma^2 + C_f^2 C_g^2)}{B_1}\alpha_1 +\frac{F(\w_{0}) - F_{*}}{\eta_1 T_1}\\
&\leq \frac{C_g^2C_f^2}{4} 
+\frac{2(2 C_f^2 \sigma^2 + 2 k^2 C_g^2 \sigma^2 + C_f^2 C_g^2)}{B_1} +\frac{\mu \Delta_f}{8}\leq \mu \epsilon_1;\\
F(\mathbf{w_1})-F_{*} 
 &\leq \frac{F(\w_{0}) - F_*}{\mu\eta_1 T_1} +   \frac{1}{2\mu}   \|\z_{1}- \nabla F(\w_1)\|^2 \leq \frac{\Delta}{8} + \frac{\epsilon_1}{2} \leq \epsilon_1
\end{split}
\end{equation*}
where we assume that
\begin{align*}
    \epsilon_1 = \frac{C_g^2C_f^2}{4\mu} 
+\frac{2(2 C_f^2 \sigma^2 + 2 k^2 C_g^2 \sigma^2 + C_f^2 C_g^2)}{\mu B_1} +\frac{ \Delta_f}{4}.
\end{align*}
By Induction, we know that 
\begin{equation*}
\begin{split}
  F(\mathbf{w_S})-F_{*} \leq \epsilon_S = \epsilon_1/2^{S-1} = (2\epsilon_1)/2^{S}
\end{split}
\end{equation*}
To ensure $F(\mathbf{w_S})-F_{*} \leq \epsilon$, we only require $S = \log(\frac{2\epsilon_1}{\epsilon})$ and the sample complexity is
\begin{align*}
    \sum_{s=1}^S T_s = \sum_{s=1}^S \frac{C}{B_1 \mu^2\epsilon_s} = \frac{C}{2 B_1 \mu^2  \epsilon_1} \sum_{s=1}^S 2^s \leq \frac{C}{B_1 \mu^2  \epsilon_1} 2^S = \frac{C}{B_1 \mu^2  \epsilon_1} \frac{2\epsilon_1}{\epsilon} = \frac{2C}{B_1 \mu^2 \epsilon } = \mathcal{O}\left(\frac{1}{B_1\mu^2 \epsilon}\right)
\end{align*}
When $F(\w)$ is convex, we define that $\hat{F}(\w) = F(\w) + \frac{\mu}{2}\|\w\|^2$. We know that $\hat{F}(\w)$ is $\mu$-strongly convex, which implies $\mu$-PL condition. We have proved: for any $\delta > 0$,  there exist $T=\mathcal{O}\left(\frac{1}{B_1\mu^2 \delta}\right)$ such that $\hat{F}(\w_{T}) - \hat{F}_{*} \leq \delta$.  It indicates that $F(\w_{T}) - F_{*} \leq \delta + \frac{\mu}{2} \|\w_{*}\|^2 -  \frac{\mu}{2} \|\w_{T}\|^2 \leq \delta + \frac{\mu}{2}D$. For any $\epsilon > 0$, if we choose $\mu = \frac{\epsilon}{D}$ and $\delta = \frac{\epsilon}{2}$, we get $F(\w_{T}) - F_{*} \leq \epsilon$, for some $T=\mathcal{O}\left(\frac{1}{B_1\epsilon^3}\right)$.

\section{Proof of Theorem~11} 
When the objective function satisfies the PL condition, we have already proved that 
\begin{equation*}
\begin{split}
  F(\mathbf{w})-F_{*} \leq &\frac{F(\w_1) - F_{*}}{\mu\eta T} +   \frac{1}{2\mu T} \sum_{t=1}^T  \|\z_{t}- \nabla F(\w_t)\|^2
\end{split}
\end{equation*}
In the analysis of Theorem~4, we know that 
\begin{equation*}
\begin{split}
  \frac{1}{T}\sum_{t=1}^T \|\nabla F(\w_t) \|^2 
 \leq &\frac{2\left(F(\w_1) - F_*\right)}{\eta T} +   \frac{1}{T} \sum_{t=1}^T  \|\z_{t}- \nabla F(\w_t)\|^2 - \frac{1}{2T}\sum_{t=1}^T\Norm{\z_t}^2 \\
\frac{1}{T}\sum_{t=1}^T \E\left[\|\z_{t} - \nabla F(\w_t)\|^2 \right]
\leq & \frac{2\left\|\z_{1} - \nabla F(\w_{1})\right\|}{ \alpha T}  +\frac{4  C_g^2C_f^2}{B_1} \alpha+ \frac{8\left(C_g^2L_f^2+C_f^2L_g^2+C_g^4k^2\right)\eta^2}{B_1 \alpha} \E\left[ \frac{1}{T}\sum_{t=1}^T  \|\z_{t}\|^2 \right]
\end{split}
\end{equation*}
Combining the above and  setting $32\left(C_g^2L_f^2+C_f^2L_g^2+C_g^4k^2\right)\eta^2 \leq \min\left\{B_1 \alpha,1\right\}$, we know that
\begin{equation*}
\begin{split}
  & \frac{1}{T}\sum_{t=1}^T  \|\z_{t}\|^2 \leq \frac{2}{T}\sum_{t=1}^T \|\nabla F(\w_t) \|^2 +\frac{2}{T} \sum_{t=1}^T  \|\z_{t}- \nabla F(\w_t)\|^2 \\ 
  \leq &\frac{4\left(F(\w_1) - F_*)\right)}{\eta T} +   \frac{4}{T} \sum_{t=1}^T  \|\z_{t}- \nabla F(\w_t)\|^2 - \frac{1}{T}\sum_{t=1}^T\Norm{\z_t}^2\\
  \leq &\frac{4\left(F(\w_1) - F_{*}\right)}{\eta T} + \frac{8\left\|\z_{1} - \nabla F(\w_{1})\right\|}{ \alpha T}  +\frac{16  C_g^2C_f^2}{B_1} \alpha +\left(\frac{32\left(C_g^2L_f^2+C_f^2L_g^2+C_g^4k^2\right)\eta^2}{\min\left\{B_1 \alpha,1\right\}} -1 \right)\frac{1}{T}\sum_{t=1}^T\Norm{\z_t}^2\\
  \leq &\frac{4\left(F(\w_1) - F_{*}\right)}{\eta T} + \frac{8\left\|\z_{1} - \nabla F(\w_{1})\right\|}{ \alpha T}  +\frac{16  C_g^2C_f^2}{B_1} \alpha 
\end{split}
\end{equation*}
Then, we know that
\begin{equation*}
\begin{split}
&\frac{1}{T}\sum_{t=1}^T \E\left[\|\z_{t} - \nabla F(\w_t)\|^2 \right]\\
\leq & \frac{2\left\|\z_{1} - \nabla F(\w_{1})\right\|}{ \alpha T}  +\frac{4  C_g^2C_f^2}{B_1} \alpha+ \frac{8\left(C_g^2L_f^2+C_f^2L_g^2\right)\eta^2}{\min\left\{B_1 \alpha,1\right\}} \E\left[ \frac{1}{T}\sum_{t=1}^T  \|\z_{t}\|^2 \right]\\
\leq & \frac{2\left\|\z_{1} - \nabla F(\w_{1})\right\|}{ \alpha T}  +\frac{4  C_g^2C_f^2}{B_1} \alpha+ \frac{1}{4} \left( \frac{4\left(F(\w_1) - F_{*}\right)}{\eta T} + \frac{8\left\|\z_{1} - \nabla F(\w_{1})\right\|}{ \alpha T}  +\frac{16  C_g^2C_f^2}{B_1} \alpha \right)\\
\leq & \frac{4\left\|\z_{1} - \nabla F(\w_{1})\right\|}{ \alpha T}  +\frac{8  C_g^2C_f^2}{B_1} \alpha+ \frac{F(\w_1) - F_{*}}{\eta T}  
\end{split}
\end{equation*}
To sum up, we know that for stage $s$ we have
\begin{equation*}
\begin{split}
  F(\mathbf{w_s})-F_{*} 
 &\leq \frac{F(\w_{s-1}) - F_*}{\mu\eta_s T_s} +   \frac{1}{2\mu}   \|\z_{s}- \nabla F(\w_s)\|^2\\
 \|\z_{s}- \nabla F(\w_s)\|^2 &\leq \frac{4\left\|\z_{s-1} - \nabla F(\w_{s-1})\right\|}{ \alpha_s T_s}  +\frac{8  C_g^2C_f^2}{B_1} \alpha_s+ \frac{F(\w_{s-1}) - F_{*}}{\eta_s T_s} 
\end{split}
\end{equation*}
Suppose that for stage $s-1$, we can ensure that $\E\left[ F(\w_{s-1}) - F_* \right]\leq \epsilon_{s-1}$ and $\E\left[ \|\z_{s-1}- \nabla F(\w_{s-1})\|^2 \right]\leq \mu\epsilon_{s-1} $. Next, we prove that $\E\left[F(\w_{s}) - F_* \right]\leq \epsilon_{s}=\epsilon_{s-1}/2$ and $\E\left[ \|\z_{s}- \nabla F(\w_s)\|^2 \right]\leq \mu\epsilon_{s}\leq \mu\epsilon_{s-1}/2 $ for stage $s$.

First, we can prove that
\begin{equation*}
\begin{split}
 & \|\z_{s}- \nabla F(\w_s)\|^2  
 \leq  \frac{4\left\|\z_{s-1} - \nabla F(\w_{s-1})\right\|}{ \alpha_s T_s}  +\frac{8  C_g^2C_f^2}{B_1} \alpha_s+ \frac{F(\w_{s-1}) - F_{*}}{\eta_s T_s}  \\
 \leq &\frac{4 \mu\epsilon_{s-1}}{\alpha_s T_s} 
+\frac{8  C_g^2C_f^2}{B_1} \alpha_s +\frac{\epsilon_{s-1}}{\eta_s T_s}
\leq  \frac{\mu \epsilon_s}{2} + \frac{\mu \epsilon_s}{4} + \frac{\mu \epsilon_s}{4}
\leq  \mu \epsilon_s,
\end{split}
\end{equation*}
by setting that
\begin{align*}
     \frac{8  C_g^2C_f^2}{B_1} \alpha_s \leq \frac{\mu \epsilon_s}{4}; \quad 32\left(C_g^2L_f^2+C_f^2L_g^2+C_g^4k^2\right)\eta^2 \leq \min\left\{B_1 \alpha,1\right\};  \quad   \alpha_s T_s \ge 16; \quad \mu \eta_s T_s \ge 8. 
\end{align*}
which can be achieved by
\begin{align*}
     \alpha_s = \frac{B_1\mu  \epsilon_s}{32 C_g^2 C_f^2},& \eta_s =\min\left\{ \frac{B_1(\mu  \epsilon_s)^{1/2}}{32\sqrt{\left(C_g^2L_f^2+C_f^2L_g^2+C_g^4k^2\right)C_g^2 C_f^2}} ,\frac{1}{\sqrt{32\left(C_g^2L_f^2+C_f^2L_g^2+C_g^4k^2\right)}} \right\},
    \\
    T_s \ge &\max\left\{ \frac{512C_g^2C_f^2}{ B_1 \mu\epsilon_s}, \frac{256\sqrt{\left(C_g^2L_f^2+C_f^2L_g^2\right)C_g^2 C_f^2}}{B_1\mu(\mu  \epsilon_s)^{1/2}},\frac{8\sqrt{32\left(C_g^2L_f^2+C_f^2L_g^2+C_g^4k^2\right)}}{\mu} \right\}\\
    \ge & \frac{C}{B_1 \mu \epsilon_s} + \frac{C}{B_1 \mu^{3/2} \epsilon_s^{1/2}}+\frac{C}{\mu},
\end{align*}
where $C=512C_g^2C_f^2 + 256\sqrt{\left(C_g^2L_f^2+C_f^2L_g^2+C_g^4k^2\right)C_g^2 C_f^2} + 8\sqrt{32\left(C_g^2L_f^2+C_f^2L_g^2+C_g^4k^2\right)}$.

Second, we can prove that
\begin{equation*}
\begin{split}
  F(\mathbf{w_s})-F_{*} 
 &\leq \frac{F(\w_{s-1}) - F_*}{\mu\eta_s T_s} +   \frac{1}{2\mu}   \|\z_{s}- \nabla F(\w_s)\|^2 \leq \frac{\epsilon_{s-1}}{\mu\eta_s T_s} + \frac{\mu \epsilon_s}{2\mu} \leq \epsilon_s.
\end{split}
\end{equation*}
For the first stage $s=1$, we can ensure that
\begin{equation*}
\begin{split}
 \|\z_{1}- \nabla F(\w_1)\|^2 &\leq \frac{4\left\|\z_{0} - \nabla F(\w_{0})\right\|}{ \alpha_1 T_1}  +\frac{8  C_g^2C_f^2}{B_1} \alpha_1+ \frac{F(\w_{0}) - F_{*}}{\eta_1 T_1} \\
&\leq \frac{C_g^2C_f^2}{4} 
+\frac{8  C_g^2C_f^2}{B_1} +\frac{\mu \Delta_f}{8}\leq \mu \epsilon_1;\\
F(\mathbf{w_1})-F_{*} 
 &\leq \frac{F(\w_{0}) - F_*}{\mu\eta_1 T_1} +   \frac{1}{2\mu}   \|\z_{1}- \nabla F(\w_1)\|^2\leq \frac{\Delta_f}{8} + \frac{\epsilon_1}{2} \leq \epsilon_1
\end{split}
\end{equation*}
where we assume that
\begin{align*}
    \epsilon_1 = \frac{C_g^2C_f^2}{4\mu} 
+\frac{8  C_g^2C_f^2}{\mu B_1} +\frac{ \Delta_f}{4}.
\end{align*}

By Induction, we know that 
\begin{equation*}
\begin{split}
  F(\mathbf{w_S})-F_{*} \leq \epsilon_S = \epsilon_1/2^{S-1} = (2\epsilon_1)/2^{S}
\end{split}
\end{equation*}
To ensure $F(\mathbf{w_S})-F_{*} \leq \epsilon$, we only require $S = \log(\frac{2\epsilon_1}{\epsilon})$ and the sample complexity is
\begin{align*}
    \sum_{s=1}^S T_s = \sum_{s=1}^S \frac{C}{B_1 \mu\epsilon_s} + \frac{C}{ B_1 \mu^{3/2}\epsilon_s^{1/2}}+\frac{C}{\mu}= \frac{C}{2 B_1 \mu  \epsilon_1} \sum_{s=1}^S 2^s +\frac{C}{B_1 \mu^{3/2}  \sqrt{2 \epsilon_1} }\sum_{s=1}^S \sqrt{2}^s + \frac{C}{\mu} S\\
    \leq \frac{C}{B_1 \mu  \epsilon_1} 2^S + \frac{3C}{B_1\mu^{3/2}  \sqrt{2\epsilon_1}} \sqrt{2}^S + \frac{C}{\mu} S 
    = \frac{C}{B_1 \mu  \epsilon_1} \frac{2\epsilon_1}{\epsilon} + \frac{3C}{B_1\mu^{3/2} \sqrt{2 \epsilon_1}} \sqrt{\frac{2\epsilon_1}{\epsilon}}+ \frac{C}{\mu} \log\left(\frac{2\epsilon_1}{\epsilon}\right) \\= \frac{2C}{B_1 \mu \epsilon }  + \frac{3C}{B_1\mu^{3/2} \sqrt{\epsilon}} + \frac{C}{\mu} \log\left(\frac{2\epsilon_1}{\epsilon}\right)= \mathcal{O}\left(\frac{1}{B_1\mu \epsilon}\right),
\end{align*}
since we usually assume that $\epsilon \leq \mu$.

When $F(\w)$ is convex, we define $\hat{F}(\w) = F(\w) + \frac{\mu}{2}\|\w\|^2$. We know that $\hat{F}(\w)$ is $\mu$-strongly convex, which implies $\mu$-PL condition. We have proved: for any $\delta > 0$,  there exist $T=\mathcal{O}\left(\frac{1}{B_1 \mu \delta}\right)$ such that $\hat{F}(\w_{T}) - \hat{F}_{*} \leq \delta$.  It indicates that $F(\w_{T}) - F_{*} \leq \delta + \frac{\mu}{2} \|\w_{*}\|^2 -  \frac{\mu}{2} \|\w_{T}\|^2 \leq \delta + \frac{\mu}{2}D$. For any $\epsilon > 0$, if we choose $\mu = \frac{\epsilon}{D}$ and $\delta = \frac{\epsilon}{2}$, we get $F(\w_{T}) - F_{*} \leq \epsilon$, for some $T=\mathcal{O}\left(\frac{1}{B_1\epsilon^2}\right)$.

\newpage
\section{Proof of Theorem~12} 
When the objective function satisfies the PL condition, we have already proved that 
\begin{equation*}
\begin{split}
  F(\mathbf{w})-F_{*} \leq &\frac{F(\w_1) - F_{*}}{\mu\eta T} +   \frac{1}{2\mu T} \sum_{t=1}^T  \|\z_{t}- \nabla F(\w_t)\|^2
\end{split}
\end{equation*}
In the analysis of Theorem~6, by setting that $\alpha I \leq 1$ we know that 
\begin{equation*}
\begin{split}
  \frac{1}{T}\sum_{t=1}^T \|\nabla F(\w_t) \|^2 
 \leq &\frac{2\left(F(\w_1) - F_*\right)}{\eta T} +   \frac{1}{T} \sum_{t=1}^T  \|\z_{t}- \nabla F(\w_t)\|^2 - \frac{1}{2T}\sum_{t=1}^T\Norm{\z_t}^2 \\
\frac{1}{T}\sum_{t=1}^T \E\left[\|\z_{t} - \nabla F(\w_t)\|^2 \right]
\leq &  \frac{8\left(C_g^2L_f^2+C_f^2L_g^2+C_f^4+C_g^4L_f^2+C_g^4k^2\right)\eta^2}{\min\left\{B_1 \alpha,1\right\}} \E\left[ \frac{1}{T}\sum_{t=1}^T  \|\z_{t}\|^2 \right]
\end{split}
\end{equation*}
Combining the above and  setting $32\left(C_g^2L_f^2+C_f^2L_g^2+C_f^4+C_g^4L_f^2+C_g^4k^2\right)\eta^2 \leq \min\left\{B_1 \alpha,1\right\}$, we know that
\begin{equation*}
\begin{split}
  & \frac{1}{T}\sum_{t=1}^T  \|\z_{t}\|^2 \leq \frac{2}{T}\sum_{t=1}^T \|\nabla F(\w_t) \|^2 +\frac{2}{T} \sum_{t=1}^T  \|\z_{t}- \nabla F(\w_t)\|^2 \\ 
  \leq &\frac{4\left(F(\w_1) - F_*)\right)}{\eta T} +   \frac{4}{T} \sum_{t=1}^T  \|\z_{t}- \nabla F(\w_t)\|^2 - \frac{1}{T}\sum_{t=1}^T\Norm{\z_t}^2\\
  \leq &\frac{4\left(F(\w_1) - F_{*}\right)}{\eta T} +\left(\frac{32\left(C_g^2L_f^2+C_f^2L_g^2+C_f^4+C_g^4L_f^2\right)\eta^2}{\min\left\{B_1 \alpha,1\right\}} -1 \right)\frac{1}{T}\sum_{t=1}^T\Norm{\z_t}^2 
  \leq \frac{4\left(F(\w_1) - F_{*}\right)}{\eta T} 
\end{split}
\end{equation*}
Then, we know that
\begin{equation*}
\begin{split}
\frac{1}{T}\sum_{t=1}^T \E\left[\|\z_{t} - \nabla F(\w_t)\|^2 \right]
\leq &  \frac{8\left(C_g^2L_f^2+C_f^2L_g^2+C_f^4+C_g^4L_f^2\right)\eta^2}{\min\left\{B_1 \alpha,1\right\}} \E\left[ \frac{1}{T}\sum_{t=1}^T  \|\z_{t}\|^2 \right]\\
\leq &  \frac{1}{4} \left( \frac{4\left(F(\w_1) - F_{*}\right)}{\eta T}  \right) 
\leq \frac{F(\w_1) - F_{*}}{\eta T}   
\end{split}
\end{equation*}
To sum up, we know that for stage $s$ we have
\begin{equation*}
\begin{split}
  F(\mathbf{w_s})-F_{*} 
 &\leq \frac{F(\w_{s-1}) - F_*}{\mu\eta_s T_s} +   \frac{1}{2\mu}   \|\z_{s}- \nabla F(\w_s)\|^2\\
 \|\z_{s}- \nabla F(\w_s)\|^2 &\leq \frac{F(\w_{s-1}) - F_{*}}{\eta_s T_s}  
\end{split}
\end{equation*}
Suppose that for stage $s-1$, we can ensure that $\E\left[ F(\w_{s-1}) - F_* \right]\leq \epsilon_{s-1}$ and $\E\left[ \|\z_{s-1}- \nabla F(\w_{s-1})\|^2 \right]\leq \mu\epsilon_{s-1} $. Next, we prove that $\E\left[F(\w_{s}) - F_* \right]\leq \epsilon_{s}=\epsilon_{s-1}/2$ and $\E\left[ \|\z_{s}- \nabla F(\w_s)\|^2 \right]\leq \mu\epsilon_{s}\leq \mu\epsilon_{s-1}/2 $ for stage $s$.

First, we can prove that
\begin{equation*}
\begin{split}
\|\z_{s}- \nabla F(\w_s)\|^2 
 \leq   \frac{F(\w_{s-1}) - F_{*}}{\eta_s T_s}  
 \leq \frac{\epsilon_{s-1}}{\eta_s T_s} 
\leq  \mu \epsilon_s,
\end{split}
\end{equation*}
by setting that
\begin{align*}
     32\left(C_g^2L_f^2+C_f^2L_g^2+C_f^4+C_g^4L_f^2+C_g^4k^2\right)\eta^2 \leq \min\left\{B_1 \alpha,1\right\};  \quad \mu \eta_s T_s \ge 2. 
\end{align*}
which can be achieved by
\begin{align*}
     I=\frac{mn}{B_1}, \alpha_s = \frac{B_1}{mn}, \eta_s = \frac{1}{\sqrt{32\left(C_g^2L_f^2+C_f^2L_g^2+C_f^4+C_g^4L_f^2+C_g^4k^2\right)}} \min\left\{\sqrt{\frac{B_1^2}{mn}},1\right\} 
    \\
    T_s \ge \frac{2\sqrt{32\left(C_g^2L_f^2+C_f^2L_g^2+C_f^4+C_g^4L_f^2+C_g^4k^2\right)}}{\mu} \left( \frac{\sqrt{mn}}{B_1}+1\right)
    \ge  \frac{C \sqrt{mn}}{ \mu B_1} +  \frac{C}{\mu},
\end{align*}
where $C=2\sqrt{32\left(C_g^2L_f^2+C_f^2L_g^2+C_f^4+C_g^4L_f^2+C_g^4k^2\right)}$.

Second, we can prove that
\begin{equation*}
\begin{split}
  F(\mathbf{w_s})-F_{*} 
 &\leq \frac{F(\w_{s-1}) - F_*}{\mu\eta_s T_s} +   \frac{1}{2\mu}   \|\z_{s}- \nabla F(\w_s)\|^2 \leq \frac{\epsilon_{s-1}}{\mu\eta_s T_s} + \frac{\mu \epsilon_s}{2\mu} \leq \epsilon_s.
\end{split}
\end{equation*}

For the first stage $s=1$, we can ensure that
\begin{equation*}
\begin{split}
 \|\z_{1}- \nabla F(\w_1)\|^2 \leq  \frac{F(\w_{0}) - F_{*}}{\eta_1 T_1} \leq \frac{\mu \Delta_f}{2}\leq \mu \epsilon_1;\\
F(\mathbf{w_1})-F_{*} \leq \frac{F(\w_{0}) - F_*}{\mu\eta_1 T_1} +   \frac{1}{2\mu}   \|\z_{1}- \nabla F(\w_1)\|^2\leq \frac{\Delta_f}{2} + \frac{\epsilon_1}{2} \leq \epsilon_1
\end{split}
\end{equation*}
where we assume that
\begin{align*}
    \epsilon_1 =\frac{ \Delta_f}{2}.
\end{align*}

By Induction, we know that 
\begin{equation*}
\begin{split}
  F(\mathbf{w_S})-F_{*} \leq \epsilon_S = \epsilon_1/2^{S-1} = (2\epsilon_1)/2^{S}
\end{split}
\end{equation*}
To ensure $F(\mathbf{w_S})-F_{*} \leq \epsilon$, we only require $S = \log(\frac{2\epsilon_1}{\epsilon})$ and the sample complexity is
\begin{align*}
    \sum_{s=1}^S T_s  = \frac{C\sqrt{mn}}{B_1\mu }\log\left(\frac{2\epsilon_1}{\epsilon}\right).
\end{align*}
When $F(\w)$ is convex, we define $\hat{F}(\w) = F(\w) + \frac{\mu}{2}\|\w\|^2$. We know that $\hat{F}(\w)$ is $\mu$-strongly convex, which implies $\mu$-PL condition. We have proved: for any $\delta > 0$,  there exist $T=\O\left(\frac{\sqrt{m n}}{\mu B_1 }  \cdot \log\frac{1}{\epsilon}\right)$ such that $\hat{F}(\w_{T}) - \hat{F}_{*} \leq \delta$.  It indicates that $F(\w_{T}) - F_{*} \leq \delta + \frac{\mu}{2} \|\w_{*}\|^2 -  \frac{\mu}{2} \|\w_{T}\|^2 \leq \delta + \frac{\mu}{2}D$. For any $\epsilon > 0$, if we choose $\mu = \frac{\epsilon}{D}$ and $\delta = \frac{\epsilon}{2}$, we get $F(\w_{T}) - F_{*} \leq \epsilon$, for some $T=\O\left(\frac{ \sqrt{mn}}{\epsilon B_1 }  \cdot
    \log\frac{1}{\epsilon}\right)$.
\end{document}